\documentclass{mcom-f} %

\usepackage{amssymb}
\usepackage{graphicx,caption,subcaption}
\usepackage{float}
\usepackage{stmaryrd}
\usepackage{hyperref}
\usepackage{xcolor}

\usepackage[bbgreekl]{mathbbol}
\DeclareSymbolFont{bbold}{U}{bbold}{m}{n}
\DeclareSymbolFontAlphabet{\mathbbold}{bbold}
\DeclareSymbolFontAlphabet{\mathbb}{AMSb}%

\newtheorem{theorem}{Theorem}[section]
\newtheorem*{theorem*}{Theorem}
\newtheorem{corollary}{Corollary}[section]
\newtheorem{lemma}{Lemma}[section]

\newtheorem{remark}{Remark}[section]
\newtheorem{prop}{Proposition}[section]
\newtheorem{assumption}{Assumption}
\usepackage[ruled,vlined,linesnumbered]{algorithm2e}
\labelformat{algocf}{Algorithm #1}

\labelformat{equation}{(#1)}
\labelformat{figure}{Figure~#1}
\labelformat{section}{Section~#1}
\labelformat{subsection}{Section~#1}
\labelformat{subsubsection}{Section~#1}
\labelformat{theorem}{Theorem~#1}
\labelformat{prop}{Proposition~#1}
\labelformat{lemma}{Lemma~#1}
\labelformat{remark}{Remark~#1}
\labelformat{assumption}{Assumption~#1}
\labelformat{corollary}{Corollary~#1}
\labelformat{algorithm}{Algorithm~#1}

\renewcommand{\vec}[1]{\overrightarrow{#1}}
\newcommand{\N}[0]{\mathbb{N}}
\newcommand{\Z}[0]{\mathbb{Z}}
\newcommand{\Q}[0]{\mathbb{Q}}
\newcommand{\R}[0]{\mathbb{R}}
\newcommand{\C}[0]{\mathbb{C}}
\newcommand{\U}[0]{\mathbb{U}}
\newcommand{\G}[0]{\mathbb{G}}
\renewcommand{\P}[0]{\mathbb{P}}
\newcommand{\Int}[2]{\displaystyle\int_{#1}^{#2}}
\newcommand{\Sum}[2]{\displaystyle\sum\limits_{#1}^{#2}}

\newcommand{\Inter}[2]{\displaystyle\bigcap\limits_{#1}^{#2}}
\newcommand{\Reu}[2]{\displaystyle\bigcup\limits_{#1}^{#2}}
\newcommand{\E}[1]{\mathbb{E}\left[#1\right]}

\newcommand{\dr}[3]{\cfrac{\partial ^{#2} {#3}}{\partial {#1}^{#2}}}
\newcommand{\dd}[0]{\mathrm{d}}

\newcommand{\Span}[0]{\mathrm{Span}}

\newcommand{\nt}[1]{{\left\vert\kern-0.25ex\left\vert\kern-0.25ex\left\vert #1
		\right\vert\kern-0.25ex\right\vert\kern-0.25ex\right\vert}}

\DeclareMathOperator{\argmin}{argmin}
\DeclareMathOperator{\Arctan}{Arctan}

\DeclareMathOperator{\sign}{sign}

\newcommand{\W}[0]{\mathrm{W}}
\renewcommand{\SS}{\mathbb{S}}

\newcommand{\app}[4]{\left\lbrace\begin{array}{ccc}
		#1 & \longrightarrow & #2 \\
		#3 & \longmapsto & #4 \\
	\end{array} \right.}

\newcommand{\oll}[1]{\overline{#1}}
\newcommand{\ull}[1]{\underline{#1}}

\newcommand{\npoints}{n}
\newcommand{\SWY}{\mathcal{E}}
\newcommand{\SWYbar}{\mathcal{E}_{\mathrm{bar}}}
\newcommand{\SWpY}{\mathcal{E}_p}

\newcommand{\SWMC}{\widehat{\mathrm{SW}}_p}
\newcommand{\SW}{\mathrm{SW}}
\newcommand{\sort}[2]{\tau_{#1}^{#2}}  %
\newcommand{\config}{\mathbf{m}}

\numberwithin{equation}{section}
\newcommand{\blue}[1]{#1} %
\newcommand{\bluetwo}[1]{#1} %

\makeatletter
\renewcommand*{\fps@figure}{h}  %
\makeatother
  
\begin{document}

	\title{Properties of discrete sliced Wasserstein losses}

	\author[E. Tanguy]{Eloi Tanguy}
	\address{Universit\'e Paris Cit\'e, CNRS, MAP5, F-75006 Paris, France}
	\email{eloi.tanguy@u-paris.fr}
	
	\author[R. Flamary]{R\'emi Flamary}
	\address{CMAP, CNRS, Ecole Polytechnique, Institut Polytechnique de Paris}
	\email{remi.flamary@polytechnique.edu}
	
	\author[J. Delon]{Julie Delon}
	\address{Universit\'e Paris Cit\'e, CNRS, MAP5, F-75006 Paris, France}
	\email{julie.delon@u-paris.fr}
	
	\subjclass[2020]{49Q22 (Optimal Transport)}
	
	\date{18/03/2024}
	
	\dedicatory{}
	
	\begin{abstract}
		The Sliced Wasserstein (SW) distance has become a popular alternative to the Wasserstein distance for comparing probability measures. Widespread applications include image processing, domain adaptation and generative modelling, where it is common to optimise some parameters in order to minimise SW, which serves as a loss function between discrete probability measures (since measures admitting densities are numerically unattainable). All these optimisation problems bear the same sub-problem, which is minimising the Sliced Wasserstein energy. In this paper we study the properties of  $\mathcal{E}: Y \longmapsto \mathrm{SW}_2^2(\gamma_Y, \gamma_Z)$, i.e. the SW distance between two uniform discrete measures with the same amount of points as a function of the support $Y \in \mathbb{R}^{n \times d}$ of one of the measures. We investigate the regularity and optimisation properties of this energy, as well as its Monte-Carlo approximation $\mathcal{E}_p$ (estimating the expectation in SW using only $p$ samples) and show convergence results on the critical points of $\mathcal{E}_p$ to those of $\mathcal{E}$, as well as an almost-sure uniform convergence \blue{and a uniform Central Limit result on the process $\SWpY$}. Finally, we show that in a certain sense, Stochastic Gradient Descent methods minimising $\mathcal{E}$ and $\mathcal{E}_p$ converge towards (Clarke) critical points of these energies. 
	\end{abstract}
	
	\maketitle

	\tableofcontents
	
	\section{Introduction}

Optimal Transport (OT) has grown in popularity as a way of lifting a notion of cost between points in a space onto a way of comparing measures on said space. In particular, endowing $\R^d$ with a $p$-norm yields the Wasserstein distance, which metrises the convergence in law on the space of Radon measures with a finite moment of order $p$.

The most studied object that arises from this theory is perhaps the 2-Wasserstein distance, which is defined as follows (see~\cite{computational_ot, santambrogio2015optimal, villani} for a complete practical and theoretical presentation):
\begin{equation}\label{eqn:W2}
	\forall \mu, \nu \in \mathcal{P}_2(\R^d),\; \W_2^2(\mu, \nu) := \underset{\pi \in \Pi(\mu, \nu)}{\inf}\Int{\R^d\times\R^d}{}\|x_1-x_2\|^2\dd \pi(x_1, x_2),
\end{equation}
where $\Pi(\mu, \nu)$ is the set of probability measures on $\R^d \times \R^d$ of first marginal $\mu$ and second marginal $\nu$. We denote $\mathcal{P}_2(\R^d)$ as the set of probability measures on $\R^d$ admitting a second-order moment.

The 1 and 2-Wasserstein distances are commonly used for generation tasks, formulated as probability density fitting problems. One defines a statistical model $\mu_\theta$, a probability measure which is designed to approach a target data distribution $\mu$. A typical way of solving this problem is to minimise in $\theta$ the distance between $\mu_\theta$ and $\mu$: one may choose any probability discrepancies (Kullback-Leibler, Ciszar divergences, f-divergences or Maximum Mean Discrepancy), or alternatively the Wasserstein Distance. In the case of Generative Adversarial Networks, the so-called "Wasserstein GAN" \cite{pmlr-v70-arjovsky17a,gulrajani2017improved} draws its formulation from the dual expression of the 1-Wasserstein distance.

Unfortunately, computing the Wasserstein distance is prohibitively costly in practice. The discrete formulation of the Wasserstein distance (the Kantorovich linear problem) is typically solved approximately using standard linear programming tools. These methods suffer from a super-cubic worst-case complexity with respect to the number of samples from the two measures. Furthermore, given $\npoints$ samples from each measure $\mu$ and $\nu$, the convergence of the estimated distance $\W_2(\hat\mu_\npoints, \hat\nu_\npoints)$ is only in $\mathcal{O}(\npoints^{-1/d})$ towards the true distance, thus OT suffers from the curse of dimensionality, as is known since Dudley, 1969 \cite{dudley1969speed}.

Several efforts have been made in recent years to make Optimal Transport more accessible computationally. In particular, many surrogates for $\W_2$ have been proposed, perhaps the most notable of which is the Sinkhorn Divergence (see~\cite{computational_ot,cuturi2013sinkhorn,genevay2018learning}). The Sinkhorn Divergence adds entropic regularisation to OT, yielding a strongly convex algorithm which can be solved efficiently.

Another alternative is the Sliced Wasserstein (SW) Distance, which leverages the simplicity of computing the Wasserstein distance between one-dimensional measures.
Indeed, given 
$$\gamma_X := \cfrac{1}{\npoints}\Sum{k=1}{\npoints}\delta_{x_k},\; \gamma_Y := \cfrac{1}{\npoints}\Sum{k=1}{\npoints}\delta_{y_k}\text{\ with\ } x_1, \cdots, x_\npoints, y_1, \cdots, y_\npoints \in \R,$$
the 2-Wasserstein distance between these two measures can be computed by sorting their supports:
\begin{equation}
	\W_2^2(\gamma_X, \gamma_Y) = \cfrac{1}{\npoints}\Sum{k=1}{\npoints}(x_{\sigma(k)} - y_{\tau(k)})^2,
\end{equation}
where $\sigma$ is a permutation sorting $(x_1, \cdots, x_\npoints)$, and $\tau$ is a permutation sorting $(y_1, \cdots, y_\npoints)$.

The idea of the Sliced Wasserstein Distance~\cite{Rabin_texture_mixing_sw} is to compute the 1D Wasserstein distances between projections of input measures. We write $P_\theta: \R^d \longrightarrow \R$ the map $x \longmapsto \theta^T x$, and $\bbsigma$ the uniform measure over the euclidean unit \blue{sphere} of $\R^d,\ \SS^{d-1}$. Denoting $\#$ the push-forward operation~\footnote{The push-forward of a measure  $\mu$ on $\R^d$ by an application $T: \R^d \rightarrow \R^k$ is defined as a measure $T\#\mu$ on $\R^k$ such that for all Borel sets $B \in \mathcal{B}(\R^k), T\#\mu(B) = \mu(T^{-1}(B))$.}, the Sliced Wasserstein distance between two measures $\mu$ and $\nu$ is defined as 
\begin{equation}\label{eqn:SW}
	\mathrm{SW}_2^2(\mu, \nu) := \Int{\theta \in \SS^{d-1}}{}\W_2^2(P_\theta\#\mu, P_\theta\#\nu)\dd \bbsigma(\theta).
\end{equation}
\blue{Similarly, for $q \ge 1$, the $q$-Sliced Wasserstein distance $\mathrm{SW}_q^q$ (to the power $q$) is obtained by replacing $\W_2^2$ in the previous equation by the $q$ Wasserstein distance (to the power $q$) $\W_q^q$.}

\blue{SW has enjoyed a substantial amount of theoretical study, albeit not as extensively as for the original Wasserstein distance.} For measures supported on a fixed compact of $\R^d$, Bonnotte (\cite{bonnotte}, Chapter 5) has shown that the Wasserstein and Sliced Wasserstein distances are equivalent. The same work also developed a theory of gradient flows for SW, which justifies some generative methods. Further discussion on this equivalence has been performed by Bayraktar and Guo~\cite{bayraktar_equivalence_W_SW}. Nadjahi et al.~\cite{nadjahi2019_guarantees_sw} showed that SW metrises the convergence in law (without restrictions of the measure supports), and further concluded guarantees for SW-based generative models. 

Continuous measures being out of the reach of practical computation, it is necessary to perform sample estimation and replace them with discrete empirical estimates. Thankfully, as shown in \cite{nadjahi_statistical_properties_sliced}, the \textit{sample complexity} (i.e. the rate of convergence of the estimates w.r.t. the number of samples) for sliced distances such as SW is in $1/\sqrt{\npoints}$, which in particular avoids the curse of dimensionality from which the Wasserstein Distance suffers. This fuels interest for the study of $Y \mapsto \mathrm{SW}(\gamma_Y,\gamma)$, which is to say the variation of SW w.r.t. the discrete support of one of the measures. It is currently unknown whether this functional presents strict local optima, for instance.

Originally, SW was introduced as a more computable alternative to the Wasserstein distance, notably for texture mixing using a barycentric formulation~\cite{Rabin_texture_mixing_sw, bonneel2015sliced}. Other uses of SW have been suggested, notably in statistics as a probability discrepancy. For instance, Nadjahi et al.~\cite{nadjahi_bayesian} proposed an approximate bayesian computation method, where the estimation of the posterior parameters is done by selecting those under which the SW distance between observed and synthetic data is below a fixed threshold. Other widespread uses of SW in image processing include colour transfer \cite{alghamdi2019patch} and colour harmonisation \cite{bonneel2015blind}.

Nowadays, SW is commonly used as a training or validation loss in generative Machine Learning. Karras et al.~\cite{karras2017progressive} propose to use SW to compare GAN results, by comparing images via multi-scale patched descriptors. Some generative models (including GANs and auto-encoders), leverage the computational advantages of SW in order to learn a target distribution. This is done under the implicit generative modelling framework, where a network $T_u$ of parameters $u$ is learned such as to minimise $\mathrm{SW}(T_u\#\mu_0, \mu)$, where $\mu_0$ is a low-dimensional input distribution (often chosen as Gaussian or uniform noise), and where $\mu$ is the target distribution. Deshpande et al.~\cite{deshpande_generative_sw} and Wu et al.~\cite{Wu2019_SWAE} train GANs and auto-encoders within this framework; Liutkus et al.~\cite{liutkus19a_SWflow_generation} perform generation by minimising a regularised SW problem, which they solve by gradient flow using an SDE formulation. SW can be used to synthesise images by minimising the SW distance between features of the optimised image and a target image, as done by \cite{heitz2021sliced} for textures with neural features, and by \cite{Tartavel2016} with wavelet features (amongst other methods).

In practice,  the integration over the unit sphere in SW is intractable, and one must resort to a Monte-Carlo approximation, taking the average between $p$ projections instead of the expectation\blue{, usually during iterations of a Stochastic Gradient Descent \cite{kolouri2018slicedAE}. This implies that for a finite number of iterations, a fixed number of projections $p$, potentially very small compared to what is needed to explore the hypersphere, is explored in practice. The question of this finite number of final projection directions is made even more important by the fact that practitioners usually optimise the expectation of the SW distance on large mini-batches \cite{deshpande_generative_sw} that also limits the total number of effective  projections $p$.}  The estimation error of this approximation has not been extensively studied, and it is common in practice to assume that this empirical version presents the same properties as the true SW distance.

An important question is the conditions under which these approximations for SW are valid. In practice, sliced-Wasserstein Generative Models compute SW in the data space or in the data encoding space (\cite{kolouri2018slicedAE, deshpande_generative_sw}), which yields high values for the dimension $d$, in particular for images. Note that the necessity behind having a large number of projections $p$ was already hinted at in~\cite{kolouri2018slicedAE}, \S 3.3. Another untreated question is the complexity of optimising this approximation of SW, and how this optimisation landscape compares to the true SW landscape.

Bonneel et al.~\cite{bonneel2015sliced} studied the uses of SW for barycentre computation, and in particular proved that the empirical SW distance is $\mathcal{C}^1$ on a certain open set, with respect to the measure positions. They remarked that in practice, numerical resolutions for discretised SW distances converged towards (eventual) local optima, however the convergence and local optima have not been studied theoretically.

In this paper, we propose to study $\SWY:  Y \mapsto
\mathrm{SW}_2^2(\gamma_Y,\gamma_Z)$, where $\gamma_Y$ and $\gamma_Z$ are two
uniform discrete measures supported by $n$ points, denoted by $Y$ and $Z$. Our
main objective is to provide optimisation properties for the landscapes of
$\SWY$ and its Monte-Carlo counterpart $\SWpY$, obtained by replacing the
expectation by an average over $p$ projections. In~\ref{sec:energies}, we prove
several regularity properties for both energies, such as semi-concavity, and we
show that the convergence of the Monte-Carlo estimation is uniform (on every
compact) w.r.t. the measure locations. \ref{sec:crit} focuses on the respective
landscapes of   $\SWY$  and  $\SWpY$, and shows that the critical points of
$\SWY$ satisfy a fixed-point equation, and how the critical points of $\SWpY$
relate to this fixed-point equation when the number of projections $p$ increases
(with convergence rates). Mérigot et al. follow a similar methodology
in~\cite{merigot2021bounds_approx_W2}, where they study optimisation properties
for $Y \longmapsto \W_2(\gamma_Y, \mu)$, with $\mu$ a continuous measure. The
main difficulty they face arises from the non-convexity of the map, and this
difficulty is also central in our work. The last two sections of our paper
tackle numerical considerations. To begin with, since  $\SWY$ and $\SWpY$ are
usually minimised in the literature using Stochastic Gradient Descent (SGD), we
provide in ~\ref{sec:SGD} the first complete convergence study of SGD for $\SWY$
and $\SWpY$, relying on the recent works~\cite{bianchi2022convergence}
\bluetwo{and \cite{davis2020stochastic}}. Finally,~\ref{sec:xp} challenges our
theoretical results with extensive numerical experiments, quantifying the impact
of the dimension and several other parameters on the convergence.

\blue{\subsection*{Notations}
\begin{itemize}
	\item $d$ is the dimension, $\npoints$ is the number of points
	\item $p$ is the number of projections $(\theta_1, \cdots, \theta_p)$
	\item $\|\cdot\|_2$: Euclidean norm of $\R^n$
	\item Matrices $X \in \R^{\npoints\times d}$ are written $X = (x_1, \cdots, x_\npoints)^T$ with the $x_i \in \R^d$
	\item $\|Y\|_{\infty, 2}$ for $Y \in \R^{\npoints \times d}$ denotes $\max_{i\in \llbracket 1, \npoints \rrbracket}\|Y_{i,\cdot}\|_2 = \max_{i\in \llbracket 1, \npoints \rrbracket}\|y_i\|_2$
	\item $M \cdot N$: inner product $\mathrm{Trace}(M^TN)$ for matrices
	\item $\W_2$: 2-Wasserstein Distance \ref{eqn:W2}
	\item $\bbsigma$: Uniform measure on the unit sphere $\SS^{d-1}$ of $\R^d$
	\item $P_\theta$: for $\theta \in \SS^{d-1}$, $P_\theta = x\longmapsto \theta^Tx$
	\item $\SW_2$: Sliced 2-Wasserstein distance \ref{eqn:SW}
	\item $\SW_q$: Sliced $q$-Wasserstein distance 
	\item $\gamma_X$: for $X \in \R^{n\times d}$: discrete measure $\frac{1}{n}\sum_i\delta_{X_{i,\cdot}} = \frac{1}{n}\sum_i\delta_{x_i}$
	\item $\SWY(Y)$: $\SW_2^2(\gamma_Y, \gamma_Z)$ \ref{eqn:SWY}
	\item $\SWpY(Y)$: Monte-Carlo approximation of $\SWY(Y)$ with $p$ projections \ref{eqn:SWpY}
	\item $\Sigma_n$: $n$-simplex: $a \in \R_+^n$ such that $\sum_ia_i=1$
	\item $\|M\|_F$: Frobenius norm: $\sqrt{\sum_{i,j} M_{i,j}^2}$
	\item $\config$ denotes $p$ permutations $(\sigma_1, \cdots, \sigma_p)$ of $\llbracket 1, \npoints \rrbracket$, see \ref{sec:cells}
	\item $\mathcal{C}_\config$: cell of configuration $\config$, see \ref{sec:cells}
\end{itemize}}

	\section{Sliced and Empirical Sliced Wasserstein Energies and their Regularities}\label{sec:energies}

\subsection{The discrete SW energies \texorpdfstring{$\SWY$}{S} and \texorpdfstring{$\SWpY$}{Sp}}

The Sliced Wasserstein distance has been widely studied as an alternative to the Wasserstein distance, in particular it is arguably simpler to compute in order to minimise measure discrepancies. In practice, one may not work with continuous measures, which are beyond the capabilities of numerical approximations, thus one must sometimes contend with discrete measures. To that end, we study in this paper the SW distance between discrete measures, and in particular the associated energy landscape with respect to the support of one of the measures:
\begin{equation}\label{eqn:SWY}
	\SWY := \app{\R^{\npoints \times d}}{\R_+}{Y}{\Int{\SS^{d-1}}{}\mathrm{W_2^2}(P_{\theta} \#\gamma_Y, P_{\theta}\#\gamma_Z) \dd \bbsigma(\theta)},
\end{equation}
where $\npoints$ denotes the number of points in the data matrices $Y, Z$, which we write as data entries stacked vertically: $Y = (y_1, \cdots, y_\npoints)^T$, with points in $\R^d$. The associated (uniform) discrete measure supported on $\lbrace y_1, \cdots, y_\npoints \rbrace$ will be denoted $\gamma_Y := \frac{1}{\npoints}\sum_k \delta_{y_k}$.

For instance, this framework encompasses SW-based implicit generative models (\cite{deshpande_generative_sw}, ~\cite{Wu2019_SWAE}), which optimise parameters $\rho$ by minimising $\SW(T_\rho\#\mu_0, \mu)$, where $\mu_0$ is comprised of samples of a simple distribution, and $\mu$ corresponds to data samples which we would like to generate. In this setting, one would need to minimise \textit{through} $\SWY$.
The use of discrete measures is also backed theoretically by the study of the \textit{sample complexity} of SW~\cite{nadjahi_statistical_properties_sliced}, which is to say the rate of decrease of the approximation error between $\SW(\mu, \nu)$ and its discretised counterpart $\SW(\hat\mu_n, \hat\nu_n)$.

In practical and realistic settings,  the only numerically accessible workaround  to optimise through $\SWY$ is a form of discretisation of the set of directions. The first and most common method, due to its efficiency and simplicity, is to minimize $\SWY$ through stochastic gradient descent (SGD): at each time set $t$, $p$ random directions $(\theta_1^{(t)}, \cdots, \theta_p^{(t)})$ are drawn, and a gradient descent step is performed by approximating $\SWY$ by a discrete sum on these $p$ random directions. This method is optimisation-centric, since it does not concern itself with computing the final SW distance and focuses on optimising the parameters.  A second possible discretisation method consists in fixing the $p$ directions $(\theta_1, \cdots, \theta_p)$ once for all and replacing $\SWY$ in the minimization by its Monte-Carlo estimator~\footnote{In this notation the projection axes $\theta_1, \cdots, \theta_p \in \SS^{d-1}$ are written implicitly, the complete notation being $\SWpY(Y; (\theta_i)_{i \in \llbracket 1, p \rrbracket})$ when required.
}
\begin{equation}\label{eqn:SWpY}
  \SWpY := \app{\R^{\npoints \times d}}{\R_+}{Y}{\cfrac{1}{p}\Sum{i=1}{p}\mathrm{W_2^2}(P_{\theta_i} \#\gamma_Y. P_{\theta_i}\#\gamma_Z)}.
\end{equation}
It is important to note that both methods are intuitively tied, since in both cases there is a finite amount of sampled directions. If the SGD method lasts $T$ iterations with $p$ projections every time, it amounts to a specific way of optimising $\mathcal{E}_{pT}$. For this reason, studying $\SWpY$ theoretically is not only interesting in itself as an approximation of $\SWY$, but also yields a better insight on the SGD strategy.

The study of $\SWY$ is also tied with the study of the SW barycentres, which solve the optimisation problem
\begin{equation}\label{eqn:bar}
	\mathrm{Bar}(\lambda_j, \gamma_{Z^{(j)}})_{j \in \llbracket 1, J \rrbracket} = \underset{Y \in \R^{\npoints \times d}}{\argmin}\ \Sum{j=1}{J}\lambda_j\SWY(Y, Z^{(j)}) =: \SWYbar(Y),
\end{equation}
where the notation $\SWY(Y, Z^{(j)})$ reflects the dependency on $Z$ in the definition of $\SWY$ \ref{eqn:SWY}. The regularity and convergence results will immediately be applicable to the barycentre energy \ref{eqn:bar}. While the optimisation results on $\SWY$ and $\SWpY$ will not generalise naturally due to the sum, the SGD convergence results shall still hold.

As a Monte-Carlo estimator, the law of large numbers yields the point-wise convergence of $\SWpY$ to $\SWY$ if the $(\theta_i)_{i \in \N}$ are i.i.d. of law $\bbsigma$:
\begin{equation}\label{eqn:Sp_pwc}
	\SWpY(Y; (\theta_i)_{i \in \llbracket 1, p \rrbracket}) \xrightarrow[p\rightarrow +\infty]{\mathrm{a.s.}} \SWY(Y).
\end{equation}
\blue{For this reason, it is often assumed  that numerically, $\SWpY$ and $\SWY$ will
behave similarly, which is perhaps why research has been scarce on the landscape
of $\SWpY$, the focus remaining on the theoretical properties of the true or mini-batch Sliced
Wasserstein Distance \cite{nadjahi2019_guarantees_sw,nadjahi_bayesian}. But as
discussed in the introduction, practitioners often optimize the SW distance
using SGD with a finite number of projection directions \cite{kolouri2018slicedAE,deshpande_generative_sw}, and the landscape of $\SWpY$ is of paramount importance.
This section and the next one are dedicated to studying the relations and
differences between  $\SWpY$ and $\SWY$.}

\blue{\begin{remark}
	Some of our results can in fact be extended to $q$-SW instead of 2-SW, especially regularity results \ref{lemma:WC_stability}, \ref{prop:w_unif_locLip} and \ref{thm:locLip}, as well as the statistical estimation results \ref{thm:Sp_cvu_S} and \ref{thm:TCL_SW}. However, as soon as we need the \textit{cell structure} of $\SWpY$ (\ref{sec:cells}), we leverage the simplicity of the quadratic case $q=2$.
\end{remark}}

\subsection{Regularity properties of \texorpdfstring{$\SWpY$}{Sp} and \texorpdfstring{$\SWY$}{S}}

In order to study the regularity of our energies, we first focus on the regularity of $w_\theta$, the 2-Wasserstein distance between two discrete measures projected on the line $\R\theta$:
\begin{equation}\label{def:wtheta}
	w_\theta := \app{\R^{\npoints \times d}}{\R}{Y}{\W_2^2(P_\theta\#\gamma_Y, P_\theta\#\gamma_Z)}.
\end{equation}
With this notation, observe that $\SWY$ and $\SWpY$ can be written
\begin{equation}\label{eq:E_expectation_notation}
  \SWY(Y) =   \mathbb{E}_{\theta \sim \bbsigma}\left[w_\theta(Y)\right]\;\;\text{and}\quad \SWpY(Y) = \mathbb{E}_{\theta \sim \bbsigma_p }\left[w_\theta(Y)\right],
\end{equation}
where $\bbsigma_p := \cfrac{1}{p}\Sum{i=1}{p}\delta_{\theta_i}$ for $p$ fixed directions $ (\theta_1, \cdots, \theta_p) \in (\SS^{d-1})^p$.

We now provide an important regularity result about the uniformly locally Lipschitz property of the functions $(w_\theta)_{\theta}$, which will yield easily that our energies $\SWY$ and $\SWpY$ are also locally Lipschitz, a central property in the convergence study of particular SGD schemes on $\SWY$ and $\SWpY$ (see \ref{sec:interpolated_SGD}). To show this result on $(w_\theta)$, we need the following \ref{lemma:WC_stability}, whose proof is provided in~\ref{sec:WC_stability}. \blue{This result shows that the Wasserstein cost is regular in some sense with respect to the measure weights and the cost matrix, which will be helpful when studying the regularity of the functions $w_\theta$. }

\begin{lemma}[Stability of the Wasserstein cost]\ \label{lemma:WC_stability} Let $\alpha, \oll{\alpha}, \in
\Sigma_\npoints$, $\beta, \oll{\beta} \in \Sigma_m$  and $C, \oll{C} \in
\R_+^{\npoints\times m}$.
	Denote by $\W(\alpha, \beta; C) := \underset{\pi \in \Pi(\alpha, \beta)}{\inf}\ \pi \cdot C$ the cost of the discrete Kantorovich problem of cost matrix $C$ between the weights $\alpha, \beta$. We have the following Lipschitz-like \blue{inequalities, assuming $\alpha, \oll{\alpha}, \beta, \oll{\beta} > 0$ entry-wise:}
	\blue{\begin{equation}\label{eqn:wass_stability_main_text}
			\left|\W(\alpha, \beta; C) - \W(\oll{\alpha}, \oll{\beta}; \oll{C})\right| \leq \|C-\oll{C}\|_{\infty} + \|C\|_{\infty}( \|\alpha - \oll{\alpha}\|_1 + \|\beta - \oll{\beta}\|_1),
	\end{equation}}
	\blue{\begin{equation}\label{eqn:wass_stabilityL2_main_text}
			\left|\W(\alpha, \beta; C) - \W(\oll{\alpha}, \oll{\beta}; \oll{C})\right| \leq \|C-\oll{C}\|_{F} + \|C\|_F\left(\|\alpha - \oll{\alpha}\|_2+\|\beta - \oll{\beta}\|_2\right).
	\end{equation}}
\end{lemma}

\blue{\begin{remark}
	Using \ref{eqn:wass_stabilityL2_main_text} twice with $(C, \oll{C})$ and $(\oll{C}, C)$ yields a symmetric second term with a factor $\min(\|C\|_{\infty}, \|\oll{C}\|_{\infty})$ instead of $\|C\|_{\infty}$, and likewise for $\|\cdot\|_F$ with \ref{eqn:wass_stabilityL2_main_text}.
\end{remark}}

\blue{\begin{remark}
	The result of \ref{lemma:WC_stability} assumes \textit{positive} weights, but in the case of the $q$-Wasserstein cost $C_{i,j} = \|x_i-y_j\|_2^q$ with $q \geq 1$, we can remove this assumption by a continuity argument, since the $q$-Wasserstein distance metrises the weak convergence of measures (see \cite{santambrogio2015optimal}, Theorem 5.10 or 5.11, applied to the simple case of discrete measures for which convergence of moments is immediate).
\end{remark}}

The following regularity property on $(w_\theta)$ uses the norm $\|X\|_{\infty, 2} = \underset{k \in \llbracket 1, \npoints \rrbracket}{\max}\ \|x_k\|_2$ on $\R^{\npoints \times d}$. We also denote $D := \npoints \times d$ for convenience.
\begin{prop}
  	The $(w_\theta)_{\theta \in \SS^{d-1}}$ are uniformly locally Lipschitz.\label{prop:w_unif_locLip}. More precisely, in a neighbourhood $X \in \R^D$ or radius $r>0$, writing $\kappa_r(X) := \blue{2}\npoints(r + \|X\|_{\infty, 2} + \|Z\|_{\infty, 2})$, each $w_\theta$ is $\kappa_r(X)$ Lipschitz, which is to say
	$$\forall X \in \R^D,\; \forall Y, Y' \in B_{\|\cdot\|_{\infty, 2}}(X, r),\; \forall \theta \in \SS^{d-1},\; |w_\theta(Y) - w_\theta(Y')| \leq \kappa_r(X) \|Y-Y'\|_{\infty, 2}.$$

\end{prop}

\begin{proof}
	Let $X \in \R^D,\; Y,Y'  \in B_{\|\cdot\|_{\infty, 2}}(X, r)$, and $\theta \in \SS^{d-1}$. By~\ref{lemma:WC_stability} \blue{Equation \ref{eqn:wass_stabilityL2_main_text}}, we have \blue{$|w_\theta(Y) - w_\theta(Y')| \leq \|C-C'\|_F$}, where for $(k,l) \in \llbracket 1, \npoints \rrbracket^2,\; C_{k,l} := (\theta^T y_k - \theta^T z_l)^2$, likewise for $C'$. Then:
	\begin{align*}[C - C']_{k,l} &= \left(\theta^T (y_k - y_k')\right)\left(\theta^T (y_k + y_k' - 2z_l)\right)\\
		&\leq \|y_k-y_k'\|_2\|y_k + y_k' - 2z_l\|_2 \\
		&= \|y_k-y_k'\|_2	\|y_k - x_k + y_k' - x_k + 2z_l + 2x_k\|_2\\
		&\leq \|y_k-y_k'\|_2 \left(2r + 2\|Z\|_{\infty, 2} + 2\|X\|_{\infty, 2}\right).
	\end{align*}
	Finally, $\|C-C'\|_F = \sqrt{\Sum{k,l \in \llbracket 1, \npoints \rrbracket}{}[C - C']_{k,l}^2} \leq 2\npoints(r + \|X\|_{\infty, 2} + \|Z\|_{\infty, 2})\|Y-Y'\|_{\infty, 2}$.
\end{proof}
As a consequence, we deduce immediately that $\SWpY$ and $\SWY$ are locally Lipschitz.
\begin{theorem}\label{thm:locLip} $\SWY$ and $\SWpY$ are locally Lipschitz.
\end{theorem}

\begin{proof}
  Let $X \in \R^D,\ r>0$ and $\mu \in \{\bbsigma, \bbsigma_p\}$. By~\ref{prop:w_unif_locLip}, for any $Y, Y' \in B_{\|\cdot\|_{2, \infty}}(X, r)$,
  $$\left|\mathbb{E}_{\theta \sim \mu}\left[w_\theta(Y)\right] - \mathbb{E}_{\theta \sim \mu}\left[w_\theta(Y')\right] \right|  \le \mathbb{E}_{\theta \sim \mu}\left[ \left| w_\theta(Y) - w_\theta(Y') \right| \right] \le \kappa_r(X)\|Y-Y'\|_{\infty, 2}.$$
\end{proof}

As a locally Lipschitz function,  $\SWY$ is differentiable almost everywhere. The expression of its gradient is quite simple and corresponds  to the simple differentiation of $w_{\theta}$ in the integral, as was shown in~\cite{bonneel2015sliced}. We remind here their result for the sake of completeness, and because the derivative will be useful on several occasions in this paper.
We define $\mathcal{U}$ the open set of matrices with distinct lines
\begin{equation}\label{eqn:U}
	\mathcal{U} =  \left\lbrace (x_1, \cdots, x_\npoints)^T \in \R^{\npoints\times d}\ |\ \forall i \neq j,\; \llbracket 1, \npoints \rrbracket^2,\; x_i \neq x_j \right\rbrace.
\end{equation}

\begin{theorem}[Regularity of $\SWY$, from Bonneel et al.~\cite{bonneel2015sliced} Theorem 1]\label{thm:bonneel_diff}
		$\SWY$ is continuous on $\R^{\npoints \times d}$, and of class $\mathcal{C}^1$ on $\mathcal{U}$.
		There exists $\kappa \geq 1$ such that $\nabla \SWY$ is $\kappa$-Lipschitz on $\mathcal{U}$. For $Y \in \mathcal{U}$, one has the expression:
	\begin{equation}\label{eqn:gradS}
		\cfrac{\partial \SWY}{\partial y_k}(Y) = \cfrac{2}{\npoints}\Int{\SS^{d-1}}{}\theta \theta^T (y_k - z_{\sort{Z}{\theta} \circ (\sort{Y}{\theta})^{-1}(k)}) \dd \bbsigma(\theta),
	\end{equation}
	where for $\theta \in \SS^{d-1}, X \in \mathcal{U}$, $\sort{X}{\theta} \in \mathfrak{S}_\npoints$ is any permutations s.t. $\theta^T x_{\sort{X}{\theta}(1)} \leq \cdots \leq \theta^T x_{\sort{X}{\theta}(\npoints)}$.
\end{theorem}
Proving this theorem requires to be cautious. Firstly, differentiating directly under the integral using standard calculus theorems is impossible, since the integrand is only differentiable on a set $\mathcal{U}_\theta$ which depends on the integration variable $\theta$. Fortunately, these irregularities are smoothed out as $\theta$ rotates, yielding differentiability almost-everywhere. Secondly, the problematic term $\sort{Y}{\theta}$ can be dealt with for $Y\in \mathcal{U}$ by remarking that for any $Y'\ \varepsilon$-close to $Y$, we have $\sort{Y}{\theta} = \sort{Y'}{\theta}$ for every $\theta$ in a certain subset of $\SS^{d-1}$ which is of $\bbsigma$-measure exceeding $1 - C\varepsilon$. \blue{Regarding the multiplicative constant, Theorem 1 in Bonneel et. al omits the $1/\npoints$ factor (we believe that this is a typing error).}

\subsection{Cell structure of \texorpdfstring{$\SWpY$}{Sp}}\label{sec:cells}

In order to further study the optimisation properties of  $\SWpY$ and $\SWY$, we need to exhibit more explicitly the structure of the landscape of $\SWpY$. The semi-concavity of  $\SWpY$ and $\SWY$ will follow, as well as the fact that $\SWY$ is semi-algebraic \blue{(see \ref{prop:SWpY_semi_algebraic})}. We can compute $\SWpY$ by leveraging the formula for 1D Wasserstein distances:
\begin{equation}
	\forall Y \in \R^{\npoints\times d},\quad \SWpY(Y) = \cfrac{1}{\npoints p}\Sum{i=1}{p}\Sum{k=1}{\npoints} \left(\theta_i^T \Big(y_k - z_{\sort{Z}{\theta_i} \circ (\sort{Y}{\theta_i})^{-1}(k)}\Big)\right)^2.
\end{equation}
For now we consider $Z$ and the $(\theta_i)$ fixed, and we write $\config(Y) := \left(\config_i(Y)\right)_{i \in \llbracket 1, p \rrbracket}$ where $\config_i(Y) = \sort{Z}{\theta_i} \circ (\sort{Y}{\theta_i})^{-1}$.  Writing $\mathfrak{S}_\npoints $ the set of permutations of $\{1,\cdots,n\}$, $\config_i$ is the element $\sigma$ of  $\mathfrak{S}_\npoints $ which solves the (Monge) quadratic optimal transport between the points $(\theta_i^T y_1, \cdots, \theta_i^Ty_\npoints)$ and \blue{$(\theta_i^T z_1, \cdots, \theta_i^T z_\npoints)$.} The matching configuration $\config(Y)$ depends implicitly on the fixed directions $(\theta_i)$.

Note that the permutations $\sort{Y}{\theta}$ and \blue{$\sort{Z}{\theta}$} are not always uniquely defined: for any $\theta \in \SS^{d-1}$, there exists $Y \in \mathcal{U}$ such that $\sort{Y}{\theta}$ is not uniquely defined (take $Y$ such that $\theta \in (y_1 - y_2)^\perp$ for instance). However, for a given set of directions $(\theta_i)$, these permutations are uniquely defined almost everywhere on $\R^{n\times d}$.

A set of interest is $\mathcal{C}_\config = \left\lbrace Y \in \mathcal{U}\ |\  \config(Y) \;\text{is uniquely defined and equal to}\; \config \right\rbrace$, the cell of points $Y$ of configuration $\config$. Writing  $\config = (\config_1, \cdots, \config_p)$, and using the optimality of each $\config_i$, note that each cell $\mathcal{C}_\config$ can be also written as
\begin{equation}\label{eqn:opt_sigma_config_L_linear}
	\begin{split}
		\mathcal{C}_\config = \Bigg\{ Y \in \R^{\npoints \times d} : \forall i \in \llbracket 1, p \rrbracket,\; \forall \sigma \in \mathfrak{S}_\npoints \setminus \lbrace \config_i \rbrace,\\ \Sum{k=1}{\npoints}z_{\config_i(k)}^T\theta_i\theta_i^T y_k > \Sum{k=1}{\npoints}z_{\sigma(k)}^T\theta_i\theta_i^T y_k \rbrace\Bigg\}.
	\end{split}
\end{equation}
Thus, each $\mathcal{C}_\config$ is an open polyhedral cone, obtained as the intersection of  $p(\npoints! - 1)$ half-open planes.
Moreover, the union of these cells $\Reu{\config \in \mathfrak{S}_\npoints^p}{}\mathcal{C}_\config$ is a strict subset of $\mathcal{U}$ (as a consequence of the non uniqueness of the permutations for some $Y$), but is dense in $\R^{\npoints \times d}$. %
These cells are of particular interest since by definition, $\SWpY$ is quadratic on each $\mathcal{C}_\config$,  and can be written
\begin{equation}\label{eqn:cell_quadratic}
	\forall Y \in \mathcal{C}_\config, \; \SWpY(Y) = \cfrac{1}{\npoints p}\Sum{i=1}{p}\Sum{k=1}{\npoints} \left(\theta_i^T \big(y_k - z_{\config_i(k)}\big)\right)^2 =: q_\config(Y).
\end{equation}
Furthermore, the sorting interpretation of the 1D Wasserstein distance allows us to re-write $\SWpY(Y)$ as a minimum of quadratics,
\begin{equation}\label{eqn:min_quadratics}
	\forall Y \in \R^{\npoints\times d}, \; \SWpY(Y) = \underset{\config \in \mathfrak{S}_\npoints^p}{\min}\ q_\config(Y) = q_{\config(Y)}(Y).
\end{equation}

\begin{remark}\label{rem:quadratics}
	To each $Y  = (y_1, \cdots, y_\npoints)^T$ (seen as a $n\times d$ matrix), we can
	associate the column vector $\mathrm{vec}(y) := (y_1^T, \cdots, y_\npoints^T)^T$, which is now a vector of $\R^D = \R^{\npoints\times d}$ without any abuse of notation. We re-write the quadratic from equation~\ref{eqn:cell_quadratic} in standard form: $q_\config(\mathrm{vec}(y)) = \frac{1}{2}\mathrm{vec}(y)^TB\mathrm{vec}(y) - a_\config^T \mathrm{vec}(y) + b$, where:
	\begin{equation}
		\begin{split}
			B := \cfrac{2}{\npoints}\left(\begin{array}{ccc}
				A & 0 & 0 \\
				0 & \ddots & 0 \\
				0 & 0 & A
			\end{array}\right);\; A := \cfrac{1}{p}\Sum{i=1}{p}\theta_i\theta_i^T; \\ a_\config := \cfrac{2}{p\npoints}\left(\begin{array}{c}
				\Sum{i=1}{p}\theta_i\theta_i^Tz_{\config_i(1)} \\
				\vdots \\
				\Sum{i=1}{p}\theta_i\theta_i^Tz_{\config_i(\npoints)}
			\end{array}\right);\; b := \cfrac{1}{\npoints}\Sum{k=1}{\npoints}z_k^TAz_k.
		\end{split}
	\end{equation}
	Note in particular that only the linear component depends on $\config$.
\end{remark}

Finding the minimum of each quadratic $q_\config$ can be done in closed form, thanks to the computations of \ref{rem:quadratics}. This computational accessibility will be leveraged during our discussions on minimising $Y \longmapsto\SWpY(Y)$ (\ref{sec:Ep_crit_and_bcd}), wherein we shall present the Block Coordinate Descent method (\ref{alg:BCD}), which computes iteratively minima of quadratics in closed form.

\subsection{Consequences of the cell structure on the regularity of \texorpdfstring{$\SWpY$}{Sp} and \texorpdfstring{$\SWY$}{S}}

The cell decomposition presented in \ref{sec:cells} permits to show several additional regularity results.

\begin{prop}\label{prop:SWpY_quadratic_on_each_cell}$\SWpY$ is quadratic on each cell $\mathcal{C}_\config$, thus of class $\mathcal{C}^{\infty}$ on $\Reu{\config \in \mathfrak{S}_\npoints^p}{}\mathcal{C}_\config$, hence $\mathcal{C}^{\infty}$ a.e..
\end{prop}

The formulation as an infimum of quadratics also allows us to prove that $\SWpY$ is semi-concave, which is an extremely useful property for optimisation.

\begin{prop}$\SWpY$ is $\frac{1}{\npoints}$-semi-concave, i.e. $\SWpY - \frac{1}{\npoints}\|\cdot\|_2^2$ is concave.\label{prop:SWpY_semi_concave}
\end{prop}
\begin{proof}
	Using the notations from \ref{rem:quadratics}, $\SWpY(\mathrm{vec}(y)) = \frac{1}{2}\mathrm{vec}(y)^TB\mathrm{vec}(y) + \underset{\config \in \mathfrak{S}_\npoints^p}{\min}\ a_\config^T \mathrm{vec}(y) + b$. Now, $\mathrm{vec}(y) \longmapsto \underset{\config \in \mathfrak{S}_\npoints^p}{\min}\ a_\config^T \mathrm{vec}(y) + b$ is concave, as an infimum of affine functions. Furthermore
	$$\frac{1}{2}\mathrm{vec}(y)^TB\mathrm{vec}(y) - \frac{1}{\npoints}\|\mathrm{vec}(y)\|_2^2 = \cfrac{1}{\npoints}\Sum{k=1}{\npoints}y_k^T(A - I)y_k,$$
	and since $A \preceq I_d$, the equation above defines a concave function of $\mathrm{vec}(y)$.
\end{proof}

The semi-concavity of $\SWpY$ and point-wise convergence allows us to deduce the semi-concavity of $\SWY$:
\begin{prop}$\SWY$ is $\frac{1}{\npoints}$-semi-concave.\label{prop:SWY_semi_concave}
\end{prop}

\begin{proof}
	By \ref{prop:SWpY_semi_concave}, $\forall p \in \N^*,\; \SWpY$ is $\frac{1}{\npoints}$-semi-concave. Let $p \in \N^*,\; Y, Y'\in \R^{\npoints \times d}$ and $\lambda \in [0, 1]$. We have
	\begin{equation*}
		\begin{split}
			\SWpY((1-\lambda)Y + \lambda Y') - \frac{1}{\npoints}\|(1-\lambda)Y + \lambda Y'\|_F^2 \\\geq (1-\lambda)\SWpY(Y) + \lambda \SWpY(Y') - \frac{1}{\npoints}\left((1-\lambda)\|Y\|_F^2 + \lambda \|Y'\|_F^2\right).
		\end{split}
	\end{equation*}
	Taking the limit $p \longrightarrow +\infty$ in this inequality yields the $\frac{1}{\npoints}$-semi-concavity of $\SWY$.
\end{proof}

The cell formulation also allows us to show that $\SWpY$ is semi-algebraic, which means that it can be written using a finite number of polynomial expressions. This result induces strong optimisation results akin to semi-concavity for our purposes. 	\blue{We recall the definition of a semi-algebraic set (\cite{Wakabayashi_semialgebraic}, Definition 1). $S \subset \R^D$ is \textit{semi-algebraic} if it can be written $S = \Reu{n=1}{N}\Inter{m=1}{M}A_{n,m}$  where $(A_{n, m})_{\substack{n \in \llbracket 1, N \rrbracket \\ m \in \llbracket 1, M \rrbracket}}$ is a finite family of sets such that $A_{n,m} = \lbrace X \in \R^D\ |\ p_{n,m}(X) \geq 0\rbrace$ or $A_{n,m} = \lbrace X \in \R^D\ |\ p_{n,m}(X) = 0\rbrace$, with $p_{n,m}$ being $D$-variate polynomials with real coefficients. A \textit{semi-algebraic} function is a function whose graph is a semi-algebraic set.}

\begin{prop}$\SWpY$ is semi-algebraic.\label{prop:SWpY_semi_algebraic}
\end{prop}

\begin{proof}
    We shall prove that the set $G := \left\lbrace (X, \SWpY(X))\ |\ X \in \mathcal{U}_\Theta\right\rbrace$ is semi-algebraic, where $\mathcal{U}_\Theta := \Reu{\config \in \mathfrak{S}_\npoints^p}{}\mathcal{C}_\config$. Observe that 
    $$G = \Reu{\config \in \mathfrak{S}_\npoints^p}{}\left\lbrace (X, y) \in \R^{D+1}, X \in \mathcal{C}_\config \text{ and } y = q_\config(X)\right\rbrace.$$
    
	The function $q_\config$ is quadratic, thus polynomial, and the cells  $\mathcal{C}_\config$ are intersections of a finite number of half planes, so we conclude that $G$ is semi-algebraic.

	The closure of $\mathcal{U}_\Theta$ verifies $\oll{\mathcal{U}_\Theta} = \R^D$, furthermore, since $\SWpY$ is continuous on $\R^D$ (by~\ref{thm:bonneel_diff}), the closure of $G$ is exactly the graph of $\SWpY$. Now by~\cite{Wakabayashi_semialgebraic}, Lemma 4, since $G$ is semi-algebraic, then $\oll{G}$ is also semi-algebraic. As a conclusion, $\SWpY$ is a semi-algebraic function.
\end{proof}

\subsection{Convergence of \texorpdfstring{$\SWpY$}{SWpY} to \texorpdfstring{$\SWY$}{SWY}}

We have already seen that $\SWpY(Y)$ converges to $\SWY(Y)$ almost surely when  $p\rightarrow +\infty$.
In practice, since we want to optimise through $\SWpY$ as a surrogate for $\SWY$, we would wish for the strongest possible convergence. Below, we show almost-sure \textit{uniform} convergence over any compact, which is substantially better than point-wise convergence. Note that this stronger mode of convergence is unfortunately still too weak to transport local optima properties.

\begin{theorem}[Uniform Convergence of $\SWpY$]\label{thm:Sp_cvu_S}\ Let $\mathcal{K} \subset \R^{\npoints\times d}$ compact. We have

	$\P\left(\|\SWpY-\SWY\|_{\ell^{\infty}(\mathcal{K})} \xrightarrow[p \rightarrow +\infty]{} 0\right) = 1,\;$ where for $f \in \mathcal{C}(\mathcal{K} , \R),\; \|f\|_{\ell^{\infty}(\mathcal{K})} := \underset{x \in \mathcal{K} }{\sup}\ |f(x)|$.
\end{theorem}

\begin{proof}
	We shall temporarily write $\SWpY(Y) = \SWpY(Y; \Theta)$ to illustrate the dependency on the random variable $\Theta := (\theta_i)_{i \in \N^*}$ on a probabilistic space $(\Omega, \mathcal{A}, \P)$ with values in $(\SS^{d-1})^\N$.
	By point-wise almost-sure convergence, for any fixed $Y \in \R^{\npoints\times d}$, there exists a $\P$-null set $\mathcal{N}_Y$ such that for every $\omega \in \Omega \setminus \mathcal{N}_Y$, the deterministic real number $\SWpY(Y; \Theta(\omega))$ converges to $\SWY(Y)$.
	Let $\mathcal{D} := \mathcal{K}  \cap \Q^{\npoints\times d}$, which is dense in $\mathcal{K} $ and countable. Let $\mathcal{N} := \Reu{Y \in \mathcal{D}}{}\mathcal{N}_Y$: $\mathcal{N}$ is $\P$-null as a countable union of $\P$-null sets.

	Fixing $\omega \in \Omega \setminus \mathcal{N}$, we have $\forall Y \in \mathcal{D}, \quad \SWpY(Y; \Theta(\omega)) \xrightarrow[p\rightarrow+\infty]{\quad} \SWY(Y)$, thus point-wise convergence on $\mathcal{D}$ of the (now) deterministic function $\SWpY(\cdot; \Theta(\omega))$ to $\SWY$. Now, a consequence of \ref{prop:w_unif_locLip} is that the family of functions $\left(Y\mapsto \SWpY( Y ; \Theta')\right)_{\Theta' \in (\SS^{d-1})^p}$ is equi-continuous on any compact (thus on $\mathcal{K}$). As a consequence, the point-wise convergence on $\mathcal{D}$ implies the uniform convergence of $\SWpY(\cdot; \Theta(\omega))$ to $\SWY$ on $\oll{\mathcal{D}} = \mathcal{K}$ (a detailed presentation of this classic result can be found in \cite{levy2012elements}, Proposition 3.2). This holds for any event $\omega \in \Omega \setminus \mathcal{N}$, with $\P(\Omega \setminus \mathcal{N}) = 1$, thus the uniform convergence is almost-sure.
\end{proof}

\blue{To complement this uniform almost-sure convergence, we prove a uniform Central Limit result on the error process $\sqrt{p}(\SWpY - \SWY)$ on a fixed compact set $\mathcal{K}$. This result provides insight on the law of the approximation error, uniformly with respect to the position $Y \in \mathcal{K}$.}

\blue{\begin{theorem}\label{thm:TCL_SW}
	Let $\mathcal{K} \subset \R^{\npoints\times d}$ be a compact and non-empty set. On this domain, we have the following uniform Central Limit convergence in distribution of the approximation error of the random process $\SWpY$:
	\begin{equation}\label{eqn:SW_TCL}
		\sqrt{p}(\SWpY - \SWY) \xrightarrow[p \longrightarrow +\infty]{\mathcal{L},\; \ell^{\infty}(\mathcal{K})} G,
	\end{equation}
	where the convergence is in law in the sense of $\ell^{\infty}(\mathcal{K})$, the space of bounded functions $z: \mathcal{K} \longrightarrow \R$ equipped with the uniform norm. The limit process $G$ is the centred Gaussian process on $\mathcal{K}$ of covariance 
	$$\C(G)[Y, Y'] = \Int{\SS^{d-1}}{}w_\theta(Y)w_\theta(Y')\dd\bbsigma(\theta) - \SWY(Y)\SWY(Y').$$
	This result implies the convergence in law of the uniform error
	\begin{equation}\label{eqn:SW_TCL_norm_inf}
		\sqrt{p}\|\SWpY - \SWY\|_{\ell^{\infty}(\mathcal{K})} \xrightarrow[p\longrightarrow+\infty]{\mathcal{L}} \|G\|_{\ell^{\infty}(\mathcal{K})}.
	\end{equation}
\end{theorem}}

\blue{We provide the proof of \ref{thm:TCL_SW} in \ref{sec:proof_donsker}, along with a brief presentation of the Donsker class arguments at hand.

\begin{remark}
	Our Central Limit result from \ref{thm:TCL_SW} allows one to build (uniform) confidence intervals for the approximation $\SWpY \approx \SWY$ on any compact, but is of limited practical interest due to the complexity of estimating the Gaussian process $G$. Nevertheless, such confidence intervals provide additional theoretical insight on the Monte-Carlo approximation of the discrete SW distance.
\end{remark}}

\blue{Our result \ref{eqn:SW_TCL} complements a result by Xi and Niles-Weed \cite{xi2022distributional}, which shows the following distributional convergence of a related process which is a function of $\theta$: 
$$\mathbb{H}_\npoints := \lbrace \sqrt{\npoints}\left(\W_q^q(P_\theta\#\hat\mu_\npoints, P_\theta\#\hat\nu_\npoints) - \W_q^q(P_\theta\#\mu, P_\theta\#\nu)\right), \theta \in \SS^{d-1} \rbrace \xrightarrow[\npoints\longrightarrow+\infty]{\mathcal{L},\; \ell^{\infty}(\SS^{d-1})} \mathbb{H},$$
where $\mu, \nu$ are compactly supported probability measures, and $\hat\mu_\npoints, \hat\nu_\npoints$ are discrete empirical versions supported on $\npoints$ samples of respective laws $\mu, \nu$, and $\mathbb{H}$ is a centred Gaussian process on $\SS^{d-1}$.

The distributional convergence in \ref{eqn:SW_TCL} also complements a (quantified) convergence in probability by Xu and Huang \cite{xu2022central}. For $q\geq 1$ and $\mu,\nu \in \mathcal{P}_q(\R^d)$, let $M_q(\mu) := (\int \|x\|^q\dd\mu)^{1/q}$ and $L:=q\W_q^{q-1}(\mu, \nu)(M_q(\mu) + M_q(\nu))$. In \cite{xu2022central}, they prove (Proposition 4) that for any $\varepsilon, \delta > 0$, if $p \geq \frac{2L^2}{(d-1)\varepsilon^2}\log(\frac{2}{\delta})$, then
\begin{equation}\label{eqn:xu_huang_probability}
	\P(|\widehat{\SW}_{q,p}^q(\mu, \nu) - \SW_{q}^q(\mu, \nu)|\geq \varepsilon) \leq \delta,
\end{equation}
where $\widehat{\SW}_{q,p}^q(\mu, \nu)$ is the Monte-Carlo approximation of $\SW_{q}^q(\mu, \nu)$ with $p$ projections. Their result is of a different nature, since it deals with general measures in $\mu,\nu \in \mathcal{P}_q(\R^d)$ and does not study the process associated to moving the support of one of the measures. Point-wise, \ref{eqn:xu_huang_probability} from \cite{xu2022central} is a more general result than \ref{eqn:SW_TCL_norm_inf}, but the strength of our result comes from the study of the \textit{process} $\SWpY$, for which \ref{eqn:xu_huang_probability} is not informative in distribution, since it is a point-wise result. Furthermore, \ref{eqn:xu_huang_probability} is not tailored to our almost-sure uniform convergence case \ref{thm:Sp_cvu_S}.}

\subsection{Illustration in a simplified case}\label{sec:L2}

Let us illustrate $\SWY$ in a simple case, that was briefly presented in Bonneel et al.~\cite{bonneel2015sliced}, in order to grasp the difficulties at hand. This example is interesting for understanding the difficulty of performing computations with $\SWY$ and $\SWpY$. Let $z_1 = (0,-1)^T$ and $z_2 = (0,1)^T$. Instead of computing $\SWY(Y)$ for any $Y \in \R^{2\times 2}$, we simplify by assuming $Y = (y, -y)^T = (y_1, y_2)^T$. We will assume further $y\neq 0$ and write $y =(u,v)^T$. The interested reader may seek the computations in \ref{sec:L2_E}. With these notations, we can show that
\begin{equation}
	\SWY(Y) = \SW_2^2(\gamma_Y, \gamma_Z)=\frac{u^2+v^2}{2} + \frac{1}{2} - \frac{2}{\pi}\left(|u| + |v|\Arctan\left|\frac{v}{u}\right|\right).
\end{equation}
For $\W_2^2$, one may show (see \ref{sec:L2_W} for the computations) that $\W_2^2(\gamma_Y, \gamma_Z) = u^2 + (|v|-1)^2$ in this setting.
We compare $\SWY$ and $\W_2^2$ in ~\ref{fig:comp_sym}.
\begin{figure}
	\centering
	\begin{subfigure}[c]{0.4\textwidth}
		\centering
		\includegraphics[width=\linewidth]{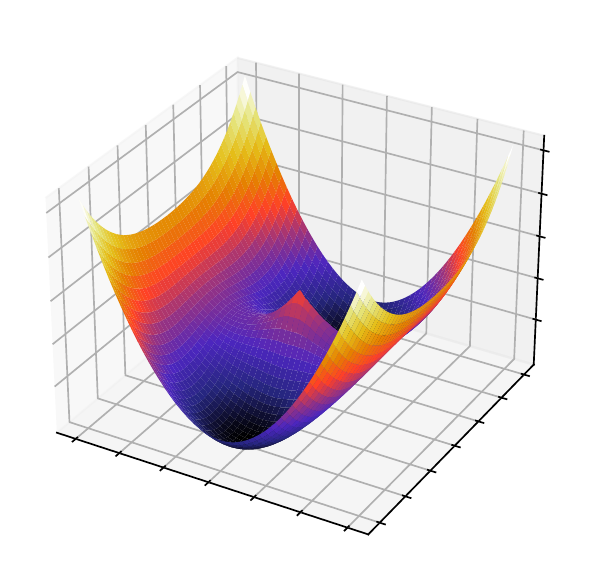}
		\subcaption{$Y \longmapsto\SWY(Y)$}
	\end{subfigure}
	\begin{subfigure}[c]{0.4\textwidth}
		\centering
		\includegraphics[width=\linewidth]{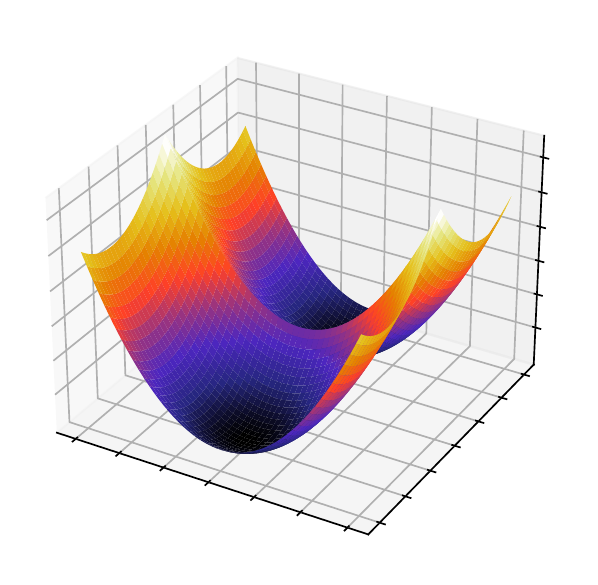}
		\subcaption{$Y \longmapsto\W_2^2(\gamma_Y, \gamma_Z) $}
	\end{subfigure}
	\caption{Comparison between Sliced Wasserstein (a) and Wasserstein (b) landscapes for 2-point discrete measures $Y = (y, -y)^T $ and $Z = (z_1,z_2)^T$ with $z_1 = (0,-1)^T$ and $z_2 = (0,1)^T$.}
	\label{fig:comp_sym}
\end{figure}

Notice differences in regularity. $\SWY$ is smooth on the open set $\mathcal{U}$ (defined in \ref{eqn:U}) of the $Y\in \R^{\npoints \times d}$ with distinct points (this is known in general,~\cite{bonneel2015sliced}), but is not differentiable anywhere in $\mathcal{U}^c$. Here this is clear at $(0, 0)$. Furthermore, $\SWY$ presents two saddle points, $(\pm\frac{2}{\pi}, 0)$. In \ref{sec:E_crit}, we shall study the critical points of $\SWY$ in full generality. Finally, $\W_2^2$ presents the typical landscape of the minimum of two quadratics.

We now move to computing $\SWpY$ in this setting. In the case $\npoints=2$, a significant simplification occurs since $\mathfrak{S}_2 = \lbrace I, (2, 1)\rbrace$, and we express a simple formula for the cells in the Appendix, see \ref{sec:L2_Ep}. We illustrate the cell structure in~\ref{fig:SWpY_sym_p3}.

\begin{figure}
	\centering
	\begin{tabular}{cccc}
		\begin{subfigure}[c]{0.22\textwidth}
			\centering
			\includegraphics[width=\linewidth]{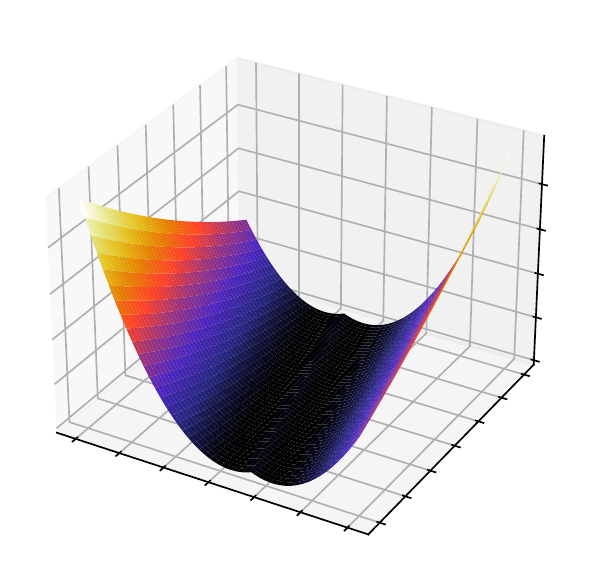}
		\end{subfigure}
		&\begin{subfigure}[c]{0.22\textwidth}
			\centering
			\includegraphics[width=\linewidth]{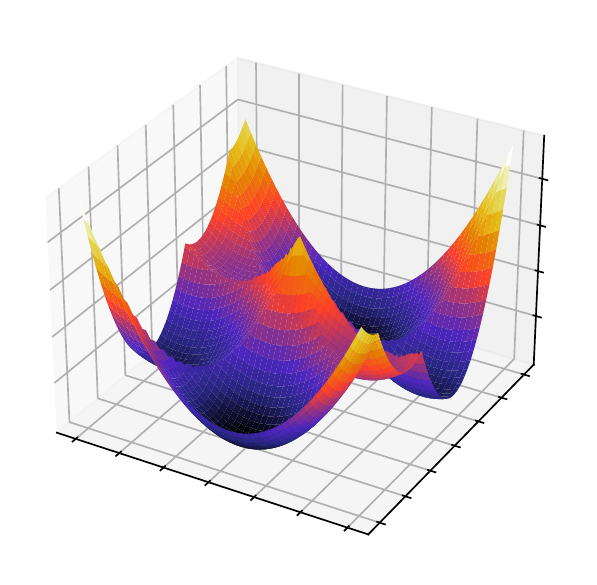}
		\end{subfigure}
		&\begin{subfigure}[c]{0.22\textwidth}
			\centering
			\includegraphics[width=\linewidth]{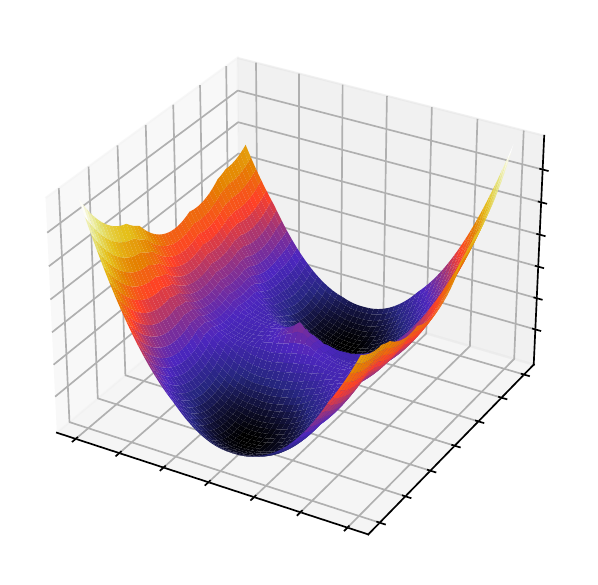}
		\end{subfigure}
		&\begin{subfigure}[c]{0.22\textwidth}
			\centering
			\includegraphics[width=\linewidth]{figures/SW_2points_sym_CMRmap.pdf}
		\end{subfigure}\\
		\begin{subfigure}[c]{0.15\textwidth}
			\centering
			\includegraphics[width=\linewidth]{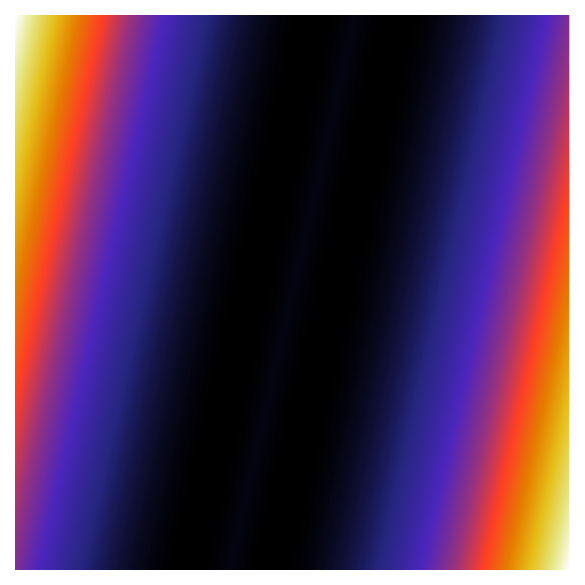}
			\subcaption{$\SWpY,\; p=1$}
		\end{subfigure}
		&\begin{subfigure}[c]{0.15\textwidth}
			\centering
			\includegraphics[width=\linewidth]{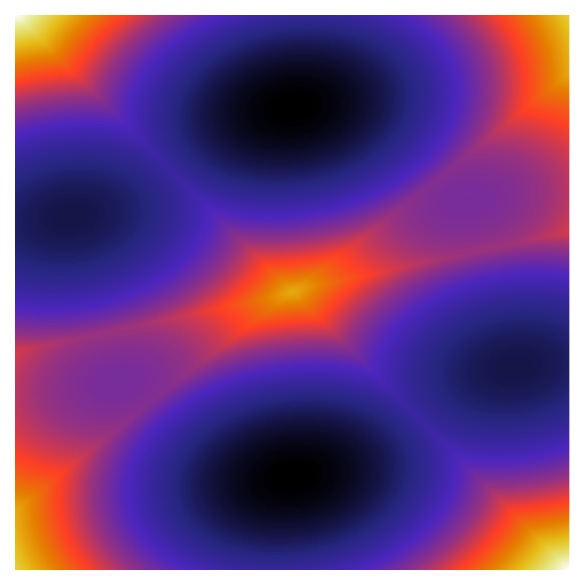}
			\subcaption{$\SWpY,\; p=3$}
		\end{subfigure}
		&\begin{subfigure}[c]{0.15\textwidth}
			\centering
			\includegraphics[width=\linewidth]{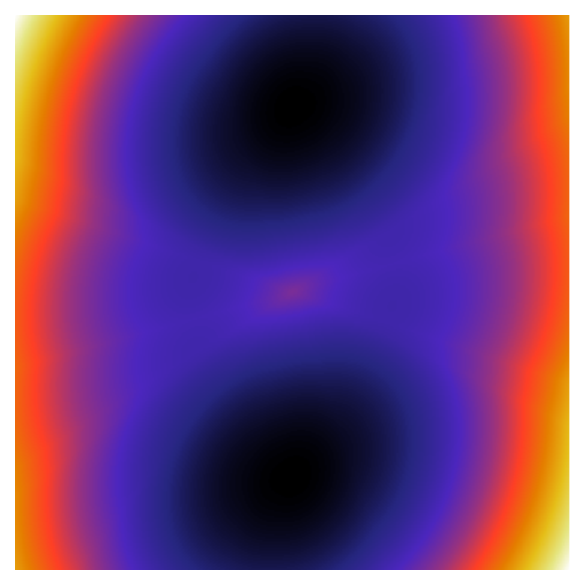}
			\subcaption{$\SWpY,\; p=10$}
		\end{subfigure}
		&\begin{subfigure}[c]{0.15\textwidth}
			\centering
			\includegraphics[width=\linewidth]{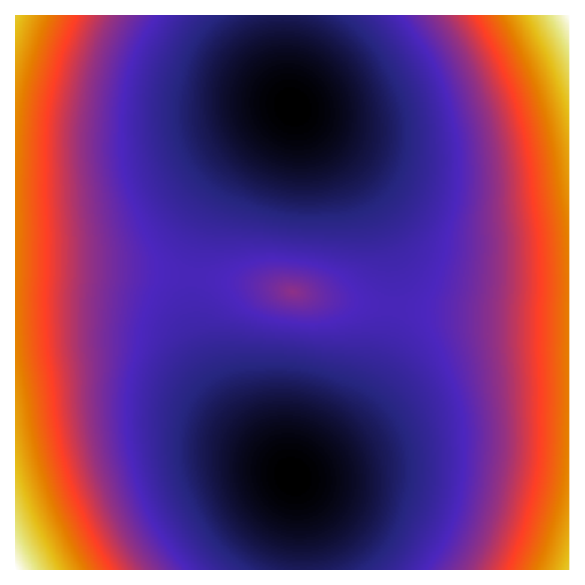}
			\subcaption{$\SWY$}
		\end{subfigure}
	\end{tabular}
	\caption{The landscape $\SWpY$ approaches $\SWY$ as $p$ increases, but introduces numerous strict local optima. Notice that when $p$ is too small ($p=1\leq d$ in particular), $\SWpY$ even introduces other global optima.}
	\label{fig:SWpY_sym_p3}
\end{figure}

Notice that as $p$ increases, the number of new strict local optima also increases, however their associated cells become very small, thus one may hope that the probability of ending up in a strict local optimum would decrease as $p$ increases. Specifically, in the heatmap visualisation, one may notice 6 large cells for $p=3$, and for $p=10$, two large cells corresponding to the global optima, and 8 small cells which may present local optima. This observation suggests that as $p\longrightarrow +\infty$, the total size of cells containing local optima decreases, and thus the probability of a numerical scheme converging to a local optimum decreases as well. Moreover, it is clear for the landscape $\SWpY$ with $p=3$ that the critical points (points of differentiability with a null gradient) are exactly the minima of the cell quadratics. Remark that a cell may not contain the minimum of its quadratic, which is why we will refer to cells containing their minimum as "stable" (as is the case for all cells in $p=3$ illustration, but seemingly not for $p=10$). 

As is suggested by \ref{fig:SWpY_sym_p3}, even with a large number of projections $p$ compared to the dimension $d$, the presence of strict local optima may prevent numerical solvers from converging to the global optimum $\gamma_Y = \gamma_Z$. This practical concern motivates the study of the landscapes $\SWY$ and $\SWpY$, which is the topic of \ref{sec:crit}.

	\section{Properties of the Optimisation Landscapes of \texorpdfstring{$\SWY$}{S} and \texorpdfstring{$\SWpY$}{Sp}}\label{sec:crit}

The goal of this section is to study the respective landscapes of $\SWY$ and $\SWpY$, their critical points and the links between them. 

\subsection{Optimising \texorpdfstring{$\SWY$}{S}}

\subsubsection{Global optima of \texorpdfstring{$\SWY$}{S}}

As its name suggests, the SW distance is indeed a distance on $\mathcal{P}_2(\R^d)$ (this result can be proven in the same manner for the $q$-SW distances, for $q \geq 1$).

\begin{prop}[Bonnotte~\cite{bonnotte}, Theorem 5.1.2]\label{prop:SW_distance} SW is a distance on $\mathcal{P}_2(\R^d)$.
\end{prop}
As a consequence, the global optima of $\SWY$ are exactly the points $Y^*$ such that $\gamma_{Y^*} = \gamma_{Z}$, or said otherwise the points such that $(y_1^*, \cdots, y_\npoints^*)$ is a permutation of $(z_1, \cdots, z_\npoints)$.

\subsubsection{Critical points of \texorpdfstring{$\SWY$}{SWY}}\label{sec:E_crit}

A first step in studying the landscape $\SWY$ is to determine its critical points, which we define as the set of points $Y$ where $\SWY$ is differentiable and $\nabla \SWY(Y) = 0$.
Thanks to \ref{thm:bonneel_diff}, these critical points can be shown to satisfy a fixed point equation. 

\begin{corollary}[Equation characterising the critical points of $\SWY$]\label{cor:crit_points_S}\
  Let $Y \in \mathcal{U}$ (defined in \ref{eqn:U}). For $(k,l) \in \llbracket 1, \npoints \rrbracket$, define $\Theta_{k,l}^{Y,Z} := \left\lbrace \theta \in \SS^{d-1}\ |\ \sort{Z}{\theta} \circ (\sort{Y}{\theta})^{-1} (k) = l \right\rbrace \subset \SS^{d-1}$ and $S_{k,l}^{Y,Z} := d\Int{\Theta_{k,l}^{Y,Z}}{}\theta \theta^T \dd \bbsigma \in S_d^+(\R)$.
  $Y$ is a critical point of $\SWY$ iif $Y$ satisfies
	\begin{equation}\label{eqn:crit_fixed_point}
		\forall k \in \llbracket 1, \npoints \rrbracket, \; y_k = \Sum{l=1}{\npoints}S_{k,l}^{Y,Z}z_l.
	\end{equation}
\end{corollary}

\begin{proof}
	Let $k \in \llbracket 1, \npoints \rrbracket$. We have $\SS^{d-1} = \Reu{l=1}{\npoints}\Theta_{k,l}^{Y, Z}$, where the union is disjoint, therefore one may write
	\begin{align*}
		\cfrac{\partial \SWY}{\partial y_k}(Y) &= \cfrac{2}{\npoints}\Int{\SS^{d-1}}{}\theta \theta^T (y_k- z_{\sort{Z}{\theta} \circ (\sort{Y}{\theta})^{-1}(k)})  \dd \bbsigma(\theta) \\
		&=  \cfrac{2}{\npoints}\Sum{l=1}{\npoints}\Int{\Theta_{k,l}}{}\theta \theta^T (y_k - z_l) \dd \bbsigma(\theta) = \cfrac{2}{d\npoints} y_k - \cfrac{2}{d\npoints}\Sum{l=1}{\npoints}S_{k,l}^{Y,Z}z_l,
	\end{align*}
where we have used $\Int{\SS^{d-1}}{}\theta \theta^T \dd \bbsigma(\theta) = I/d$ in the last equality.
	Equating the partial differential to 0 yields~\ref{eqn:crit_fixed_point}.
\end{proof}
Equation~\ref{eqn:crit_fixed_point} shows that the critical points can be written as combinations of the points $(z_l)$, "weighted" by the normalised conditional covariance matrices $S_{k,l}^{Y,Z} = d\mathbb{E}_{\theta \sim \bbsigma}\left[\mathbbold{1}(\theta \in \Theta_{k,l}^{Y,Z})\theta\theta^T\right]$.
With $\Psi := \app{\mathcal{U}}{\R^{\npoints\times d}}{Y}{\left(\begin{array}{c}
	\Sum{l=1}{\npoints}z_l^TS_{1,l}^{Y,Z} \\
	\vdots \\
	\Sum{l=1}{\npoints}z_l^TS_{\npoints,l}^{Y,Z}
\end{array}\right)}$,
Equation \ref{eqn:crit_fixed_point} writes as a fixed-point equation $Y = \Psi(Y)$.

Further notice that $\Psi$ cannot be properly defined on $\mathcal{U}^c$, for instance if $\npoints=2$, and if $Y = (y,y)$, the two possible sorting choices $\sort{Y}{\theta} \in \lbrace (1, 2), (2, 1)\rbrace$ yield two different values for $\Psi(Y)$ (the first value is the second with the indices exchanged).
We show below that $\Psi$ is continuous on $\mathcal{U}$. Unfortunately, $\Psi$ cannot be extended to the whole space $\R^{\npoints \times d}$, since the restrictions $\Psi|_{C_\config}$ may have distinct limits at the borders of the cells.

\begin{prop}[Regularity of $\Psi$]\label{prop:Fcont}\  $\Psi$ is continuous on $\mathcal{U}$ (defined in \ref{eqn:U}).
\end{prop}
\begin{proof}
  It is sufficient to prove the continuity of $G := Y \rightarrow S_{k,l}^{Y,Z}$ on $ \mathcal{U}$, for  $k,l$ fixed.
	Let $Y \in \mathcal{U}$ and $\varepsilon > 0$. Define 
	\begin{equation}
		\Theta_\varepsilon(Y) := \left\lbrace \theta \in \SS^{d-1} \ |\ \forall \delta Y \in B(0, \varepsilon),\ \left(\theta^T y_{\sort{Y}{\theta}(k)} + \theta^T \delta y_{\sort{\delta Y}{\theta}(k)}\right)_{k \in \llbracket 1, \npoints \rrbracket} \in \mathcal{U}_{\npoints, 1} \right\rbrace,
	\end{equation}
	
	with $\mathcal{U}_{\npoints, 1}$ the open set of lists $(x_1, \cdots, x_\npoints) \in \R^{\npoints}$ with distinct entries.
	By Bonneel et al.~\cite{bonneel2015sliced}, Appendix A, Lemma 2, $\forall \theta \in \Theta_\varepsilon(Y),\; \forall \delta Y \in B(0, \varepsilon),\; \sort{Y}{\theta} = \sort{Y+\delta Y}{\theta}$. Let $\varepsilon$ small enough such that $\forall \delta Y \in B(0,\varepsilon),\; Y+\delta Y \in \mathcal{U}$. Let $\delta Y \in B(0,\varepsilon)$. Separating the integral yields:
	\begin{align*}
		G(Y+\delta Y) &= \Int{\Theta_{k,l}^{Y+\delta Y,Z}}{}\theta\theta^T \dd\bbsigma(\theta) \\ 
		&=\Int{\Theta_{k,l}^{Y+\delta Y,Z} \bigcap \Theta_\varepsilon(Y)}{}\theta\theta^T \dd\bbsigma(\theta) + \Int{\Theta_{k,l}^{Y+\delta Y,Z} \bigcap \Theta_\varepsilon(Y)^c}{}\theta\theta^T \dd\bbsigma(\theta).
	\end{align*}
	Using the fact that $\Theta_{k,l}^{Y+\delta Y,Z} \bigcap \Theta_\varepsilon(Y) = \Theta_{k,l}^{Y,Z} \bigcap \Theta_\varepsilon(Y)$, and denoting $\|\cdot\|_{\mathrm{op}}$ the $\|\cdot\|_2$-induced operator norm on $\R^{d \times d}$, we get
	\begin{align*}
		G(Y+\delta Y) - G(Y) &= \Int{\Theta_{k,l}^{Y+\delta Y,Z} \bigcap \Theta_\varepsilon(Y)^c}{}\theta\theta^T \dd\bbsigma(\theta) - \Int{\Theta_{k,l}^{Y,Z} \bigcap \Theta_\varepsilon(Y)^c}{}\theta\theta^T \dd\bbsigma(\theta), \\
		\|G(Y+\delta Y) - G(Y)\|_{\mathrm{op}} &\leq \Int{\Theta_{k,l}^{Y+\delta Y,Z} \bigcap \Theta_\varepsilon(Y)^c}{}\left\|\theta\theta^T\right\|_{\mathrm{op}} \dd\bbsigma + \Int{\Theta_{k,l}^{Y,Z} \bigcap \Theta_\varepsilon(Y)^c}{}\left\|\theta\theta^T\right\|_{\mathrm{op}} \dd\bbsigma \\
		&\leq 2\Int{ \Theta_\varepsilon(Y)^c}{}1 \dd\bbsigma = 2\bbsigma(\Theta_\varepsilon(Y)^c).
	\end{align*}
	By Bonneel et al.~\cite{bonneel2015sliced}, Appendix A, Lemma 3, there exists a constant $C$ such that $\bbsigma(\Theta_\varepsilon(Y)^c) \leq C \varepsilon$, which  proves the continuity of $G$ on $\mathcal{U}$.
\end{proof}

\subsection{Optimising \texorpdfstring{$\SWpY$}{Sp}}

\subsubsection{Global optima of \texorpdfstring{$\SWpY$}{Sp}}

We saw in~\ref{prop:SW_distance} that SW is a distance. Unfortunately, its discretised version $\SWMC$ is only a pseudo-distance: \blue{non-negativity, symmetry and the triangular inequality still hold, but separation fails.}

For generic measures, a measure-theoretic way of seeing this is through characteristic functions. Given $\mu, \nu \in \mathcal{P}_2(\R^d)$ and $(\theta_1, \cdots, \theta_p) \in (\SS^{d-1})^p$, the condition $\SWMC(\mu, \nu) = 0$ is equivalent to $\forall i \in \llbracket 1, p \rrbracket,\; \forall t \in \R,\; \phi_\mu(t\theta_i) = \phi_\nu(t\theta_i)$, where $\phi_\mu$ (resp. $\phi_\nu$) is the characteristic function of $\mu$ (resp. $\nu$). 
This condition only constrains the characteristic functions on $p$ radial lines, and Bochner or P\'olya-type criteria may be considered to find a characteristic function $\phi$ which equals $\phi_\mu$ on these lines but differs on a non-null set.

The discrete case pertains more to our setting. As shown in~\cite{tanguy2023reconstructing}, for $p$ large enough, almost-sure separation holds. This result can be proven by leveraging the geometrical consequences of the constrains $P_{\theta_i}\#\gamma_Y = P_{\theta_i}\#\gamma_Z$, and determining the a.s. solution set using random affine geometry.

\begin{theorem}[\cite{tanguy2023reconstructing}, Theorem 4]\label{thm:SW_insufficient_projections}
	Let {$\gamma_Z := \Sum{l=1}{\npoints}b_l\delta_{z_l}$}, where the $(z_l)$ are fixed and distinct.	Assuming $\theta_1, \cdots, \theta_p \sim \bbsigma^{\otimes p}$ \blue{and $\npoints \geq 2$}, we have
	\begin{itemize}
		\item if $p \leq d$, there exists $\bbsigma$-a.s. an infinity of measures $\gamma \neq \gamma_Z \in \mathcal{P}_2(\R^d)$ s.t. $\widehat{\mathrm{SW}}_p(\gamma, \gamma_Z) = 0$.

		\item if $p > d$, we have $\bbsigma$-almost surely $\lbrace\gamma_Z\rbrace = \underset{\gamma \in \mathcal{P}_2(\R^d)}{\argmin}\; \widehat{\mathrm{SW}}_p(\gamma, \gamma_Z)$.
	\end{itemize}
\end{theorem}

With a sufficient amount of projections, $\SWMC(\gamma_Y, \gamma_Z) = 0 \Rightarrow \gamma_Y = \gamma_Z$ (a.s.), hence when minimising $\SWMC(\gamma_Y, \gamma_Z)$ in $Y$, there is some hope of recovering $\gamma_Z$. Unfortunately, this does not guarantee that the (unique) solution will be attained numerically. This practical reality motivates the study of eventual local optima of $\SWpY$.

The computation of the critical points of $\SWpY$ can be done using the cell decomposition of~\ref{sec:cells}. %
We show that the critical points of $\SWpY$ are exactly the local optima of $\SWpY$, and correspond to "stable cells", which is to say cells that contain the minimum of their quadratic.

\subsubsection{Critical points of \texorpdfstring{$\SWpY$}{Sp} and cell stability}

The objective of this section is to confirm theoretically some of the intuitions provided by the illustrations of~\ref{sec:L2}, namely that the critical points of $\SWpY$ correspond to stable cells. Since the union of cells is exactly the differentiability set of $\SWpY$, any critical point $Y$ of $\SWpY$ is necessarily within a cell $\mathcal{C}_\config$. Since $\SWpY$ is quadratic on $\mathcal{C}_\config$, then a critical point $Y$ is the minimum of the cell's quadratic $q_\config$. As a consequence, the critical points of $\SWpY$ are exactly the "stable cell optima", i.e. the $Y \in \mathcal{U}$ (see the definition \ref{eqn:U}) such that $Y = \underset{Y' \in \R^{\npoints \times d}}{\argmin}\ q_{\config(Y)}(Y')$.

The following theorem shows that there are no local optima of $\SWpY$ outside of $\mathcal{U}$, and therefore that the set of local optima of $\SWpY$, the set of critical points of $\SWpY$ and the set of stable cell optima coincide. As previously, we define the set of critical points of $\SWpY$ as the set of points $Y$ where $\SWpY$ is differentiable and $\nabla \SWpY(Y) = 0$.

\begin{theorem}[The local optima of $\SWpY$ are within cells]\label{thm:Sp_crit_optloc_stable}\
 Assume that $(\theta_1, \cdots, \theta_p) \sim \bbsigma^{\otimes p}$, then the following results hold $\bbsigma$-almost surely.
  Let $Y \in \R^{\npoints\times d}$ a local optimum of $\SWpY$, then $\exists \config\in \mathfrak{S}_\npoints^p$ such that $ Y \in \mathcal{C}_\config$.
	As a consequence, we have the equality between the three sets:
	\begin{itemize}
		\item Local optima of $\SWpY$;
		\item Critical points of $\SWpY$;
		\item Stable cell optima: $\left\lbrace Y \in \mathcal{U}\ |\ Y = \underset{Y' \in \R^{\npoints \times d}}{\argmin}\ q_{\config(Y)}(Y') \right\rbrace$.
	\end{itemize}
\end{theorem}

\begin{proof}

	Let $Y \in \R^{\npoints\times d}$ a local optimum of $\SWpY$. Let $M := \lbrace \config \in \mathfrak{S}_\npoints^p\ |\ Y \in \overline{\mathcal{C}_\config} \rbrace$.

	Let $\config \in M$. Let us show that $\nabla q_\config(Y) = 0$ by contradiction: suppose $\nabla q_\config(Y) \neq 0$. For $t$ positive and small enough,
	\begin{align*}
		\SWpY(Y) &\leq \SWpY\left(Y - t\cfrac{\nabla q_\config(Y)}{\|\nabla q_\config(Y)\|}\right) \leq q_\config\left(Y - t\cfrac{\nabla q_\config(Y)}{\|\nabla q_\config(Y)\|}\right)\\ 
		&= q_\config(Y) - t \|\nabla q_\config(Y)\| + o(t) = \SWpY(Y) -t \|\nabla q_\config(Y)\| +o(t).
	\end{align*}
	Therefore, for $t>0$ sufficiently small, we have $\SWpY(Y) < \SWpY(Y),$ which is a contradiction. We now prove that $\# M = 1$.
	Using the notations of \ref{rem:quadratics}, for $\config \in M$, we have $\nabla q_\config(Y) = 0$, thus $B\vec{y} = a_\config$.
	For $(\theta_1, \cdots, \theta_p) \sim \bbsigma^{\otimes p}$, we have $\bbsigma$-almost surely that $B$ is invertible and that $\config \neq \config' \Longrightarrow a_\config \neq a_{\config'}$, thus $\bbsigma$-almost surely, $\# M = 1$, proving that in fact $Y$ belongs to $\mathcal{C}_\config$ and not to its boundary.

\end{proof}

\subsubsection{Closeness of critical points of \texorpdfstring{$\SWpY$}{Sp} and \texorpdfstring{$\SWY$}{S}}\label{sec:closeness}

In practice, all numerical optimisation methods converge towards a local optimum. %
One may wonder what is the link between the critical points of $\SWpY$, which we reach in practice, and the critical points of $\SWY$, among which are the theoretical solutions we would like to reach.

The following theorem shows that at the limit $p \rightarrow +\infty$, any sequence of critical points of $\SWpY$ become fixed points of $\Psi$ \ref{eqn:crit_fixed_point} in probability, which is to say that they exhibit similar properties to the critical points of $\SWY$.

\begin{theorem}[Approximation of the fixed-point equation]\label{thm:cv_fixed_point_distance}\

	For $p > d$, let $Y_p$ any critical point of $\SWpY$. Then we have the convergence in probability:
	\begin{equation}\label{eqn:cv_fixed_point_distance}
		Y_p - \Psi(Y_p) \xrightarrow[p \longrightarrow +\infty]{\P} 0.
	\end{equation}

	Specifically (see \ref{cor:simplified_condition_concentration}), in order to reach a precision of $\varepsilon$, we have $\|Y_p - \Psi(Y_p)\|_{\infty, 2} \leq \varepsilon$ with probability exceeding $1 - \eta$ if $p \geq \mathcal{O}\left(d^3\npoints\log(1/\eta)/\varepsilon^3\right)$ and $p \geq \mathcal{O}\left(d^3\npoints^2\log(1/\eta)/\varepsilon^2\right)$, omitting logarithmic multiplicative terms in $d$ and $\npoints$.
\end{theorem}

We provide the proof in \ref{p:cv_fixed_point_distance}, where we also estimate more precisely the convergence rate. The idea behind this result stems from computing the minima of the quadratics. Let $Y^* := \underset{Y}{\argmin}\ q_\config(Y)$, we have
\begin{equation}\label{eqn:next_opt_pos}
	y_k^* = A^{-1}\left(\cfrac{1}{p}\Sum{i=1}{p}\theta_i \theta_i^Tz_{\config_i(k)}\right) = \cfrac{A^{-1}}{p}\Sum{l \in \llbracket 1, \npoints \rrbracket}{}\Sum{\substack{i \in \llbracket 1, p \rrbracket \\ \config_i(k)=l}}{}\theta_i \theta_i^Tz_l,
\end{equation}
with $A = \cfrac{1}{p}\Sum{i=1}{p}\theta_i\theta_i^T$ which approaches the covariance matrix of $\theta \sim \bbsigma$, i.e. $I/d$.
Likewise, $\cfrac{1}{p}\Sum{\substack{i \in \llbracket 1, p \rrbracket \\ \config_i(k)=l}}{}\theta_i \theta_i^T$ can be seen as an empirical conditional covariance, and it approaches $S_{k,l}^{Y,Z} / d$.
We then apply matrix concentration inequalities to quantify the approximation error.

\subsubsection{Critical points of \texorpdfstring{$\SWpY$}{Ep} and Block Coordinate Descent}\label{sec:Ep_crit_and_bcd}

Leveraging on the cell structure of $\SWpY$, we present an algorithm alternatively solving for the transport matrices and for the positions. Writing $\U$ the set of valid transport plans between two uniform measures with $\npoints$ points, we minimise the following energy (with $(\theta_1,\cdots, \theta_p)$ fixed)
\begin{equation}\label{eqn:BCD_J}
	J := \app{\U^p \times \R^{\npoints \times d}}{\R_+}{(\pi^{(1)}, \cdots, \pi^{(p)}), Y}{\cfrac{1}{ p}\Sum{i=1}{p}\Sum{k=1}{\npoints}\Sum{l=1}{\npoints}(\theta_i^T y_k - \theta_i^T z_l)^2\pi_{k,l}^{(i)}}.
\end{equation}
Observe that minimising $J$ amounts to minimising $\SWpY$.

\begin{figure}
	\centering
	\begin{minipage}{.95\linewidth}
		\begin{algorithm}[H]
			\SetAlgoLined
			\KwData{Fixed axes $(\theta_1, \cdots, \theta_p) \in (\SS^{d-1})^p$, projections $(z_k^T \theta_i)_{k \in \llbracket 1, \npoints \rrbracket,\; i \in \llbracket 1, p \rrbracket}$.}
			\KwResult{Positions $Y \in \R^{\npoints \times d}$.}
			\textbf{Initialisation:} Draw $Y^{(0)} \in \R^{\npoints \times d}$\;
			\For{$t \in \llbracket 1, T_{\max} \rrbracket$}{
				Update the OT maps by solving $\pi^{(t)} \in \underset{\pi \in \U^p}{\argmin}\ J(\pi, Y^{(t-1)})$\;

				Update the positions by solving $Y^{(t)} = \underset{Y \in \R^{\npoints \times d}}{\argmin}\ J(\pi^{(t)}, Y)$\;

				\If{$\|Y^{(t)} - Y^{(t-1)}\|_{\infty, 2} < \varepsilon$}{
					Declare convergence and terminate.
				}
			}
			\caption{Minimising $\SWpY$ with Block-Coordinate Descent}
			\label{alg:BCD}
		\end{algorithm}
	\end{minipage}
\end{figure}

The computation in~\ref{alg:BCD}, line 3 is done using standard 1D OT solvers \cite{flamary2021pot}, and the update on the positions at line 4 can be computed in closed form \blue{(we provide the closed-form expression in~\ref{sec:closed_form_BCD} for the sake of reproducibility)}. BCD can be seen as a walk from cell to cell (see \ref{sec:cells}), as illustrated in \ref{fig:tessellation}. BCD moves from cell to cell and converges towards a stable cell optimum, and thus towards a local optimum of $\SWpY$ (since these two sets are equal by \ref{thm:Sp_crit_optloc_stable}). This behaviour is further studied in the experimental section. 

\begin{figure}[h]
	\centering
	\includegraphics[width=0.8\linewidth]{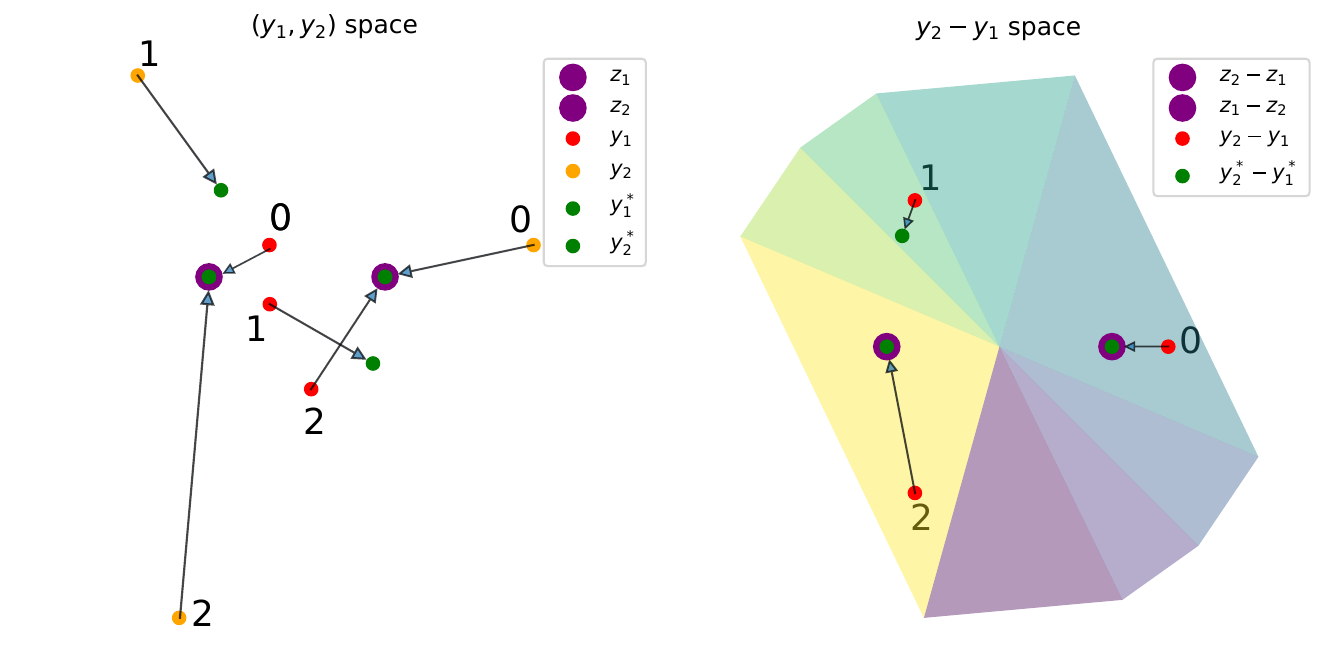}
	\caption{Illustration of the cell structure for $p=4$ in dimension 2 from a BCD viewpoint. On the left, we view different points $Y = (y_1, y_2)$ (in red and orange) and the minima of their respective quadratics: $(y_1^*, y_2^*)$, which should be compared to the original points $(z_1, z_2)$ in purple. On the right, we view the cell structure depending on the position of $y_2 - y_1 \in \R^2$, since the cell conditions only depend on this difference (see~\ref{eqn:2cell_polytope}). We can see that in this example all cells are stable, thus there are three strict local optima of $\SWpY$ in addition to the global optimum. The $(y_1, y_2)$ pair number 0 is sent to $(z_2, z_1)$, while the pair "1" is sent to a local optimum, and the pair "2" is sent to $(z_1, z_2)$.}
	\label{fig:tessellation}
\end{figure}

	\section{Stochastic Gradient Descent on \texorpdfstring{$\SWY$}{S} and \texorpdfstring{$\SWpY$}{Sp}}\label{sec:SGD}

As seen in \ref{sec:E_crit}, the optimisation properties of $\SWY$ and $\SWpY$
indicate that optimising their landscapes might prove difficult in practice. In
real-world applications, these landscapes (and especially $\SWY$, which is the
most used) are minimised using Stochastic Gradient Descent. Perhaps
unsurprisingly given the difficulties presented in \ref{sec:E_crit} and due to
the non-differentiable and non-convex properties of the landscapes, there has
been no attempt to prove the convergence of such SGD schemes in the literature
(to our knowledge). This section aims to bridge this knowledge gap, using recent
theoretical results on the convergence of \blue{constant-step} SGD schemes due
to Bianchi et al.~\cite{bianchi2022convergence}\bluetwo{, and using results on
decreasing-step SGD by Davis et al. \cite{davis2020stochastic}}. Related works
include Minibatch Wasserstein \cite{fatras2021minibatch} in particular Section 5
wherein they leverage another non-convex non-differentiable SGD convergence
framework from Majewski et al. \cite{majewski2018analysis} in order to derive
convergence results for minibatch gradient descent on the Wasserstein and
entropic Wasserstein distances. 

\blue{There are other frameworks than Bianchi et al.
\cite{bianchi2022convergence} \bluetwo{or Davis et al.
\cite{davis2020stochastic}} that one may consider in order to prove non-smooth,
non-convex convergence of SGD, in particular the work of Majewski et al.
\cite{majewski2018analysis}. Unfortunately, this work focuses on the case of
\textit{tame} functions, which is to say either \textit{Clarke regular}
functions (\cite{clarke1990optimization}), or \textit{stratifiable functions}
(\cite{bolte2007clarke}). It is not known whether $\SWY$ is Clarke regular (the
graph in \ref{fig:comp_sym} could intuitively point to the contrary, due to the
local shape in $-\|x\|$, and it is known that $-\|\cdot\|$ is not Clarke
regular). Likewise, it is not known whether $\SWY$ is stratifiable. Thankfully,
our regularity results from \ref{sec:energies} will allow us to show that $\SWY$
is \textit{path differentiable}, which is another (more general) regularity
class which is enough to apply the results from Bianchi et al.
\cite{bianchi2022convergence} \bluetwo{and Davis et al.
\cite{davis2020stochastic}.}}

\blue{Let us also mention the very recent work~\cite{Li:2023aa} (contemporary to
ours), which studies the convergence of stochastic gradient schemes
\bluetwo{with decreasing steps} and applied directly on probability
measures instead of point clouds. Working with absolutely continuous measures
$\mu$ and $\nu$, the authors consider a scheme of the form
\begin{equation}
\label{eq:li_sgd}
\mu^{(t+1)} = \left((1-\alpha^{(t)})I + \alpha^{(t)}T_{\theta^{(t)},\mu^{(t)},\nu}\right)\# \mu^{(t)},
\end{equation}
where the $\theta^{(t)}$ are i.i.d. drawn from the uniform measure on the
sphere, and where $T_{\theta^{(t)},\mu^{(t)},\nu}(x) := x +
\left(\tau^{(t)}((\theta^{(t)})^Tx)- (\theta^{(t)})^Tx\right)\theta^{(t)}$, with
$\tau^{(t)}$ the one-dimensional optimal transport map between the projected
measures $P_{\theta^{(t)}}\#\mu^{(t)}$ and $P_{\theta^{(t)}}\#\nu$ (this map is
uniquely defined on the support of $P_{\theta^{(t)}}\#\mu^{(t)}$ because
$\mu^{(t)}$ and $\nu$ are absolutely continuous). It is quite easy to see that
this scheme implements a stochastic gradient descent on $\mu\rightarrow\frac 1 2
\SW_2^2(\mu,\nu)$. Building on proof techniques
of~\cite{Backhoff-Veraguas:2022aa}, they show that if the sequence of learning
rates $(\alpha^{(t)})_k$ is decreasing with $\sum \alpha^{(t)} = +\infty$ and
$\sum (\alpha^{(k)})^2 < +\infty$, under some assumptions the sequence of
measures $(\mu^{(t)})$ converges to $\nu$ for $\W_2$. Unfortunately, extending
these methods to discrete measures is not straightforward. Indeed, the authors
of~\cite{Backhoff-Veraguas:2022aa} claim that although they believe their
results might hold for discrete measures, the difficulties of generalising their
proofs to the discrete case were too important to achieve satisfying proofs in
this case. }

Before presenting our core results and the necessary theoretical framework from Bianchi et al.~\cite{bianchi2022convergence}, we provide in~\ref{alg:SGD} the description of the SGD scheme used to minimise either $\SWY$ or $\SWpY$, i.e. for projections drawn with $\mu \in \lbrace \bbsigma, \bbsigma_p \rbrace$ respectively.
Starting with random initial points $Y^{(0)} \sim \nu$, at each step $t$, we draw a random projection $\theta^{(t+1)} \sim \bbsigma$ and compute an SGD iteration of step $\alpha^{(t)}$ in the direction of the gradient of $Y \longmapsto w_{\theta^{(t+1)}}(Y)$.
This scheme uses optionally an additive noise term controlled by a parameter $a$
(that can be set to $0$). \bluetwo{In \ref{sec:interpolated_SGD} to
\ref{sec:noised_sgd}, we shall study constant-step SGD schemes, and in
\ref{sec:decr_sgd}, we will focus on decreasing-step SGD.}

\begin{figure}[h]
	\centering
	\begin{minipage}{.95\linewidth}
		\begin{algorithm}[H]
			\SetAlgoLined \KwData{Learning rate sequence $(\alpha^{(t)})_{t\in
			\N}$, noise level $a \geq 0$, convergence threshold $\beta > 0$, and
			probability distribution $\mu$ on $\SS^{d-1}$.} \KwResult{Positions
			$Y \in \R^{\npoints \times d}$, assignment $\tau \in
			\mathfrak{S}_\npoints$.} \textbf{Initialisation:} Draw $Y^{(0)} \in
			\R^{\npoints \times d}$\; \For{$t \in \llbracket 0, T_{\max} - 1
			\rrbracket$}{ Draw $\theta^{(t+1)} \sim \mu$ and
			$\varepsilon^{(t+1)} \sim \mathcal{N}(0, I_{\npoints d})$.\\
				SGD update: \\$Y^{(t+1)} = Y^{(t)} - \alpha^{(t)}
				\left[\dr{Y}{}{} \W_2^2(P_{\theta^{(t+1)}}\#\gamma_Y,
				P_{\theta^{(t+1)}}\#\gamma_Z)\right]_{Y = Y^{(t)}} +
				\alpha^{(t)} a \varepsilon^{(t+1)}$

				\If{$\|Y^{(t+1)} - Y^{(t)}\|_{\infty, 2} < \beta$}{
					Declare convergence and terminate.
				}
			}
			\KwRet{$Y^{(t_{\mathrm{final}})}$ and the assignment $\tau$ of  $\W_2^2(P_{\theta^{(t_{\mathrm{final}})}}\#\gamma_{Y^{(t_{\mathrm{final}})}}, P_{\theta^{(t_{\mathrm{final}})}}\#\gamma_Z)$.}
			\caption{Minimising $\SWY$ or $\SWpY$ with Stochastic Gradient Descent}
			\label{alg:SGD}
		\end{algorithm}
	\end{minipage}
\end{figure}

\subsection*{Overview of Main Results}

In~\cite{bianchi2022convergence}, Bianchi et al. establish conditions under
which a constant-step SGD converges (in a certain sense), for a non-convex,
locally Lipschitz cost function. Observe that both $\SWY$ and $\SWpY$ are indeed
locally Lipschitz, as shown in~\ref{thm:locLip}. In \ref{sec:interpolated_SGD}
and \ref{sec:noised_sgd}, we verify the required conditions for $\SWY$ and
$\SWpY$ (with $p$ fixed projections), and prove results which can be broadly
summarised as follows:

\begin{theorem*}[\ref{thm:SGD_interpolated_cv}: Convergence of the interpolated SGD (without noise) for $\SWY$ and $\SWpY$]
	Given a sequence of SGD schemes $(Y_\alpha^{(t)})$ for $\SWY$ (resp. $\SWpY$) of steps $\alpha$, their associated piecewise affine interpolated schemes $(Y_\alpha)$ converge, in a weak sense as $\alpha \longrightarrow 0$, to the set of solutions of the differential inclusion equation $\dot{Y}(s) \in -\partial_C \SWY(Y(s))$ (resp. $\SWpY$), where $\partial_C \SWY$ denotes the Clarke differential of $\SWY$.

\end{theorem*}

If we instead consider a \textit{noised} SGD scheme (with noise magnitude $\alpha \times a,\; a>0$), we have a stronger convergence result:

\begin{theorem*}[\ref{thm:noised_SGD_cv}: Convergence of the noised SGD for $\SWY$ and $\SWpY$]
	Given a sequence of noised SGD schemes $(Y_\alpha^{(t)})$ for $\SWY$ (resp. $\SWpY$) of steps $\alpha$, they converge, in a weak sense as $\alpha \longrightarrow 0$, to the set of (Clarke) critical points of $\SWY$ (resp. $\SWpY$).

\end{theorem*}

These results rely on the notion of Clarke differentiability, which generalises differentiability to non smooth functions as soon as these functions are locally Lipschitz ({\em i.e.} Lipschitz in a neighbourhood of each point). More precisely, for such a function $f:\mathbb{R}^d \rightarrow \mathbb{R}$, its Clarke sub-differential at $x$ is defined as the convex hull of the limits of gradients of $f$
$$\partial_C f(x) := \mathrm{conv}\left\{v \in \R^d:\; \exists (x_i) \in (\mathcal{D}_f)^\N: x_i \xrightarrow[i \longrightarrow +\infty]{} x\ \mathrm{and} \; \nabla f(x_i) \xrightarrow[i \longrightarrow +\infty]{} v\right\},$$
where $\mathcal{D}_f$ denotes the set of differentiability of $f$, whose
complementary is of Lebesgue measure 0 by Rademacher's theorem, since $f$ is
locally Lipschitz. This notion of differentiability coincides with the classical
one for differentiable functions, and with the usual sub-differential for convex
functions. Clarke \textit{critical points} of $f$ are points $x$ such that $0\in
\partial_C f(x) $.

\bluetwo{In the case of decreasing-step SGD, in \ref{sec:decr_sgd} we leverage
the results of \cite{davis2020stochastic} to prove the following result, under
typical assumptions on the decreasing steps $(\alpha^{(t)})$.
\begin{theorem*}[\ref{thm:cv_decreasing_lr}: Convergence of decreasing-step
	noised SGD for $\SWY$ and $\SWpY$] Consider $(Y^{(t)})$ a trajectory of
	noised decreasing-step SGD for $\mu\in \{\bbsigma, \bbsigma_p\}$
	respectively, assume that it is almost-surely bounded. Then the sequence
	$F(Y^{(t)})$ is almost-surely convergent for $F\in \{\SWY, \SWpY\}$
	respectively, and almost-surely, any subsequential limit of $(Y^{(t)})$
	belongs to the set of Clarke critical points of $F$. \end{theorem*}}

\subsection{Theoretical framework}\label{sec:bianchi_setting}

In the following, we briefly present the theoretical framework of Bianchi et al.~\cite{bianchi2022convergence}. They consider a function $f: \R^D \times \Theta \longrightarrow \R$, locally Lipschitz continuous in the first variable (for each $\theta$), and $\mu$ a probability measure on $\Theta \subset \R^d$. Since $f$ is locally Lipschitz in the first variable, the gradient $\nabla f(\cdot,\theta)$ of $f(\cdot,\theta)$ (w.r.t. the first variable) can be defined almost everywhere on $\R^D$, and any function $\varphi: \R^D \times \Theta \longrightarrow \R^D$ such that $\lambda \otimes \mu$ a.e., $\varphi = \nabla f$ is called an almost-everywhere gradient of $f$ (see \cite{bianchi2022convergence}, Definition 1). Let $F := \blue{Y} \longrightarrow \int_\Theta f(Y, \theta)\dd\mu(\theta)$. A SGD scheme of step $\alpha > 0$ for $F$ is a sequence $(Y^{(t)})$ of the form:
\begin{equation}\label{eq:SGD_Bianchi}
Y^{(t+1)} = Y^{(t)} - \alpha \varphi(Y^{(t)}, \theta^{(t+1)}),\quad \left(Y^{(0)}, (\theta^{(t)})_{t \in \N}\right) \sim \nu \otimes \mu^{\otimes \N},
\end{equation}
where $\nu$ is the distribution of the initial position $Y^{(0)}$, which we shall assume to be absolutely continuous w.r.t. the Lebesgue measure.

Within this framework, we can define an SGD scheme for $\SWY$ and $\SWpY$. The function $w_\theta$ (Equation~\ref{def:wtheta}) plays the role of $f$. We know from~\ref{prop:w_unif_locLip} that $w_\theta$ is locally Lipschitz (uniformly in $\theta$), hence differentiable almost everywhere, and that at these points of differentiability, using~\cite{bonneel2015sliced} (Appendix A, "proof of differentiability"), the derivative of $w_{\theta}$ in $Y$ is
\begin{equation}\label{eqn:ae_grad}
	\varphi(Y, \theta) := \left[\cfrac{2}{\npoints}\theta\theta^T\left(y_k - z_{\sort{Z}{\theta} \circ (\sort{Y}{\theta})^{-1}(k)}\right)\right]_{k \in \llbracket 1, \npoints \rrbracket},
\end{equation}
which corresponds to the definition of an almost-everywhere gradient as proposed by~\cite{bianchi2022convergence}. Moreover, $\varphi$ can be extended everywhere by choosing the sorting permutations arbitrarily when there is ambiguity.
Within this framework, given a step $\alpha > 0$, and an initial position $Y^{(0)} \sim \nu$, the fixed-step SGD iterations~\ref{eq:SGD_Bianchi} can be applied to  $F = \SWY$ by choosing $\mu = \bbsigma$ or to $F = \SWpY$ by choosing $\mu = \bbsigma_p := \frac{1}{p}\sum_i^p\delta_{\theta_i}$. We assume that $\Span(\theta_i)_{i \in \llbracket 1, p \rrbracket} = \R^d$, which is satisfied $\bbsigma$-almost surely if $(\theta_i)_{i \in \llbracket 1, p \rrbracket} \sim \bbsigma^{\otimes p}$, since $p > d$.

\subsection{Convergence of piecewise affine interpolated SGD schemes on \texorpdfstring{$\SWY$}{E} and \texorpdfstring{$\SWpY$}{Ep}}\label{sec:interpolated_SGD}

The \textit{piecewise-affine interpolated SGD scheme} associated to a discrete SGD scheme $(Y_\alpha^{(t)})$ of step $\alpha$ is defined as:
$$Y_\alpha(s) = Y_\alpha^{(t)} + \left(\frac{s}{\alpha} - t\right)(Y_\alpha^{(t+1)} - Y_\alpha^{(t)}),\quad \forall s \in [t\alpha, (t+1)\alpha[, \quad \forall t \in \N.$$
We consider the space of absolutely continuous curves from $\R_+$ to $\R^D$, denoted $\mathcal{C}_{\mathrm{abs}}(\R_+, \R^D)$, and endow it with the metric of uniform convergence on all segments:
\begin{equation}
d_c(Y, Y') := \Sum{k \in \N^*}{}\cfrac{1}{2^k}\min\left(1, \underset{s \in [0, k]}{\max}\|Y(s) - Y'(s)\|_{\infty, 2}\right).
\label{eq:metric_uc}
\end{equation}
We will show that when the step decreases, the interpolated processes approach the set of solutions of a differential inclusion equation. To that end, we define the set of absolutely continuous curves that start within a given compact $\mathcal{K}$ of $\R^D$ and are a.e. solutions of the differential inclusion:
\begin{equation}\label{eqn:Clarke_DI}
	S_{-\partial_C F}(\mathcal{K}) := \left\lbrace Y \in \mathcal{C}_{\mathrm{abs}}(\R_+, \R^D)\ |\  \ull{\forall} s\in \R_+,\; \dot{Y}(s) \in -\partial_C F(Y(s));\; Y(0) \in \mathcal{K} \right\rbrace,
\end{equation}
where $\ull{\forall}$ denotes "for almost every". Bianchi et al.~\cite{bianchi2022convergence} present three conditions under which they prove the convergence (in a certain weak sense) of interpolated SGD schemes on $F$. For the sake of self-containedness, we reproduce them here and verify them successively. Recall that for our two respective applications, $f(Y, \theta) = w_\theta(Y)$, $\mu \in \lbrace \bbsigma,\ \bbsigma_p\rbrace$ and $F \in \lbrace \SWY,\ \SWpY\rbrace$.

\begin{assumption}\label{ass:A1}\
	\begin{itemize}
		\item[i)] There exists $\kappa: \R^D \times \SS^{d-1} \longrightarrow \R_+$ measurable such that each $\kappa(X, \cdot)$ is $\mu$-integrable, and:
		$$\exists \varepsilon > 0,\; \forall Y, Y' \in B(X, \varepsilon),\; \forall \theta \in \SS^{d-1},\; |f(Y, \theta) - f(Y', \theta)|\leq \kappa(X, \theta)\|Y-Y'\|. $$
		\item[ii)] There exists $X \in \R^D$ such that $f(X, \cdot)$ is $\mu$-integrable.
	\end{itemize}
\end{assumption}

Since $f$ is the same in both cases, we can satisfy \ref{ass:A1} for both schemes simultaneously. The (quantified) uniformly locally Lipschitz property of the $w_\theta$ (\ref{prop:w_unif_locLip}) allows us to verify \ref{ass:A1}, by letting $r := 1$ and $\kappa(X, \theta) := \kappa_1(X)$. \ref{ass:A1} ii) is immediate since for \textit{all} $Y \in \R^D,\; \theta \longmapsto w_\theta(Y)$ is continuous, therefore $\bbsigma-L^1$ and $\bbsigma_p-L^1$.

\begin{assumption}\label{ass:A2} The function $\kappa$ of \ref{ass:A1} verifies:
	\begin{itemize}
		\item[i)] There exists $c \geq 0$ such that $\forall X\in \R^D,\; \Int{\SS^{d-1}}{}\kappa(X, \theta)\dd\mu(\theta) \leq c(1+\|X\|)$.
		\item[ii)] For every $\mathcal{K}$ compact of $\R^D,\; \underset{X \in \mathcal{K}}{\sup}\ \Int{\SS^{d-1}}{}\kappa(X, \theta)^2\dd\mu(\theta) <+\infty$.
	\end{itemize}
\end{assumption}

The choice $\kappa(X, \theta) := \kappa_1(X)$ (independent on $\theta$, and as defined in \ref{prop:w_unif_locLip}) satisfies \ref{ass:A2}. We now consider the Markov kernel associated to the SGD schemes, denoting the Borel sets $\mathcal{B}(\R^D)$:
$$P_\alpha : \app{\R^D \times \mathcal{B}(\R^D)}{[0, 1]}{Y, B}{\Int{\SS^{d-1}}{}\mathbbold{1}_B(Y - \alpha \varphi(Y, \theta))\dd \mu(\theta)}.$$
With $\lambda$ denoting the Lebesgue measure on $\R^D$, let 
$$\Gamma := \left\lbrace \alpha \in\ ]0, +\infty[\ |\ \forall \rho \ll \lambda,\ \rho P_\alpha \ll \lambda\right\rbrace.$$ 
We will verify the following assumption for both schemes:

\begin{assumption}\label{ass:A3} The closure of $\Gamma$ contains 0.
\end{assumption}

\blue{In \ref{prop:SWY_Gamma}, we prove a stronger result, namely that $\Gamma$ contains $]0, \frac{\npoints}{2}[$, which allows us to simply take learning rates $0 < \alpha < \frac{\npoints}{2}$, instead of having to specify $\alpha \in \Gamma$. As a weaker alternative, \ref{ass:A3} could also be verified by noticing that for any $\theta \in \SS^{d-1}$, the function $w_\theta$ is of class $\mathcal{C}^2$ almost everywhere (as detailed in \ref{sec:cells}), which allows us to apply \cite{bianchi2022convergence} Proposition 4 and shows that \ref{ass:A3} holds.}

\begin{prop}\label{prop:SWY_Gamma} For schemes \ref{eq:SGD_Bianchi} applied to $\SWY$ or $\SWpY$,
  $\Gamma\ = \R_+^{*} \setminus \{ \frac n 2 \}$. %
\end{prop}

\begin{proof}
	Let $\mu \in \lbrace \bbsigma,\; \bbsigma_p\rbrace$. Recall the line-by-line notation $Y = (y_1, \cdots, y_\npoints)^T \in \R^{\npoints \times d}$. We also denote $Z_\tau := (z_{\tau(1)}, \cdots, z_{\tau(\npoints)})^T$ for $\tau \in \mathfrak{S}_\npoints$. Let $\rho \ll \lambda$ and $B \in \mathcal{B}(\R^D)$ such that $\lambda(B) = 0$. We have, with $\alpha' := 2\alpha/\npoints$:
	\begin{align*}\rho P_\alpha(B) &= \Int{\R^D}{}\Int{\SS^{d-1}}{}\mathbbold{1}_B\left(Y(I - \alpha'\theta\theta^T) + \alpha'Z_{\sort{Z}{\theta} \circ (\sort{Y}{\theta})^{-1}}\theta\theta^T\right)\dd \mu(\theta)\dd\rho(Y)\\
		&\leq \Sum{\tau \in \mathfrak{S}_\npoints}{}\ \Int{\R^D}{}\Int{\SS^{d-1}}{}\mathbbold{1}_B\left(Y(I - \alpha'\theta\theta^T) + \alpha'Z_{\tau}\theta\theta^T\right)\dd \mu(\theta)\dd\rho(Y)\\
		&= \Sum{\tau \in \mathfrak{S}_\npoints}{}\ \Int{\SS^{d-1}}{}I_\tau(\theta)\dd \mu(\theta),
	\end{align*}
	where $I_\tau(\theta) := \Int{\R^D}{}\mathbbold{1}_B\left(Y(I - \alpha'\theta\theta^T) + \alpha'Z_{\tau}\theta\theta^T\right)\dd\rho(Y)$, and where the last line is obtained by applying Tonelli's theorem. Let $\tau \in \mathfrak{S}_\npoints$ and $\theta \in \SS^{d-1}$. We now assume $\alpha' \neq 1$, which is to say $\alpha \neq \npoints/2$.
	We operate the affine change of variables $X =\phi(Y):= Y(I - \alpha'\theta\theta^T) + \alpha'Z_{\tau}\theta\theta^T$, which is invertible for $\alpha'\neq 1$.
	We have 
	$$I_\tau(\theta) = \Int{\R^D}{}\mathbbold{1}_B(\phi(Y))\dd \rho(Y) = \Int{\R^D}{}\mathbbold{1}_B(X)\dd\phi \# \rho(X) = \phi\#\rho(B).$$
	Now since $\phi$ is affine and invertible, $\phi \#\rho \ll \lambda$, thus $\phi \#\rho(B) = 0$, and finally $\rho P_\alpha(B) = 0$. This proves that $\rho P_\alpha \ll \lambda$ for $\alpha > 0$ differing from $\npoints/2$.
\end{proof}

Now that we have verified \ref{ass:A1}, \ref{ass:A2} and \ref{ass:A3}, we can apply \cite{bianchi2022convergence}, Theorem 2 to $\SWY$ and $\SWpY$. Let $0 < \alpha_0 < \npoints/2$.

\begin{theorem}[\cite{bianchi2022convergence}, Theorem 2 applied to $\SWY$ and $\SWpY$: convergence of the interpolated SGD scheme]\label{thm:SGD_interpolated_cv}
	Let $(Y_\alpha^{(t)}), \alpha \in ]0, \alpha_0], t \in \N$ a collection of SGD sequences associated to~\ref{eq:SGD_Bianchi} applied to $\SWY$ or $\SWpY$. Consider $(Y_\alpha)$ their associated piecewise affine interpolations. For any $\mathcal{K}$ compact of $\R^D$ and any $\varepsilon > 0$, we have for $F \in \lbrace \SWY,\ \SWpY\rbrace$ and $\mu \in \lbrace \bbsigma,\ \bbsigma_p\rbrace$ respectively
	\begin{equation}
		\underset{\substack{\alpha \longrightarrow 0 \\ \alpha \in\ ]0, \alpha_0]}}{\lim}\ \nu \otimes \mu^{\otimes\N}\left( d_c(Y_\alpha, S_{-\partial_C F}(\mathcal{K})) > \varepsilon\right) = 0,
	\end{equation}
where $d_c$ is the metric of uniform convergence defined in~\ref{eq:metric_uc}.
\end{theorem}

It is to be understood that when the SGD step decreases, the interpolated schemes converge towards the set of solutions of the differential inclusion related to the continuous SGD equation. This convergence is weak: the distance to this set approaches 0 in probability, and $S_{-\partial_C F}(\mathcal{K})$ is a set of solutions which we do not know how to compute, however we can study some theoretical properties of the solutions given a suitable starting point $Y(0)$, see \ref{rem:flows}.

\begin{remark}\label{rem:flows}
	For $\SWY$, if the initial position $Y^{(0)}$ belongs to a maximal connected component $\mathcal{V}$ of the differentiability set $\mathcal{U}$ (which is open), then consider the \textit{gradient flow} differential equation
	\begin{equation}\label{eqn:SWY_flow}
		\cfrac{\partial \gamma}{\partial t}(Y, s) = -\nabla \SWY(\gamma(Y, s)),\quad  \gamma(Y, 0) = Y,\quad \gamma(Y, s) \in \mathcal{V}.
	\end{equation}
	Since $\SWY$ is of class $\mathcal{C}^1$ on $\mathcal{V}$ (by \ref{thm:bonneel_diff}), with $\nabla\SWY$ Lipschitz (locally would suffice), standard flow results show that there exists a unique solution $\gamma(Y, \cdot)$ for any $Y \in \mathcal{V}$ defined on some interval $]a_Y, b_Y[ \subset \R$, which defines a continuous function $\gamma: \mathcal{D} \longrightarrow \mathcal{V}$, with $\mathcal{D} = \lbrace (s, Y) \in \R \times \mathcal{V}\ |\ s \in ]a_Y, b_Y[] \rbrace$. Since in our case, we consider a \textit{gradient} flow, and since for any $c \in \R,$ the set $A_c := \lbrace Y \in \mathcal{V}\ |\ \SWY(Y) \leq c \rbrace$ is compact, in fact the flows $\gamma(Y, s)$ are defined for $s \in [0, +\infty[$. Furthermore, if a sequence $(\gamma(Y, s_m))_{m\in \N}$ were to converge to a limit $Y^{\infty}$, then one would have $Y^{\infty} \in \mathcal{V}$ and $\nabla\SWY(Y^{\infty}) = 0$. Our work does not show that the set $\mathcal{Z}_\mathcal{V} := \lbrace Y \in \mathcal{V}\ |\ \nabla \SWY(Y) = 0\rbrace$ of critical points of $\SWY$ is finite, however if that were the case, then more standard euclidean gradient flow results show that for any $Y \in \mathcal{V},\: \exists Y^\infty \in \mathcal{Z}_\mathcal{V}:\; \gamma(Y, s) \xrightarrow[s \longrightarrow +\infty]{} Y^\infty$.

	Note that given a learning rate $\alpha > 0$, an SGD scheme \ref{eq:SGD_Bianchi} applied to $\SWY$ and starting in $\mathcal{V}$ has no reason to stay in $\mathcal{V}$, and we unfortunately do not have equality with a discretised version of the gradient flow. However, thanks to Bianchi et al. \cite{bianchi2022convergence} Theorem 1, the trajectories stay almost-surely in differentiability points of $\SWY$ and $w_\theta$, and thus almost-surely, $\varphi(Y^{(t)}, \theta^{(t+1)}) = \nabla w_{\theta^{(t+1)}}(Y^{(t)})$.
\end{remark}

\subsection{Convergence of Noised SGD Schemes on \texorpdfstring{$\SWY$}{E} and \texorpdfstring{$\SWpY$}{Ep}}\label{sec:noised_sgd}
In order to prove stronger convergence results we need to consider noised variants of our SGD schemes. Consider $\varepsilon \sim \eta := \mathcal{N}(0, I_D)$ an independent noise, our schemes become:
\begin{equation}\label{eqn:SWY_noisedSGD}
	Y^{(t+1)} = Y^{(t)} - \alpha \varphi(Y^{(t)}, \theta^{(t+1)}) + \alpha a\varepsilon^{(t+1)},\quad (Y^{(0)}, (\theta^{(t)})_{t \in \N}, (\varepsilon^{(t)})_{t \in \N}) \sim \nu \otimes \mu^{\otimes \N} \otimes \eta^{\otimes \N}
\end{equation}

where $\mu= \bbsigma$ for $\SWY$ and $\bbsigma_p$ for $\SWpY$.
We follow the method from \cite{bianchi2022convergence}, which suggests that adding a small perturbation (that decreases with the step size) allows us to verify additional suitable assumptions. Note that this modification does not impact our verification of the previous assumptions 1 through 3. Bianchi et al. introduce the following assumption:

\begin{assumption}\label{ass:A4}
	there exists $V,p: \R^D \rightarrow \R_+$ and $\beta : \R_+^* \rightarrow \R_+^*$ measurable, as well as $C \geq 0$, such that for any $\alpha \in \Gamma\ \cap\ ]0, \alpha_0]$:
	\begin{itemize}
		\item[i)] $\exists R(\alpha) > 0,\; \delta(\alpha) >0,\; \exists \rho(\alpha)$ a probability measure on $\R^D$, such that:
		$$\forall Y \in \oll{B}(0, R),\; \forall A \in \mathcal{B}(\R^D),\; P_\alpha(Y, A) \geq \delta \rho(A).$$
		\item[ii)] $\underset{Y \in \oll{B}(0, R)}{\sup}\ V(Y) < +\infty$ and $\underset{Y \in B(0, R)^c}{\inf}\ p(Y) >0$, with:
		$$\forall Y \in \R^D,\; P_\alpha V(Y) \leq V(Y) - \beta(\alpha)p(Y) + C\beta(\alpha)\mathbbold{1}_{\oll{B}(0,R)}(Y).$$
		\item[iii)] $p(Y) \xrightarrow[Y \longrightarrow \infty]{} +\infty$.
	\end{itemize}
\end{assumption}

Thanks to Bianchi et al. \cite{bianchi2022convergence}, Proposition 5, this noised setting implies immediately \ref{ass:A4} i), for \textit{any} choice of $R>0$. They also suggest more restrictions on $f$ that imply \ref{ass:A4} ii) and iii), which our use case does not satisfy. We shall verify \ref{ass:A4} ii) and iii) for $\SWY$ and $\SWpY$ separately, but using similar methods. Beforehand, let us remark that the Markov kernel associated to~\ref{eqn:SWY_noisedSGD} is determined by the following action on measurable functions $g: \R^D \longrightarrow \R$:
$$P_\alpha g(Y) = \Int{\SS^{d-1} \times \R^D}{}g(Y - \alpha\varphi(Y, \theta)+ \alpha a X)\dd\mu(\theta)\dd\eta(X).$$

\begin{prop}[Drift property for noised SGD on $\SWY$]\label{prop:drift_SGD_SWY}\
	Let $V := \|\cdot\|_F^2, \; \alpha_0 < 1$ and $p(Y) := \cfrac{2}{d\npoints}(1-\alpha_0)\Sum{k=1}{\npoints}\|y_k\|_2^2$. There exists $R > 0$ and $C \geq 0:$
	$$\forall \alpha \in\ ]0, \alpha_0],\;\forall Y \in \R^D,\; P_\alpha V(Y) \leq V(Y) - \alpha p(Y) + C\alpha\mathbbold{1}_{\oll{B}(0,R)}(Y).$$
	Therefore, \ref{ass:A4} is satisfied for \ref{eqn:SWY_noisedSGD} when $\mu = \bbsigma$.
\end{prop}

\begin{proof}
	Let $Y \in \R^D,\; \alpha \in ]0, 1[$ and $\Delta(Y) := P_\alpha V(Y) - V(Y)$. We expand the square, then leverage the fact that $\eta$ is centred, and decompose\blue{, writing $\varphi_k := \varphi(Y, \theta)_{k, \cdot}$}:
	$$\Delta(Y) = \alpha^2 a^2\npoints d + \underbrace{\alpha^2 \Sum{k=1}{\npoints}\Int{\SS^{d-1}}{}\blue{\varphi_k^T\varphi_k}\dd \bbsigma(\theta)}_{\Delta_1(Y)} \underbrace{-2\alpha\Sum{k=1}{\npoints}\Int{\SS^{d-1}}{}\blue{y_k^T\varphi_k}\dd \bbsigma(\theta)}_{\Delta_2(Y)} .$$
	We have $\blue{\varphi_k^T\varphi_k} = \cfrac{4}{\npoints^2}(y_k - z_{\sort{Z}{\theta} \circ (\sort{Y}{\theta})^{-1}(k)})^T \theta \theta^T (y_k - z_{\sort{Z}{\theta} \circ (\sort{Y}{\theta})^{-1}(k)})$.
	Then recall that for all $\theta \in \Theta_{k,l}^{Y,Z}$, $ z_{\sort{Z}{\theta} \circ (\sort{Y}{\theta})^{-1}(k)} = z_l$. It follows that %
	\begin{align*}\Delta_1(Y) &= \cfrac{4\alpha^2}{\npoints^2}\Sum{(k,l) \in \llbracket 1, \npoints \rrbracket^2}{}\Int{\Theta_{k,l}^{Y,Z}}{}(x_k-z_l)^T\theta\theta^T(y_k-z_l)\dd\bbsigma(\theta) \\
		&= \cfrac{4\alpha^2}{d\npoints^2}\Sum{(k,l) \in \llbracket 1, \npoints \rrbracket^2}{}(y_k-z_l)^TS_{k,l}^{Y,Z}(y_k-z_l)\\
		&\leq  \cfrac{4\alpha\alpha_0}{d\npoints^2}\Sum{(k,l) \in \llbracket 1, \npoints \rrbracket^2}{} \|y_k-z_l\|_2^2
		\leq \cfrac{4\alpha\alpha_0}{d\npoints}\left(\Sum{k=1}{\npoints}(\|y_k\|_2^2 + 2\|Z\|_{\infty, 2}\|y_k\|_2) + \npoints \|Z\|_{\infty, 2}^2 \right),
	\end{align*}
	where we used the inequality $S_{k,l}^{Y, Z} \preceq I_d$.

	Now for $\Delta_2$, we have $\blue{y_k^T\varphi_k} = \frac{2}{\npoints}(\theta^T y_k) \theta^T(y_k - z_{\sort{Z}{\theta} \circ (\sort{Y}{\theta})^{-1}(k)})$, hence
	\begin{align*}\Delta_2(Y) &= -\cfrac{4\alpha}{d\npoints}\Sum{(k,l) \in \llbracket 1, \npoints \rrbracket^2}{}y_k^TS_{k,l}^{Y,Z}(y_k - z_l) \\
		&=-\cfrac{4\alpha}{d\npoints}\Sum{k=1}{\npoints}\|y_2\|^2 + \cfrac{4\alpha}{d\npoints}\Sum{(k,l) \in \llbracket 1, \npoints \rrbracket^2}{}y_k^TS_{k,l}^{Y,Z}z_l\\
    	&\leq -\cfrac{4\alpha}{d\npoints}\Sum{k=1}{\npoints}\|y_2\|^2 + \cfrac{4\alpha}{d}\|Z\|_{\infty, 2}\Sum{k=1}{\npoints}\|y_k\|_2,
    \end{align*}
    since $\Sum{l=1}{\npoints}S_{k,l}^{Y, Z} = I_d$ and $S_{k,l}^{Y, Z} \preceq I_d$. Finally,
	\begin{align*}\Delta(Y) &\leq \alpha\bigg[- 		\underbrace{\cfrac{4}{d\npoints}(1-\alpha_0)\Sum{k=1}{\npoints}\|y_k\|_2^2}_{q(Y)}\\
		&+\underbrace{\alpha_0a^2\npoints d + \cfrac{4\alpha_0}{d}\|Z\|_{\infty, 2}^2 + \cfrac{4}{d}\|Z\|_{\infty, 2}\Sum{k=1}{\npoints}\|y_k\|_2 + \cfrac{8\alpha_0}{d\npoints}\|Z\|_{\infty, 2}\Sum{k=1}{\npoints}\|y_k\|_2}_{r(Y)}\bigg].
	\end{align*}
	Now since $\cfrac{r(Y)}{q(Y)} \xrightarrow[\|Y\| \longrightarrow +\infty]{} 0$, there exists $R>0$ such that for $Y \in \R^D$ such that $\|Y\|_{\infty, 2} > R$, we have $r(Y) \leq q(Y)/2$. In that case, we have $\Delta(Y) \leq  \alpha(-q(Y) + q(Y)/2) = -\alpha q(Y)/2$. For $Y \in \R^D$ such that $\|Y\|_{\infty, 2}\leq R$, we have $\Delta(Y) \leq \alpha r(Y) \leq \alpha\underset{\|Y\|_{\infty, 2} \leq R}{\max}r(Y) =: C\alpha$ ($C$ exists since $r$ is continuous on the compact $\oll{B}(0, R)$.)
	This proves that for any $Y \in \R^D,\; \Delta(Y) \leq -\alpha q(Y)/2 + C\alpha\mathbbold{1}_{\oll{B}(0, R)}(Y)$.
	\end{proof}

We now turn to the scheme for  $\SWpY$. Let $A := \cfrac{1}{p}\Sum{j=1}{p}\theta_j\theta_j^T$, and consider $\lambda_{\min}(A)$ its smallest eigenvalue. Note that $\lambda_{\min}(A) > 0$, since we assumed $\Span(\theta_j)_{j \in \llbracket 1, p \rrbracket} = \R^d$.

\begin{prop}[Drift property for noised SGD on $\SWpY$]\label{prop:drift_SGD_SWpY}\
		Let $V := \|\cdot\|_F^2, \; \alpha_0 < \npoints$ and $q(Y) := \cfrac{2}{\npoints}\left(1-\cfrac{\alpha_0}{\npoints}\right)\lambda_{\min}(A)\Sum{k=1}{\npoints}\|y_k\|_2^2$. There exists $R > 0$ and $C \geq 0:$

	$$\forall \alpha \in\ ]0, \alpha_0],\;\forall Y \in \R^D,\; P_\alpha V(Y) \leq V(Y) - \alpha q(Y) + C\alpha\mathbbold{1}_{\oll{B}(0,R)}(Y).$$

	Therefore, \ref{ass:A4} is satisfied for \ref{eqn:SWY_noisedSGD} when $\mu = \bbsigma_p$.
\end{prop}

We leverage the same strategy as \ref{prop:drift_SGD_SWY}, yet the technicalities of the upper-bounds differ.

\begin{proof}
	Let $Y \in \R^D$ and $\alpha \in\ ]0, \alpha_0]$. We expand the squares and use that $\eta$ is centred:
	\begin{align*}
		\Delta(Y) &:= P_\alpha V(Y) - V(Y) \\
		&= \alpha^2a^2\npoints d + \underbrace{\alpha^2\cfrac{1}{p}\Sum{j=1}{p}\Sum{k=1}{\npoints}\Sum{i=1}{d}\varphi(Y, \theta_j)_{k,i}^2}_{\Delta_1(Y)} \underbrace{-2\alpha\cfrac{1}{p}\Sum{j=1}{p}\Sum{k=1}{\npoints}\Sum{i=1}{d}y_{k,i}\varphi(Y, \theta_j)_{k,i}}_{\Delta_2(Y)}.
	\end{align*}
    On the one hand,
	\begin{align*}\Delta_1(Y) &= \cfrac{4\alpha^2}{p\npoints^2}\Sum{j=1}{p}\Sum{k=1}{\npoints}(y_k - z_{\sort{Z}{\theta_j} \circ (\sort{Y}{\theta_j})^{-1}(k)})^T\theta_j \theta_j^T(y_k - z_{\sort{Z}{\theta_j} \circ (\sort{Y}{\theta_j})^{-1}(k)})\\
		&\leq \cfrac{4\alpha\alpha_0}{\npoints^2}\left(\npoints \|Z\|_{\infty, 2}^2 + \Sum{k=1}{\npoints}\left(y_k^TAy_k + 2\|Z\|_{\infty, 2}\|y_k\|_2\right)\right).
	\end{align*}
	Similarly, $\Delta_2(Y) \leq -\cfrac{4\alpha}{\npoints}\Sum{k=1}{\npoints}y_k^TAy_k + \cfrac{4\alpha}{\npoints}\|Z\|_{\infty, 2}\Sum{k=1}{\npoints}\|y_k\|_2$. 
	Let $$q_0(Y) := \cfrac{4}{\npoints}\left(1 - \cfrac{\alpha_0}{\npoints}\right)\lambda_{\min}(A)\Sum{k=1}{\npoints}\|y_k\|_2^2,$$ $$r(Y) := \alpha_0a^2\npoints d + \cfrac{4\alpha_0}{\npoints}\|Z\|_{\infty, 2}^2 + \left(\cfrac{8\alpha_0}{\npoints^2} + \cfrac{4}{\npoints}\right)\|Z\|_{\infty, 2}\Sum{k=1}{\npoints}\|y_k\|_2.$$
	We have $\Delta(Y) \leq \alpha(-q_0(Y) + r(Y))$, and we can conclude using the same method as \ref{prop:drift_SGD_SWY}.
\end{proof}
Finally, we require the fairly natural assumption that $F$ admits a "chain rule".
\begin{assumption}\label{ass:A5}\ \\
	For any $Y \in \mathcal{C}_{\mathrm{abs}}(\R_+, \R^D),\; \ull{\forall} s > 0,\; \forall V \in \partial_CF(Y(s)),\; V^T \dot Y(s) = (F \circ Y)'(s)$.
\end{assumption}

In order to satisfy \ref{ass:A5}, we will use the following result:
\begin{prop}Any  $F: \R^D \longrightarrow \R$ locally Lipschitz and semi-concave admits a chain rule for the Clarke sub-differential, and thus satisfies \ref{ass:A5}.\label{prop:semi_concave_chain_rule}
\end{prop}

\begin{proof}
	Let $F: \R^D \longrightarrow \R$ locally Lipschitz and semi-concave. By Vial
	(1983)~\cite{vial1983strong}, Proposition 4.5, this implies that $-F$ is
	Clarke regular. Then, by Bolte and Pauwels \cite{bolte2021conservative},
	Proposition 2, the fact that $-F$ is Clarke regular implies that $F$ is path
	differentiable, and thus admits a chain rule, by Bolte et al.
	\cite{bolte2021conservative}, Corollary 2.
\end{proof}
Since $\SWY$ is semi-concave (\ref{prop:SWY_semi_concave}) and locally Lipschitz, \ref{prop:semi_concave_chain_rule} allows us to verify \ref{ass:A5} for \ref{eqn:SWY_noisedSGD}. We may follow the same line of thought for $\SWpY$, or alternatively we may use the fact that it is semi-algebraic (\ref{prop:SWpY_semi_algebraic}). By Bolte and Pauwels (2021),~\cite{bolte2021conservative}, Proposition 2, this implies that $\SWpY$ is path differentiable. Then by Bolte and Pauwels~\cite{bolte2021conservative}, Corollary 2, path differentiability implies having a chain rule for the Clarke sub-differential, which is verbatim~\cite{bianchi2022convergence}, \ref{ass:A5}. We now have all the assumptions for~\cite{bianchi2022convergence}, Theorem 3:

\begin{theorem}[Applying~\cite{bianchi2022convergence}, Theorem 3: convergence of noised SGD schemes to a critical point]\label{thm:noised_SGD_cv}\
	Consider a collection of noised SGD schemes $(Y_\alpha^{(t)})$, associated to~\ref{eqn:SWY_noisedSGD}, respectively for $F \in \lbrace \SWY,\ \SWpY\rbrace$, with steps $\alpha \in\ ]0, \alpha_0]$, with $\alpha_0 < 1$. Let $\mathcal{Z}$ the set of Clarke critical points of $F$, i.e. $\mathcal{Z} := \left\lbrace Y \in \R^D\ |\ 0 \in \partial_CF(Y) \right\rbrace$. For $\mu \in \lbrace \bbsigma,\ \bbsigma_p\rbrace$ respectively, we have:
	$$\forall \varepsilon > 0,\; \underset{t \longrightarrow +\infty}{\oll{\lim}}\ \nu \otimes \mu^{\otimes \N}\otimes \eta^{\otimes \N}\left(d(Y_\alpha^{(t)}, \mathcal{Z}) > \varepsilon\right) \xrightarrow[\substack{\alpha \longrightarrow 0\\ \alpha \in ]0, \alpha_0]}]{} 0.$$
\end{theorem}

It is to be understood that the euclidean distance between any sub-sequential limit of $(Y_\alpha^{(t)})_t$ and set of Clarke critical points $\mathcal{Z}$ approaches 0 in probability as the step size decreases. The distance $d$ in the Theorem refers to the $\|\cdot\|_2$-induced distance between the point $Y_\alpha^{(t)}\in \R^D$ and the set $\mathcal{Z} \subset \R^D$.

\blue{Computing the set of Clarke critical points of $\SWY$ remains an open problem, and seems out of reach considering the difficulty of the simpler problem of computing the points where $\SWY$ is differentiable and $\nabla \SWY = 0$ (see the discussion in \ref{sec:E_crit}). The difficulty lies at the boundaries of $\mathcal{U}$ (see \ref{eqn:U}), where there is no longer unicity of the sorting permutations of $(\theta^T x_k)_{k=1,\dots,n}$ and $(\theta^T z_l)_{l=1,\dots,n}$ for $\bbsigma$-almost-every $\theta \in \SS^{d-1}$. Computing the Clarke sub-differential at such points in closed form and determining the associated critical points seems out of reach since there is already no closed form for smooth critical points \ref{cor:crit_points_S}.}

For $\SWpY$, the set of Clarke critical points strictly contains the set of
critical points established in \ref{thm:Sp_crit_optloc_stable}. \blue{In
general, the set of Clarke critical points that lie outside of the set of
differentiability $\widetilde{\mathcal{Z}} := \mathcal{Z} \cap
\left(\cup_\config \mathcal{C}_\config\right)^c$ is not empty, yet by
\ref{thm:Sp_crit_optloc_stable} it cannot contain a local optimum, and thus only
contains saddle points. We believe that these saddle points will in practice
never be the limit of our noised SGD trajectories, since intuition suggests that
a trajectory attaining such a point at a certain time will find a decreasing
direction almost-surely. Showing such a result rigorously is out of the scope of
this paper since this question is still an active field in simpler smooth
cases~\cite{jin2017escape,jin2021nonconvex}. } \bluetwo{More precisely, the
minimisation of generic non-smooth non-convex functions $F$ is also still
actively studied. For instance, \cite{davis2022proximal} and
\cite{bianchi2023stochastic} investigate conditions to avoid convergence to
certain saddle points, for randomly initialised deterministic (and potentially
noised) proximal methods for semi-convex functions under novel strict saddle
conditions. Another related reference is \cite{davis2023active}, which studies
the non-convergence of noised sub-gradient descent to saddle points.} We
illustrate in \ref{fig:SWpY_sym_p3_Clarke} the Clarke critical points of $\SWpY$
for $p=3$, on the numerical example of \ref{sec:L2}.

\begin{figure}
	\centering
	\begin{subfigure}[c]{0.45\textwidth}
		\centering
		\includegraphics[width=\linewidth]{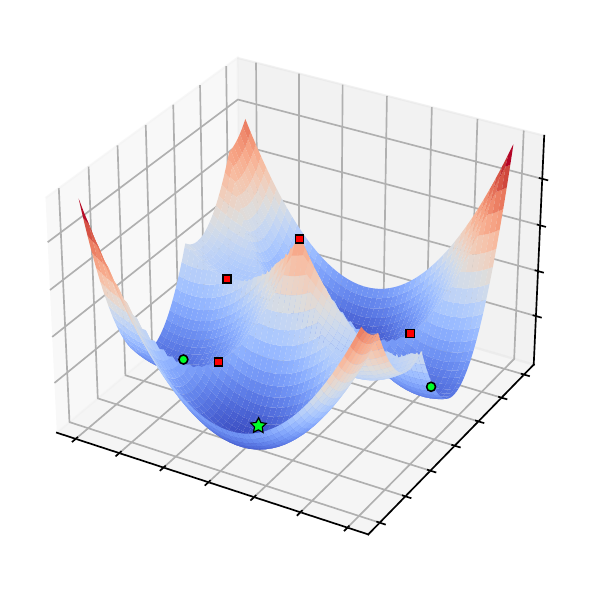}
	\end{subfigure}
	\begin{subfigure}[c]{0.45\textwidth}
		\centering
		\includegraphics[width=\linewidth]{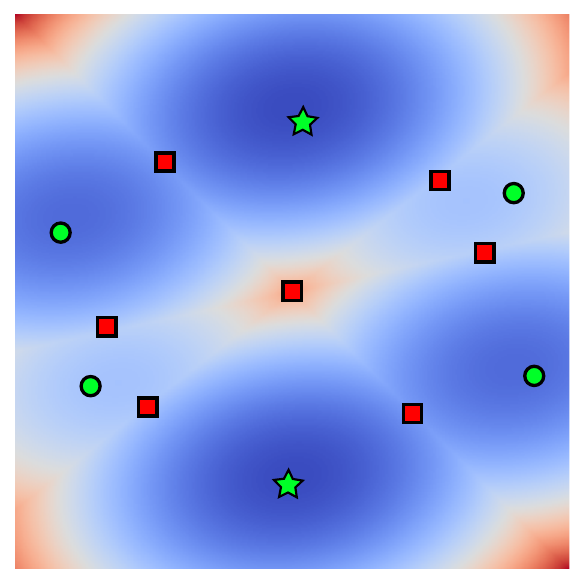}
	\end{subfigure}
	\caption{The stars, circles and squares are the Clarke critical points of $x \mapsto \SWpY(X = (-x, x)^T),\; x \in \R^2$ for $p=3$. The squares do not correspond to local optima of $\SWpY$, and are unlikely to be reached numerically. The circles and stars correspond to local optima of $\SWpY$: the stars correspond to the global optima and satisfy the desired results $\SWpY = 0$, while the circles are strict local optima. }
	\label{fig:SWpY_sym_p3_Clarke}
\end{figure}
In full generality (without the symmetry restriction, and with larger parameters $p, \npoints, d$), \blue{one may expect that} the Clarke critical points will have a similar structure.

\subsection{Discussion on result generalisation}

\textbf{Batching.} One may consider a variant in which at each step $t$, one draws a random batch of $b$ directions independently from a measure $\mu$ over $\SS^{d-1}$ ($\mu \in \lbrace \bbsigma, \bbsigma_p\rbrace$, for our purposes). Algorithmically, one does the following SGD scheme:
\begin{equation}\label{eqn:batched_SGD}
	Y^{(t+1)} = Y^{(t)} - \cfrac{\alpha}{b}\Sum{j=1}{b} \varphi(Y^{(t)}, \theta_j^{(t+1)}),\quad (Y^{(0)}, (\theta_j^{(t)})_{\substack{j \in \llbracket 1, b \rrbracket \\ t \in \N}}) \sim \nu \otimes (\mu^{\otimes b})^{\otimes \N}.
\end{equation}
In order to fit our theoretical framework (see \ref{sec:bianchi_setting}), we define 
$$g(Y, (\theta_1, \cdots, \theta_b)) := \cfrac{1}{b}\Sum{j=1}{b}f(Y, \theta_j).$$
Furthermore, the a.e. gradient of $g$ becomes $\psi(\cdot, (\theta_1, \cdots, \theta_b)) := \cfrac{1}{b}\Sum{j=1}{b}\varphi(\cdot, \theta_j)$ instead of $\varphi(\cdot, \theta^{(t)})$. The function over which \ref{eqn:batched_SGD} performs SGD is:
\begin{align*}
  G(Y)&= \Int{\SS^{d-1}}{}g(Y, \theta_1, \cdots, \theta_b)\dd\mu^{\otimes b}(\theta_1, \cdots, \theta_p) \\
  &= \Int{\SS^{d-1}}{}\cfrac{1}{b}\Sum{j=1}{b}f(Y, \theta_j)\dd\mu^{\otimes b}(\theta_1, \cdots, \theta_p) \\
  &= \Int{\SS^{d-1}}{}f(Y, \theta)\dd\mu(\theta) = F(Y).
\end{align*}
One may check easily that if Assumptions 1 through 5 of \ref{sec:bianchi_setting} are satisfied for $(f, F)$, then they are satisfied for $(g, G)$. As a consequence, all our results can be adapted without any difficulty to the batched setting.

\textbf{Barycentres.} If one were to replace $\SWY$ with the barycentre energy $\SWYbar$ \ref{eqn:bar}, the sample loss would become 
$$g(Y, \theta) = \Sum{j=1}{J}\lambda_jf_j(Y, \theta_j),\; \text{where}\; f_j(Y, \theta) := \W_2^2(P_\theta\#\gamma_Y, P_\theta\#\gamma_{Z^{(j)}}).$$
By sum, all of the previous results will hold, with the only technical point
being path differentiability, which is stable by sum
(\cite{bolte2021conservative}, Corollary 4). Note that this extension is also
valid for a Monte-Carlo approximation of $\SWY$, replacing $\SWY$ with $\SWpY$
in the barycentre formulation.

\subsection{A Result for Decreasing Learning Rates}\label{sec:decr_sgd}

\bluetwo{In \cite{davis2020stochastic}, Davis et al. show the convergence of
\textit{decreasing-step} noised SGD of a function $F$ under certain conditions.
Our goal is to apply their Theorem 4.2 to $F\in \{\SWY, \SWpY\}$ with the
following SGD scheme:
\begin{align}\label{eqn:noised_sgd_decreasing_lr}
	Y^{(t+1)} &= Y^{(t)} - \alpha^{(t)} \varphi(Y^{(t)}, \theta^{(t+1)}) + \alpha^{(t)} a\varepsilon^{(t+1)}, \\
	& (Y^{(0)}, (\theta^{(t)})_{t \in \N}, (\varepsilon^{(t)})_{t \in \N}) \sim \nu \otimes \mu^{\otimes \N} \otimes \eta^{\otimes \N}, \nonumber
\end{align}
where as before, $\mu \in \{\bbsigma, \bbsigma_p\}$ for $\SWY, \SWpY$
respectively, $\nu$ is the distribution of the initial position $Y^{(0)}$, $\eta
:= \mathcal{N}(0, I_D)$ is the noise distribution (it could be chosen more
generally, but we attempt to stay close to our previous formalism for
simplicity), and finally the learning rate sequence $(\alpha^{(t)})$ verify:
$$\forall t\in \R,\; \alpha^{(t)}\geq 0,\quad
\Sum{t=0}{+\infty}\alpha^{(t)}=+\infty,\; \text{and} \;
\Sum{t=0}{+\infty}(\alpha^{(t)})^2 < +\infty.$$

\begin{theorem}\label{thm:cv_decreasing_lr} Consider $(Y^{(t)})$ a trajectory of
	\ref{eqn:noised_sgd_decreasing_lr} for $\mu\in \{\bbsigma, \bbsigma_p\}$
	respectively, assume that it is almost-surely bounded. Then the sequence
	$F(Y^{(t)})$ is almost-surely convergent for $F\in \{\SWY, \SWpY\}$
	respectively, and almost-surely, any subsequential limit of $(Y^{(t)})$
	belongs to the set of Clarke critical points of $F$.
\end{theorem}

\begin{proof}
	We verify assumptions C.1, C.2, C.3 and D.1, D.2 of
	\cite{davis2020stochastic}, allowing us to apply Theorem 4.2. To begin with,
	the sequence $(\alpha^{(t)})$ was chosen to verify assumption C.1, and
	\ref{thm:cv_decreasing_lr} explicitly assumes C.2. Regarding C.3, the simple
	choice of independent noise verifies the martingale difference assumption
	trivially. 

	For D.1, we need to show that the set of non-critical points of $\SWY$ and
	$\SWpY$ are dense in $\R^D$. For $\SWY$, we can use \ref{cor:crit_points_S},
	which implies that critical points $Y$ of $\SWY$ necessarily verify $\sum
	y_k = \sum_k z_k,$ since $\sum_kS_{k,l}^{Y,Z}=I_d$, or are points of
	non-differentiability, which implies belonging to $\mathcal{U}$ (see
	\ref{eqn:U}) In particular, critical points are necessarily within a union
	of two strict subspaces of $\R^D$, whose complementary is dense in $\R^D$.
	For $\SWpY$, the property is easily verified using its decomposition into
	(non-trivial) quadratics on cells \ref{prop:SWpY_quadratic_on_each_cell}.

	For D.2, we leverage \cite{davis2020stochastic} Lemma 5.2 along with the
	fact that $\SWY$ and $\SWpY$ are path-differentiable (\ref{sec:noised_sgd}),
	which shows that D.2 is verified.
\end{proof}

The assumption that the trajectories are almost-surely bounded can be seen as a
discrete version of \cite{Li:2023aa} Assumption 4.4: A1), which Li and
Moosmüller require to prove the convergence of their decreasing-step SGD scheme
for SW between absolutely continuous measures. While this assumption is
theoretical costly, numerically we observe that the measure support $Y$ remains
bounded. Nevertheless, lifting this assumption would be of substantial
mathematical interest.}

	\section{Numerical Experiments}\label{sec:xp}

This section illustrates the optimisation properties of $\SWY$ and $\SWpY$
with several numerical experiments. \ref{sec:xp_bcd} studies the optimisation of
$\SWpY$ using the BCD algorithm described in \ref{alg:BCD}, which offers insights on the
cell structure of $\SWpY$ (\ref{sec:cells}). \ref{sec:traj_sgd} focuses on
stochastic gradient descent \ref{alg:SGD} and showcases various SGD trajectories
on $\SWY$ and $\SWpY$ for different learning rates, noise levels or numbers of
projections, as well as the Wasserstein error along iterations. All the
convergence curves shown throughout our experiments also showcase margins of
error, computed by repeating the experiments several times, and corresponding to
the 30\% and 70\% quantiles of the experiment.

In order to assess the quality of a position $Y^{(t)}$, perhaps the most germane
metric is the Wasserstein distance: $\W_2^2(\gamma_{Y^{(t)}}, \gamma_Z)$, which
is why we will study the 2-Wasserstein error of BCD and SGD trajectories in this section.
Unfortunately, this metric is not quite comparable for different dimensions $d$,
notably because $\|(1, \cdots, 1)\|_2^2 = d$. We shall attempt to compensate
this phenomenon by using $\frac{1}{d}\W_2^2(\gamma_{Y^{(t)}}, \gamma_Z)$
instead, which makes the metric more comparable for measures on spaces of
different dimensions.

\subsection{Empirical study of Block Coordinate Descent on \texorpdfstring{$\SWpY$}{Ep}}\label{sec:xp_bcd}

In this section, we shall focus on studying the optimisation properties of the
$\SWpY$ landscape using the BCD algorithm (\ref{alg:BCD}). This method leverages
the cell structure of $\SWpY$ (see \ref{sec:cells}), by moving from cell to cell
by computing the minimum of their associated quadratics (see the discussion in
\ref{sec:Ep_crit_and_bcd}). By \ref{thm:Sp_crit_optloc_stable}, all local optima
of $\SWpY$ are stable cell optima, i.e. fixed points of the BCD, which summarises
briefly the ties between BCD and the optimisation properties of $\SWpY$. As for
the numerical implementation, \ref{alg:BCD} was implemented in Python with
Numpy \cite{harris2020array} using the closed-form formulae for the
updates.

\subsubsection{Illustration in 2D}\label{sec:traj_bcd}

\paragraph{Dataset and implementation details.} 
We start by setting a simple 2D measure $\gamma_Z$ with a support of only
two points represented with stars in \ref{fig:traj_bcd_seed0}. The measure weights are taken as uniform.
We fix sequences of $p$ projections $(\theta_1,
\cdots, \theta_p)$ for $p\in \lbrace 3, 10, 30, 100\rbrace$ respectively. We
then draw 100 BCD schemes with different initial positions $Y^{(0)} \in
\R^{2\times 2}$, drawn with independent standard Gaussian entries. We take a
stopping criterion threshold of $10^{-5}$ (see \ref{alg:BCD}), and limit to 500
iterations. 

\begin{figure}
	\centering
	\includegraphics[width=\linewidth]{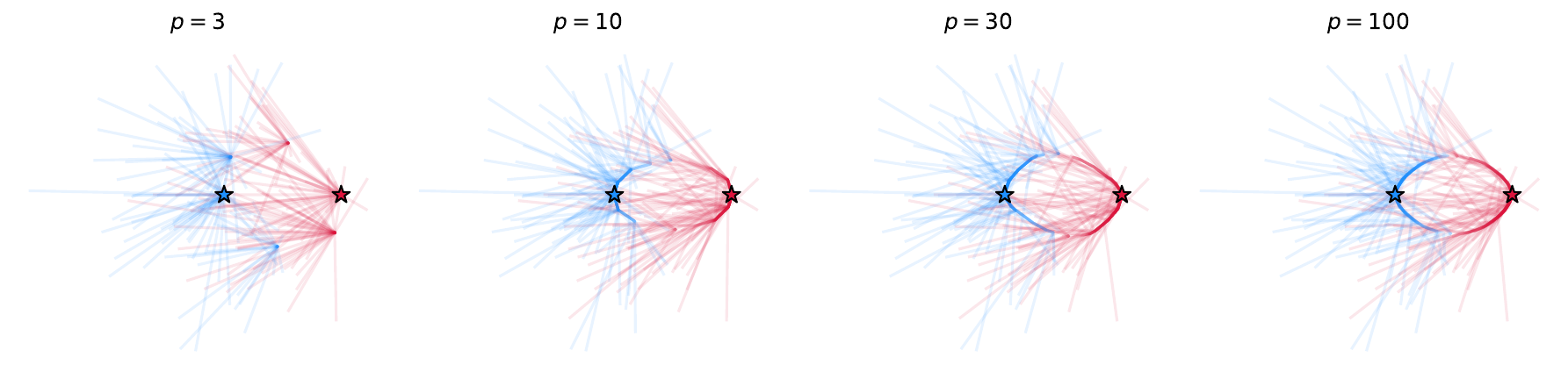}
	\caption{BCD on $\SWpY$ with different initial positions $Y^{(0)}$, with fixed projections (first sample). Each of the two points of the trajectory $Y^{(t)} = (y_1^{(t)}, y_2^{(t)})$ is coloured with respect to the point of the original measure $\gamma_Z$ to which they converge.}
	\label{fig:traj_bcd_seed0}
\end{figure}

In the case
$p=3$, we observe on \ref{fig:traj_bcd_seed0} \blue{four} points which correspond to strict local optima, and the
schemes appear to have a comparable probability of converging towards each of
them. Note that these points are essentially the same as the ones represented in
\ref{fig:SWpY_sym_p3} for $p=3$, but that they depend on the projection sample. 
Between the two projection realisations, we
observe that these local optima change locations. The cases $p \in \lbrace 10,
30, 100\rbrace$ also exhibit strict local optima, however they appear to be
decreasingly likely to be converged towards. For $p=30$ and $p=100$, notice that
most trajectories end up on the same ellipsoid arcs towards the solution $Z$,
and further remark that these arcs strongly resembles the trajectories of SGD
schemes on $\SWY$ for small learning rates (see \ref{fig:traj_sgd_E_lr_multiple}
in \ref{sec:xp_sgd}).

\subsubsection{Wasserstein convergence of BCD schemes on \texorpdfstring{$\SWpY$}{Ep}}

\paragraph{Final Wasserstein error of BCD Schemes.} For a dimension $d \in \lbrace 10, 30, 100\rbrace$ and $\npoints=20$ points, the original measure $\gamma_Z,\; Z \in \R^{\npoints \times d}$ is sampled once for all with independent standard Gaussian entries. Then, for varying numbers of projections $p$, we draw a starting position $Y^{(0)} \in \R^{\npoints \times d}$ with entries that are uniform on $[0, 1]$; and draw $p$ projections as input to the BCD algorithm. We set the stopping criterion threshold as $\varepsilon=10^{-5}$ and the maximum iterations to 1000. In order to produce \ref{fig:BCD_its}, we record the normalised 2-Wasserstein discrepancy $\frac{1}{d}\W_2^2(\gamma_{Y^{(T)}}, \gamma_Z)$ at the final iteration $T$ for 10 realisations for each value of $p$ and $d$.

\begin{figure}
	\centering
	\begin{subfigure}[c]{0.3\textwidth}
		\centering
		\includegraphics[width=\linewidth]{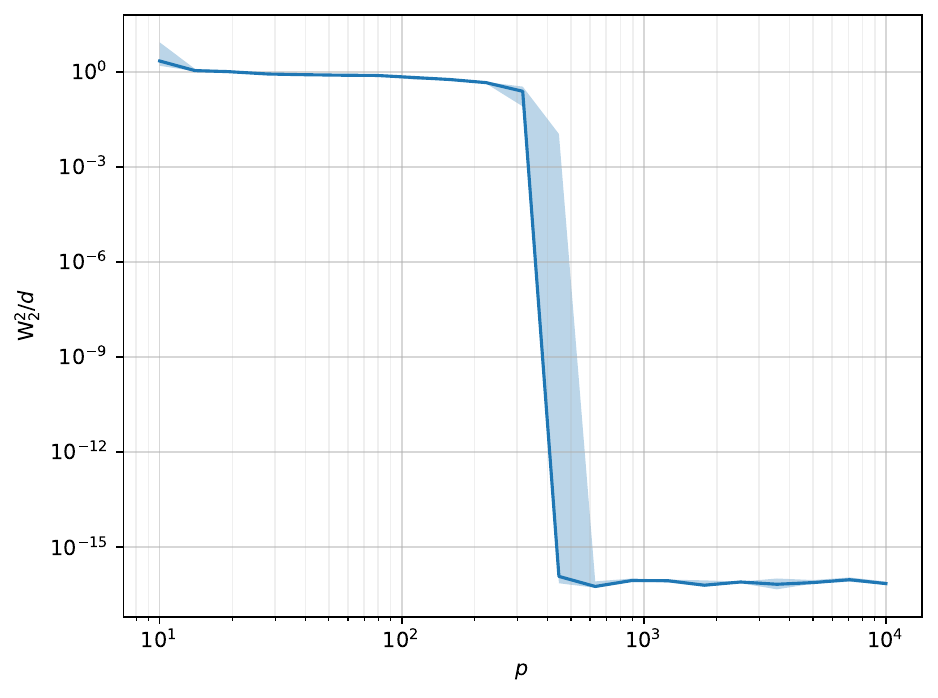}
		\subcaption{$d=10$}
	\end{subfigure}
	\begin{subfigure}[c]{0.3\textwidth}
		\centering
		\includegraphics[width=\linewidth]{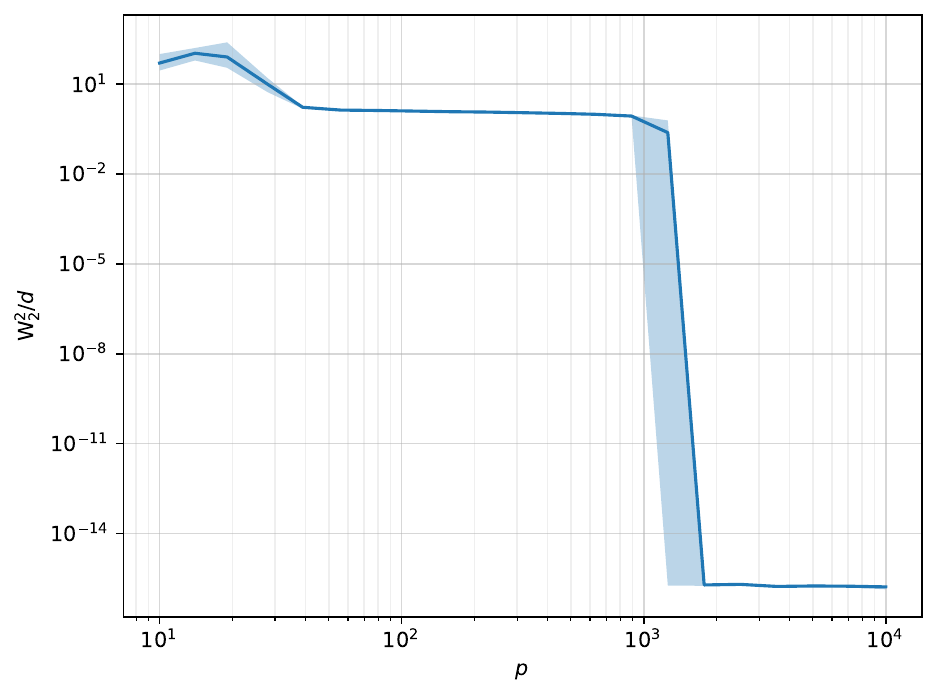}
		\subcaption{$d=30$}
	\end{subfigure}
	\begin{subfigure}[c]{0.3\textwidth}
		\centering
		\includegraphics[width=\linewidth]{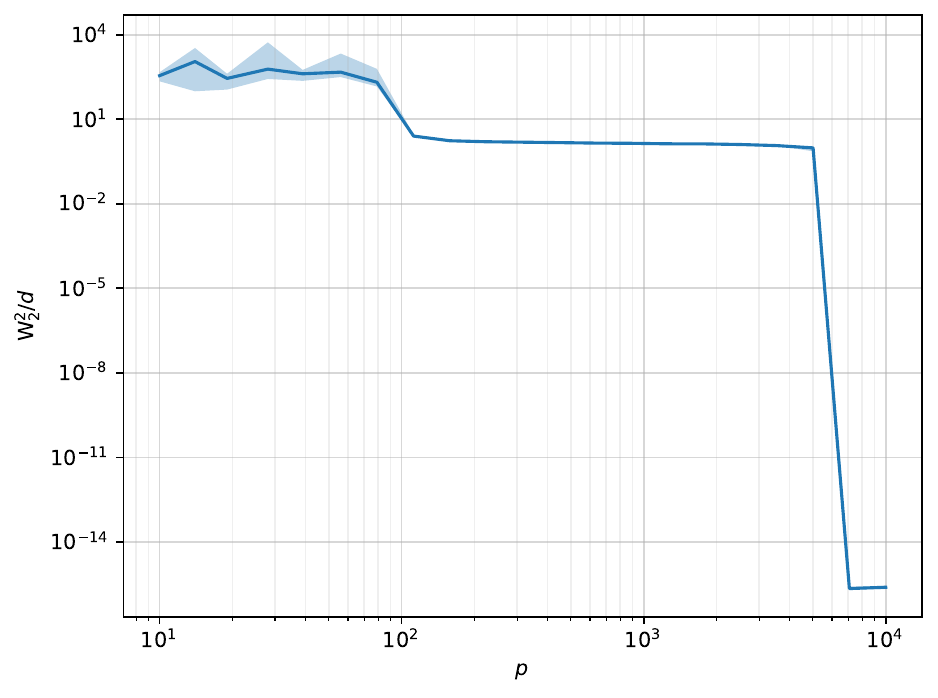}
		\subcaption{$d=100$}
	\end{subfigure}
	\caption{We consider BCD schemes with different amounts of projections $p$, and with an original measure $\gamma_Z$ comprised of $\npoints=10$ points in dimension $d \in \lbrace 10, 30, 100 \rbrace$, which is fixed as a standard Gaussian realisation for each value of $d$. The stopping threshold was chosen as $\varepsilon = 10^{-5}$, and we plot the final Wasserstein errors $\frac{1}{d}\W_2^2(\gamma_{Y^{T}}, \gamma_Z)$ at the final iteration $T$. For each set of values for the parameters, we perform 10 realisations with different initialisations $Y^{(0)}$ (drawn with uniform $[0, 1]$ entries), and different projections $(\theta_1, \cdots, \theta_p)$.}
	\label{fig:BCD_its}
\end{figure}

As a first estimation of the difficulty of optimising $\SWpY$, we consider the evolution - as $p$ increases - of final $\W_2^2$ errors of BCD schemes. The results of the experiments presented in \ref{fig:BCD_its} suggests the existence of a phase transition between an insufficient and a sufficient amount of projections. For instance, in the case $d=10$, there appears to be a cutoff around $p=400$, under which all the BCD realisations converge towards strict local optima, and past which we observe convergence up to numerical precision.

\paragraph{Probability of convergence of BCD schemes.} We can investigate further this empirical cutoff phenomenon by estimating the probability of convergence of a BCD algorithm. This probability is loosely related to the difficulty of optimising the landscape $\SWpY$, since a high probability of BCD convergence indicates either a small number of strict local optima, or that their corresponding cells are extremely small and seldom reached in practice. For varying numbers of projections $p$ and dimensions $d$, we run 100 realisations of BCD schemes. Each sample draws a target measure $\gamma_Z,\; Z \in \R^{\npoints \times d}$ with independent standard Gaussian entries and $\npoints=10$ points, as well as its initialisation $Y^{(0)} \in \R^{\npoints \times d}$ with entries that are uniform on $[0, 1]$ and $p$ projections. Every BCD scheme has a stopping threshold of $\varepsilon=10^{-5}$ and a maximum of 1000 iterations. We consider that a sample scheme has converged (towards the global optimum $\gamma_Z$) if $\frac{1}{d}\W_2^2(\gamma_{Y^{(T)}}, \gamma_{Z}) < 10^{-5}$, which allows us to compute an empirical probability of convergence for each value of $(p, d)$.

\begin{figure}
	\centering
	\begin{subfigure}[c]{0.45\textwidth}
		\centering
		\includegraphics[width=\linewidth]{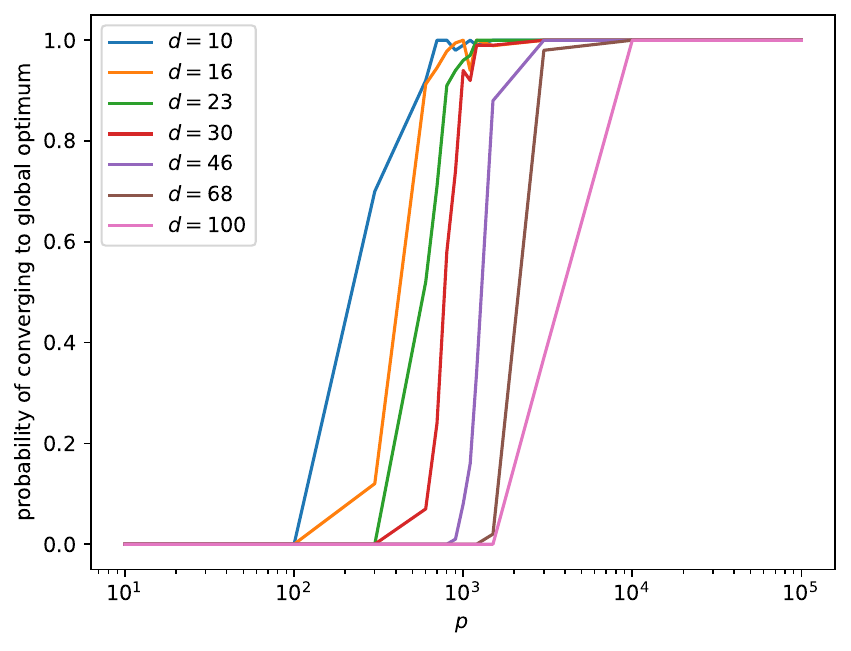}
	\end{subfigure}
	\begin{subfigure}[c]{0.45\textwidth}
		\centering
		\includegraphics[width=\linewidth]{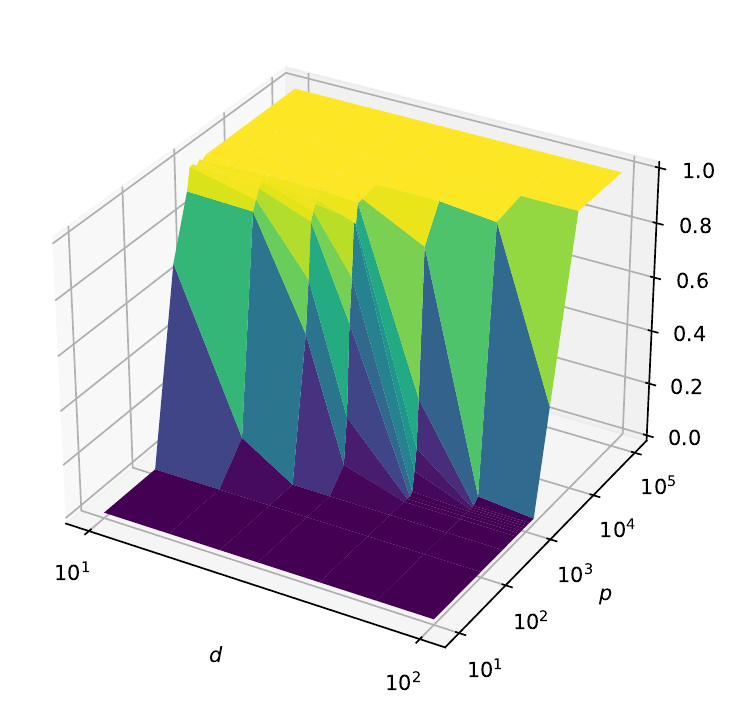}
	\end{subfigure}
	\caption{Given a number of projections $p$, we run 100 BCD trials with different initial positions (with entries drawn as uniform on $[0, 1]$), projections and target measure supported by $Z \in \R^{\npoints \times d}$, with $\npoints=10$ points in different dimensions $d \in [10, 100]$, where $Z$ is drawn with independent standard Gaussian entries. At the final iteration $T$, we determine whether the optimum is global by a threshold criterion: $\frac{1}{d}\W_2^2(\gamma_{Y^{(T)}}, \gamma_Z) < 10^{-5}$ and compute an empirical probability of convergence.}
	\label{fig:cv_proba}
\end{figure}

The findings in \ref{fig:cv_proba} indicate that the $\W_2^2$ error cutoffs from \ref{fig:BCD_its} have a probabilistic counterpart: the probability of converging to a global optimum transitions from almost 0 to almost 1 relatively suddenly (in the logarithmic scale). We can conjecture that this drop in optimisation difficulty is tied to the number of iterations needed for the convergence of SGD schemes on $\SWY$, especially given the similar behaviour for the $\W_2^2$ error in \ref{fig:SGD_its_E_alpha}.

\subsection{Empirical study of SGD on \texorpdfstring{$\SWY$}{E} and
\texorpdfstring{$\SWpY$}{E}}\label{sec:xp_sgd}

\paragraph{General numerical implementation.} In order to perform gradient
descent on $\SWY$ or $\SWpY$, we compute the gradient \ref{eqn:ae_grad} using
{Pytorch}'s 
\cite{pytorch} Stochastic Gradient Descent optimiser, which back-propagates
gradients through the loss $w_\theta:= Y \mapsto\W_2^2(P_{\theta}\#\gamma_{Y},
P_{\theta}\#\gamma_{Z})$, which we compute using the 1D Wasserstein solver from
{Python Optimal Transport} \cite{flamary2021pot}.

\subsubsection{Illustration in 2D}\label{sec:traj_sgd}

\paragraph{2D dataset and implementation details.} We define a 2D spiral dataset with the measure $\gamma_Z,\; Z = (z_1, \cdots, z_{10})^T\in \R^{10\times 2}$ with $z_k = \frac{2k}{10}\left(\cos\left(2k\pi / 10\right), \sin\left(2k\pi / 10\right)\right)^T,$ and $k \in \llbracket 1, 10 \rrbracket$. The initial position $Y^{(0)}$ is fixed and remains the same across realisations. For schemes on $\SWY$, the projections $\theta^{(t)} \sim \bbsigma$ are fixed beforehand and are the same across experiments. For every realisation of a scheme on $\SWpY,\; p$ unique projections $(\theta_1, \cdots, \theta_p)$ are drawn, then the projections $(\theta^{(t)})$ for the iterations  are drawn from these $p$ fixed projections. For noised schemes, the only variable that is drawn at every sample is the noise $(\varepsilon^{(t)})$. Note that the associated energy landscapes are extremely similar to those illustrated in \ref{sec:L2} and in particular in \ref{fig:SWpY_sym_p3}.

\begin{figure}
	\centering
	\includegraphics[width=\linewidth]{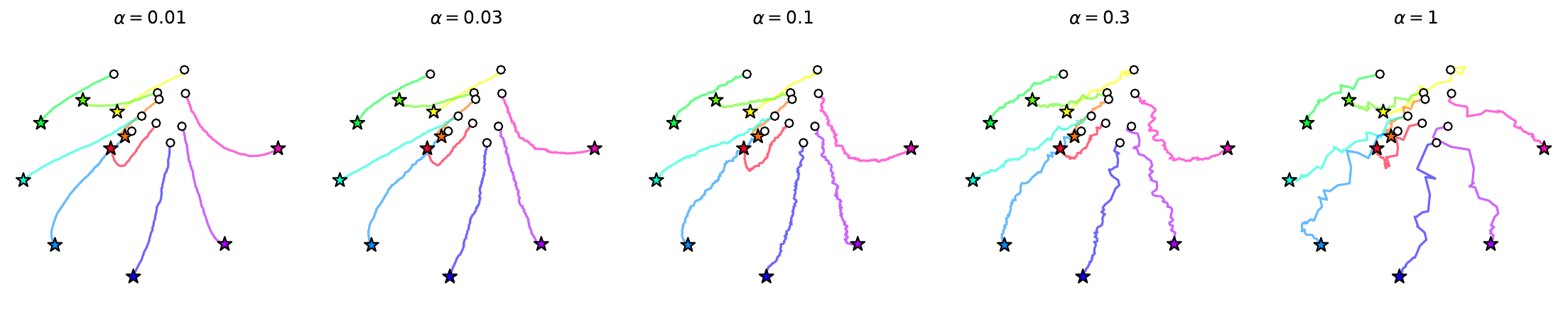}
	\caption{SGD trajectories on $\SWY$ for different learning rates $\alpha$.
	All the trajectories are computed using the same projection sequence
	$(\theta^{(t)})$.}
	\label{fig:traj_sgd_E_lr}
\end{figure}

\begin{figure}
	\centering
	\includegraphics[width=\linewidth]{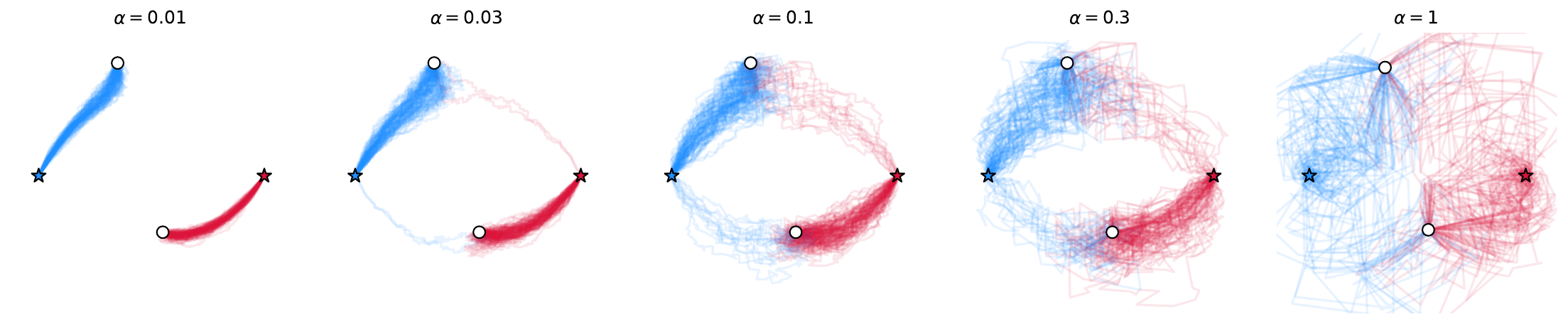}
	\caption{SGD trajectories on $\SWY$ for different learning rates $\alpha$.
	For each value of $\alpha$, 100 samples are drawn with different projections
	$(\theta^{(t)})$, and for each realisation, each of the two points of the
	trajectory is coloured with respect to the point of the original measure
	$\gamma_Z$ (represented by stars) to which they converge. The initial
	position $Y^{(0)}$ is represented by circles.}
	\label{fig:traj_sgd_E_lr_multiple}
\end{figure}

\ref{fig:traj_sgd_E_lr} and \ref{fig:traj_sgd_E_lr_multiple} illustrate the
convergence of SGD schemes on $\SWY$ towards the original measure $\gamma_Z$,
for different learning rate $\alpha$ (provided that $\alpha$ is under a divergence threshold).
\ref{thm:SGD_interpolated_cv} allowed us only to expect a convergence to a
\textit{solution of a Clarke Differential Inclusion} on $\SWY$
\ref{eqn:Clarke_DI}, yet in practice we seem to have convergence to a global
optimum. Furthermore, \ref{thm:SGD_interpolated_cv} shows that the interpolated
SGD trajectories are approximately solutions of the DI $\dot{X}(t) \in
-\partial_C\SWY(X(t))$, which, assuming that the trajectory stays in
$\mathcal{U}$, amounts to $\dot{X}(t) + \nabla \SWY (X(t)) = 0$, which is
exactly the Euclidean Gradient Flow of $\SWY$, as discussed in more detail in
\ref{rem:flows}. This illustration suggests that the SGD schemes approach the
gradient flow \ref{eqn:SWY_flow} as $\alpha \longrightarrow 0$, whereas
\ref{thm:SGD_interpolated_cv} predicts a (weak) convergence towards the set of
solutions of the DI \ref{eqn:Clarke_DI}, which is equal to the gradient flow
provided that the initial position $Y^{(0)}$ belongs to the differentiability set
of $\SWY$ (see \ref{rem:flows} for details). Note that higher learning rates
lead to a "noisier" trajectory, which may impede upon the quality of the
assignment. This shows that there is a trade-off: lower values of $\alpha$ allow
for a better approximation of the (or a) gradient flow of $\SWY$ and potentially
a more precise final position $Y$ and assignment $\tau$, however a larger value
of $\alpha$ yields a substantially faster convergence.

\begin{figure}
	\centering
	\includegraphics[width=\linewidth]{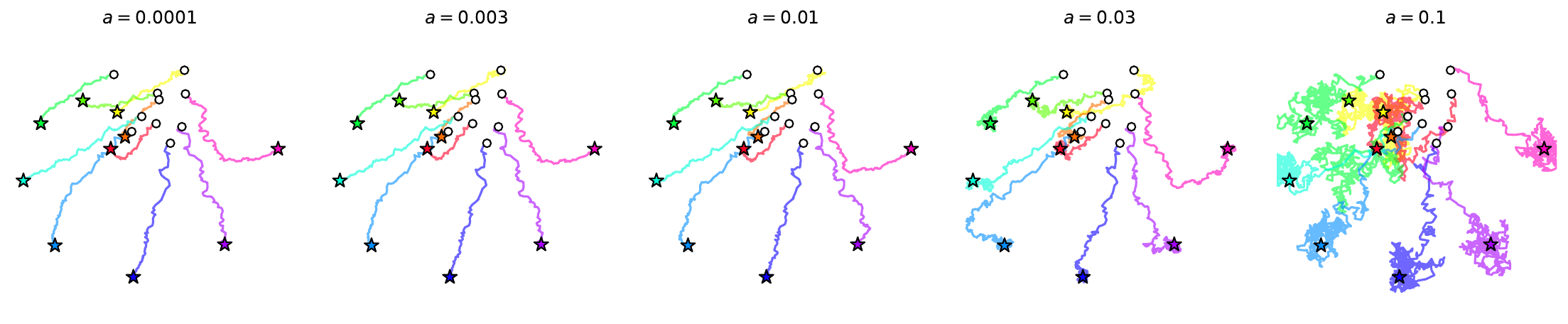}
	\caption{SGD trajectories on $\SWY$ for different noise levels $a$. All the
	trajectories are computed using the same projection sequence
	$(\theta^{(t)})$. The learning rate is fixed at $\alpha = 0.3$.}
	\label{fig:traj_sgd_E_noise}
\end{figure}

\ref{fig:traj_sgd_E_noise} presents a case where noised SGD schemes on $\SWY$ "converge" whatever the noise level to a global optimum of $\SWY$. Note that the additive noise causes the scheme to oscillate around a solution, with a movement akin to Brownian motion with a scale tied to $\alpha a$. \ref{thm:noised_SGD_cv} shows that such schemes converge (as the step approaches 0) to \textit{Clarke critical points} of $\SWY$, which could theoretically be a saddle point of strict local optimum. In this experiment,  we observe convergence to a global optimum.

\begin{figure}
	\centering
	\includegraphics[width=\linewidth]{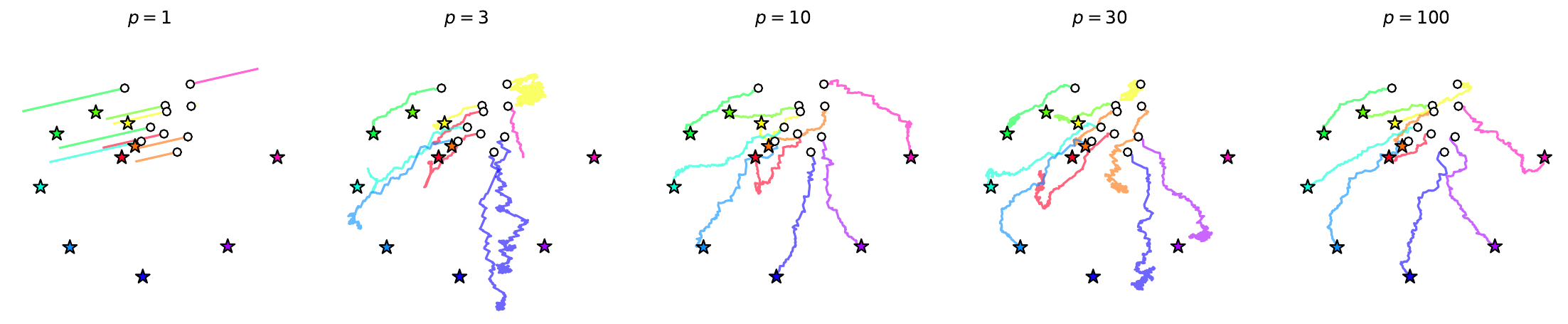}
	\caption{SGD schemes on $\SWpY$ for different number of projections $p$. The learning rate is fixed at $\alpha = 0.3$.}
	\label{fig:traj_sgd_Ep}
\end{figure}

\begin{figure}
	\centering
	\includegraphics[width=\linewidth]{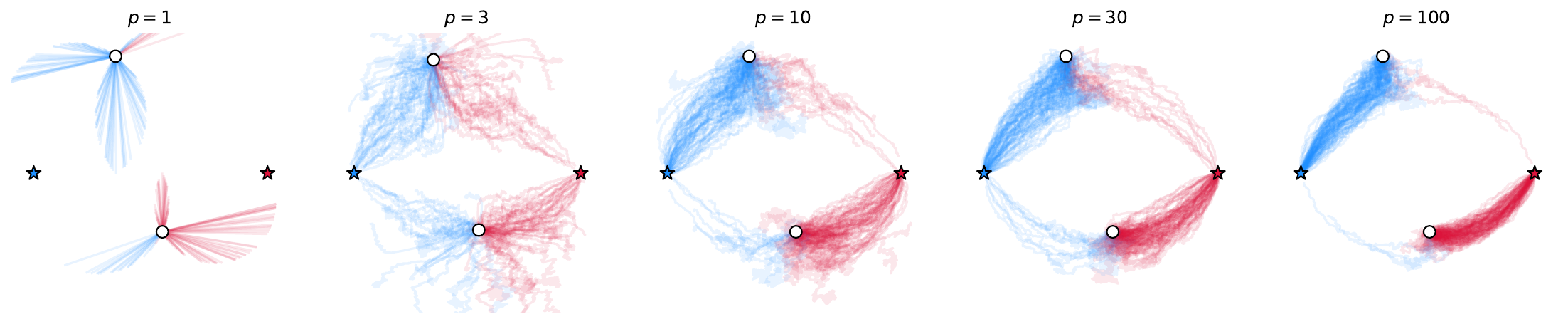}
	\caption{SGD schemes on $\SWpY$  for different number of projections $p$. For each value of $p$, 100 samples are drawn with different projections $(\theta_1, \cdots, \theta_p)$. For each realisation, each of the two points of the trajectory is coloured with respect to the point of the original measure $\gamma_Z$ (represented by stars) to which they converge. The initial position $Y^{(0)}$ is represented by circles. The learning rate is fixed at $\alpha = 0.03$.}
	\label{fig:traj_sgd_Ep_multiple}
\end{figure}

\ref{fig:traj_sgd_Ep} illustrates that SGD schemes on $\SWpY$ may converge to strict local optima, which is to be expected, given how numerous they may be (see the discussion in \ref{sec:L2} and \ref{fig:SWpY_sym_p3} therein). For $p=1$, entire lines are local optima, and for $p=3$ and $p=30$, we also observe convergence to strict local optima. Notice that for a large value of $p$ such as $p=100 \gg d=2$, we have similar trajectories in \ref{fig:traj_sgd_Ep_multiple} compared to the $\SWY$ counterpart in \ref{fig:traj_sgd_E_lr_multiple} ($\alpha=0.03$). This observation suggests a stronger property than our results on the approximation of $\SWY$ by $\SWpY$: uniform convergence in \ref{thm:Sp_cvu_S} and a weak link between critical points \ref{thm:cv_fixed_point_distance}. To be precise, this illustration could allow one to hope for a result on the high probability for the proximity of SGD schemes on $\SWpY$ and on $\SWY$ as $p \longrightarrow +\infty$, perhaps with conditions on the sequence of projections $(\theta^{(t)})$.

\subsubsection{Wasserstein convergence of SGD schemes on \texorpdfstring{$\SWY$}{E} and \texorpdfstring{$\SWpY$}{Ep}}

\paragraph{SGD on $\SWY$.}\label{para:numerical_setup_sgd_E} The original measure $\gamma_Z,\; Z \in \R^{\npoints \times d}$ is sampled once for all with independent standard Gaussian entries. For each value of the parameter of interest (the learning rate $\alpha$ or the dimension $d$ respectively), 10 realisations of the SGD schemes are computed with a different initial position $Y^{(0)}$, drawn with independent entries uniform on $[0, 1]$, and different projections $(\theta^{(t)})$. The SGD stopping criterion threshold (see~\ref{alg:SGD}) is set as negative, in order to always end at the maximum number of iterations, $10^6$. For the experiment with varying learning rates $\alpha$, we consider measures with $\npoints=20$ points in dimension $d=10$. For the experiment with varying dimensions $d$, we still take $\npoints=20$ and use the learning rate $\alpha=10$.

\begin{figure}
	\centering
	\begin{subfigure}[c]{0.45\textwidth}
		\centering
		\includegraphics[width=\linewidth]{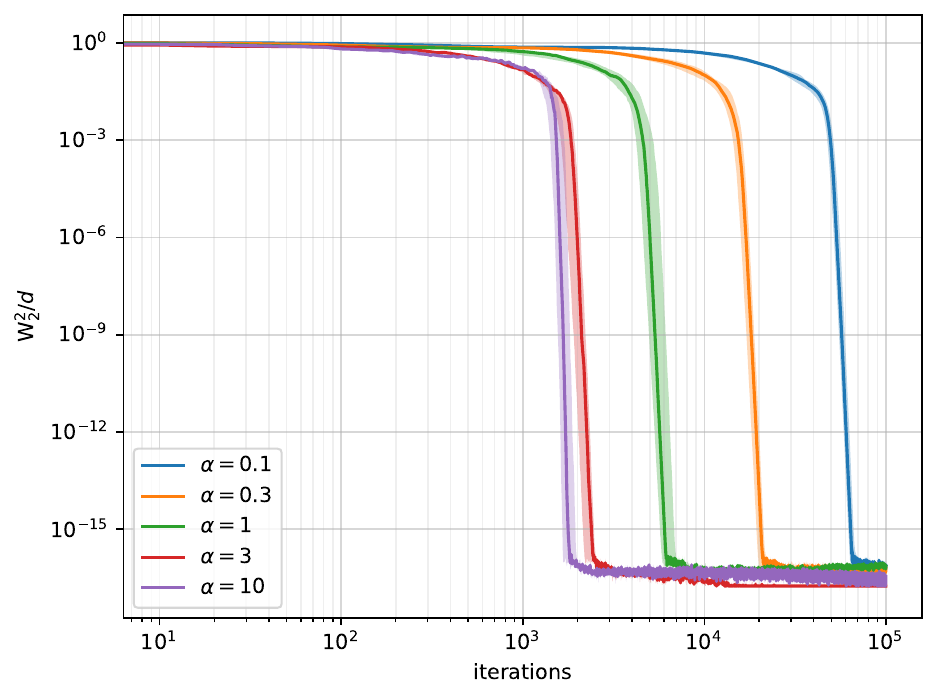}
	\end{subfigure}
	\begin{subfigure}[c]{0.45\textwidth}
		\centering
		\includegraphics[width=\linewidth]{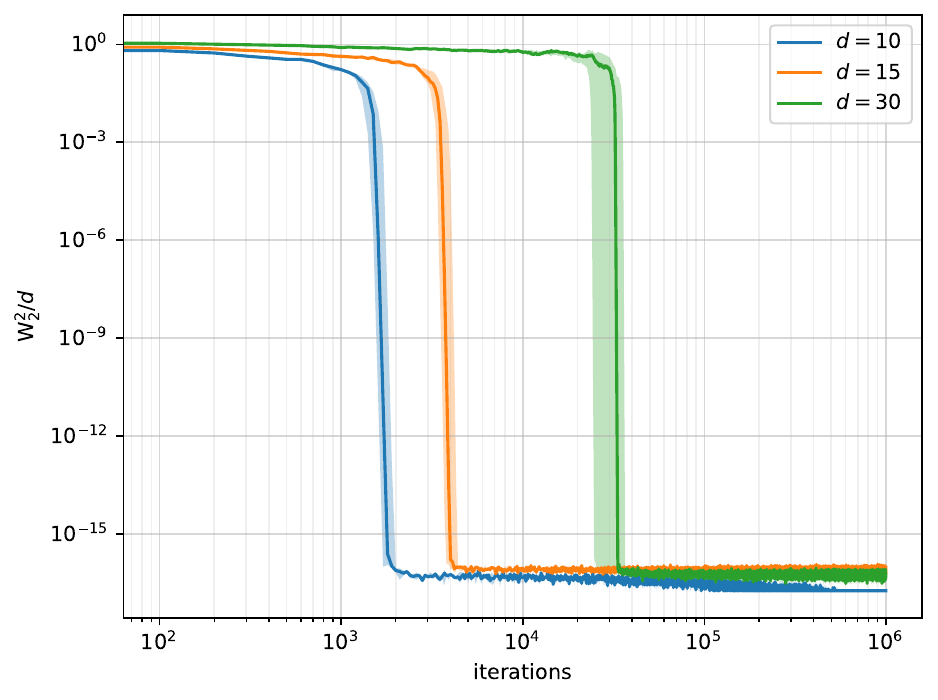}
	\end{subfigure}
	\caption{Wasserstein error $\frac{1}{d}\W_2^2(\gamma_{Y^{(t)}}, \gamma_Z)$ for SGD iterations $Y^{(t)}$ on $\SWY$, given a fixed measure $\gamma_Z,\; \R^{\npoints \times d}$ with $\npoints=20$ points. Left: different learning rates $\alpha$ for points in dimension $d=10$. Right: different dimensions with $\alpha=10$ (right).}
	\label{fig:SGD_its_E_alpha}
\end{figure}

In \ref{fig:SGD_its_E_alpha}, we observe that the SGD schemes converge towards the true measure $\gamma_Z$ up to numerical precision, which corresponds to a stronger convergence than the one predicted by~\ref{thm:SGD_interpolated_cv}. The number of iterations needed for convergence obviously depends on the learning rate $\alpha$, which notably can be chosen larger than $\npoints/2$, which is a case that does not fall under the conditions for \ref{thm:SGD_interpolated_cv}. However, in this particular experiment, the SGD schemes diverged as soon as $\alpha \geq 30$, which could suggest that limiting oneself to $\alpha \preceq \npoints$ is reasonable. The dimension $d$ increases significantly the number of iterations required for convergence, furthermore we observe a transition from high $\W_2^2$ error to low error, which is relatively sudden in logarithmic space. These first studies invites an in-depth analysis of the amount of iterations needed to reach convergence, which we propose in \ref{fig:SGD_its_for_cv}. The final $\frac{1}{d}\W_2^2$ error does not seem to depend significantly on the dimension $d$, which provides empirical grounds for the $1/d$ normalisation choice.

\paragraph{Noised SGD on $\SWY$.} \ref{fig:SGD_its_E_noise} shows the Wasserstein error $\frac{1}{d}\W_2^2(\gamma_{Y^{(t)}}, \gamma_Z)$ for the noised SGD iterations on $\SWY$. The numerical setup is the same as above, with the addition of the noise $a\alpha\varepsilon^{(t)}$ at each iteration, where $\varepsilon^{(t)}$ has independent standard Gaussian entries, $a$ is the noise level and $\alpha$ is the learning rate (set to $\alpha=10$). This noise is drawn differently for each SGD scheme. For the experiment with different dimensions, the noise level is taken as $a=10^{-4}$.

\begin{figure}
	\centering
	\begin{subfigure}[c]{0.45\textwidth}
		\centering
		\includegraphics[width=\linewidth]{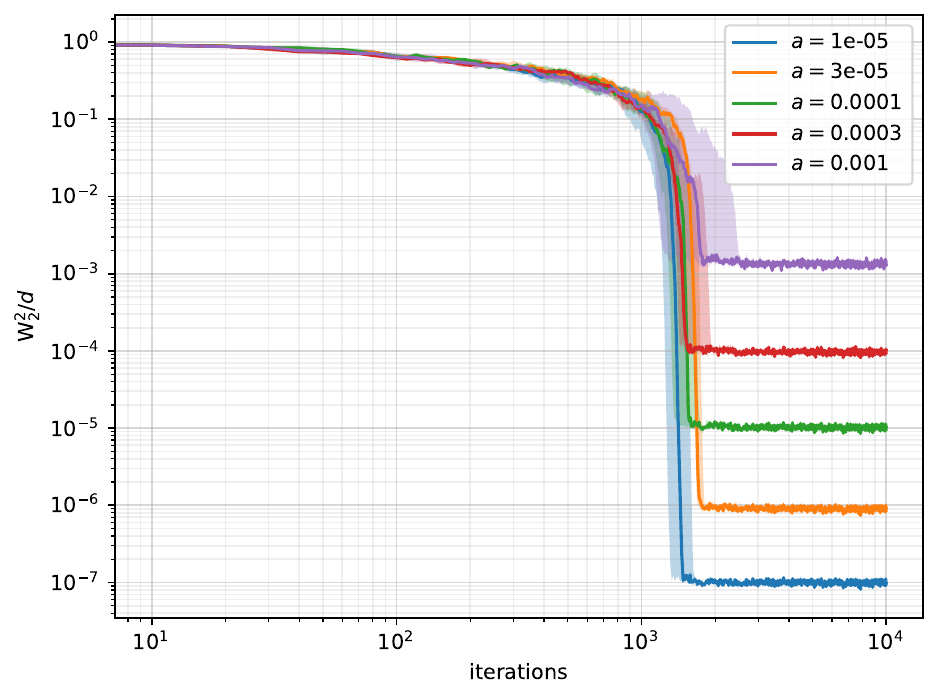}
	\end{subfigure}
	\begin{subfigure}[c]{0.45\textwidth}
		\centering
		\includegraphics[width=\linewidth]{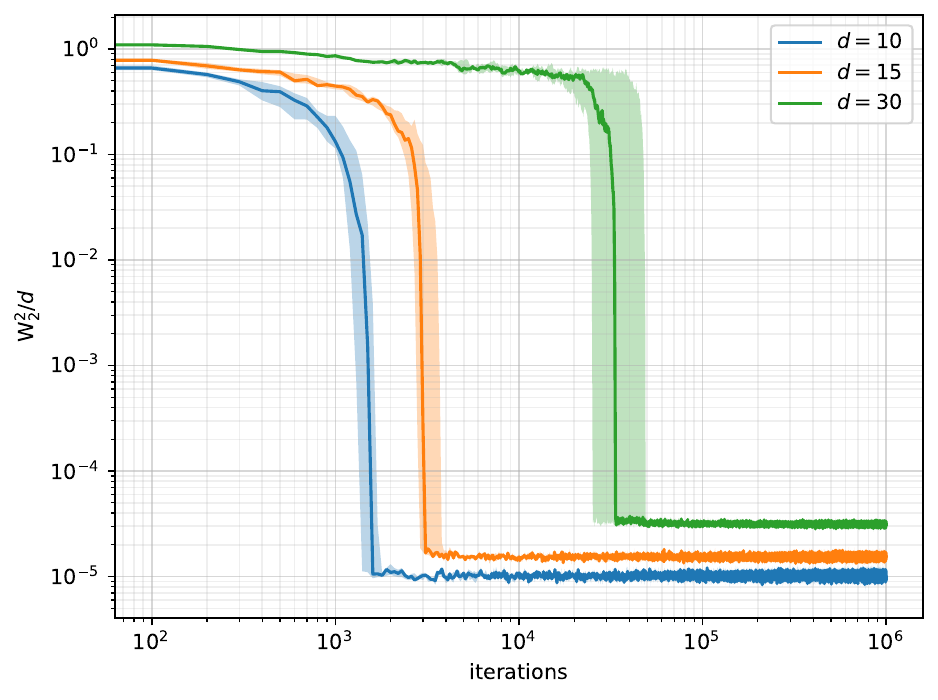}
	\end{subfigure}
	\caption{Wasserstein error $\frac{1}{d}\W_2^2(\gamma_{Y^{(t)}}, \gamma_Z)$ for noised SGD iterations $Y^{(t)}$ on $\SWY$, given a fixed measure $\gamma_Z,\; \R^{\npoints \times d}$ with $\npoints=20$ points. The noise is additive standard Gaussian, scaled by the learning rate $\alpha=10$ times the noise level $a$. Left: different noise levels $a$ for points in dimension $d=10$. Right: different dimensions with $a=10^{-4}$.}
	\label{fig:SGD_its_E_noise}
\end{figure}

The noised SGD scheme errors oscillate around a certain level which depends on the noise level, as the trajectories from \ref{fig:traj_sgd_E_noise} suggest: we observed Brownian-like motion around the target points. Note that the error begins falling drastically past the same iteration threshold, albeit with a higher variance across samples for higher noise levels. At a fixed noise level, the final $\frac{1}{d}\W_2^2$ still depends on the noise level, despite the $1/d$ normalisation. Empirically, the final $\W_2^2$ error seems to be smaller than the noise level $a$, which is reassuring since the noise is entry-wise of law $\mathcal{N}(0, a^2\alpha^2)$, where $\alpha$ is the learning rate. 

\paragraph{SGD on $\SWpY$.} \ref{fig:SGD_its_Ep} also illustrates the Wasserstein error along iterations but this time for  $\SWpY$.
The general SGD setup and initial measure $\gamma_Z$ remain unchanged compared to the schemes on $\SWY$ (with also a learning rate of $\alpha=10$ in particular). In order to handle the projections $(\theta^{(t)})$, for each sample we draw $p$ independent projections $(\theta_1, \cdots, \theta_p)$, then select the $(\theta^{(t)})$ by drawing uniformly amongst these $p$ projections. Given this sequence of projections $(\theta^{(t)})$, the SGD algorithm is then exactly the same as for $\SWY$.

\begin{figure}
	\centering
	\begin{subfigure}[c]{0.45\textwidth}
		\centering
		\includegraphics[width=\linewidth]{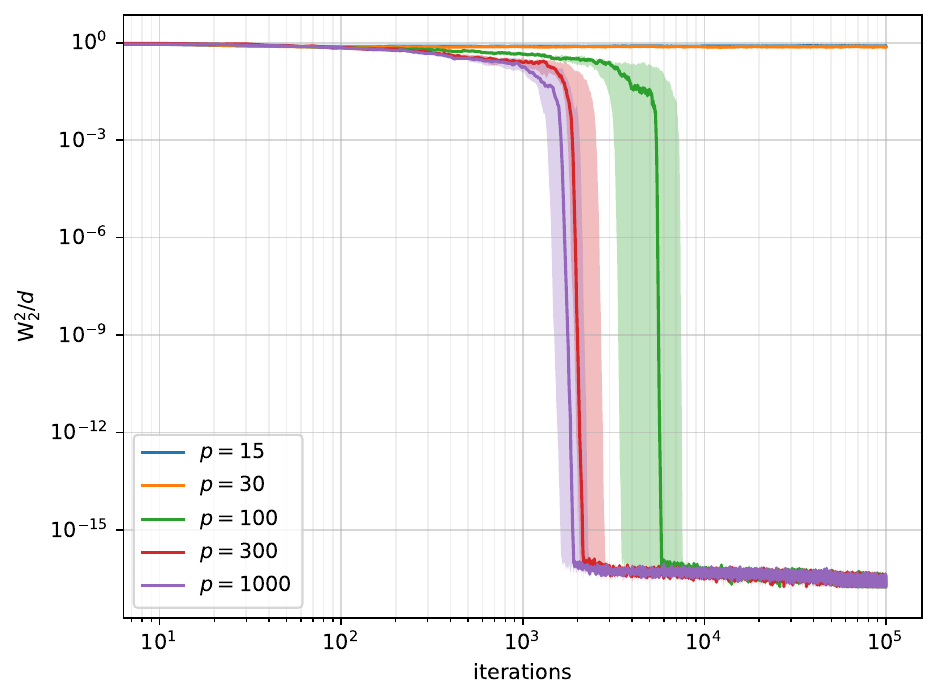}
	\end{subfigure}
	\begin{subfigure}[c]{0.45\textwidth}
		\centering
		\includegraphics[width=\linewidth]{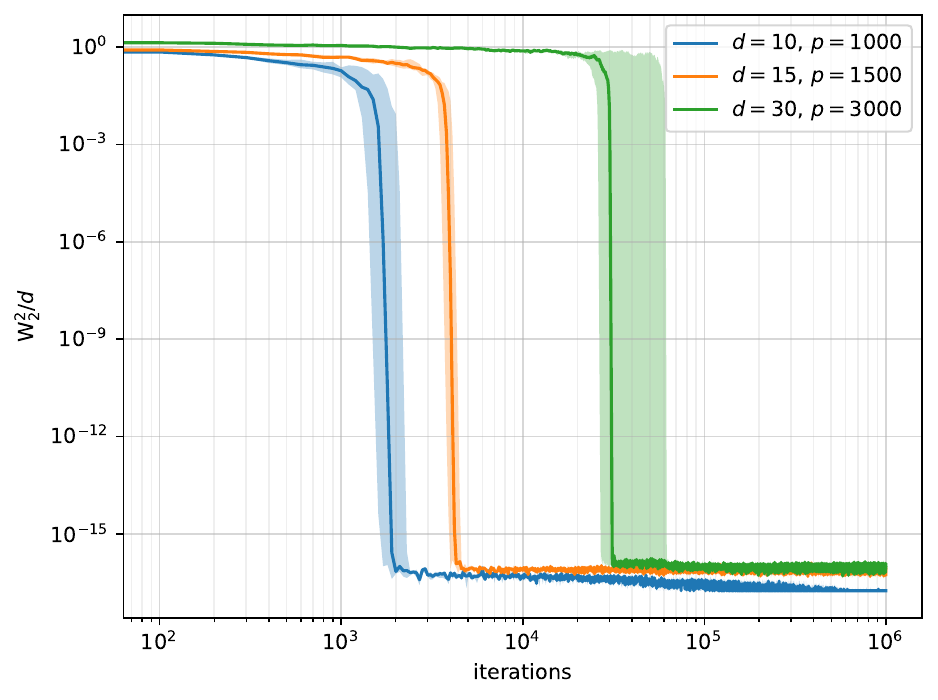}
	\end{subfigure}
	\caption{Wasserstein error $\frac{1}{d}\W_2^2(\gamma_{Y^{(t)}}, \gamma_Z)$ for SGD iterations $Y^{(t)}$ on $\SWpY$, given a fixed measure $\gamma_Z,\; \R^{\npoints \times d}$ with $\npoints=20$ points. The $p$ projections in $\SWpY$ are drawn randomly for each sample. Left: different noise levels $a$ for points in dimension $d=10$. Right: different dimensions with $\alpha=10$.}
	\label{fig:SGD_its_Ep}
\end{figure}

For SGD schemes on $\SWpY$ with small values of projections $p$, we do not have
convergence to $\gamma_{Y^{(t)}} = \gamma_Z$. Intuitively, this could be
understood as the approximation $\SWpY \approx \SWY$ being too rough, allowing
for an excessive amount of numerically attainable strict local optima. This is
illustrated in \ref{fig:SWpY_sym_p3} in a simple case: with $p=3$ in dimension
2, the landscape presents numerous strict local optima that lie within large
basins. However, it is notable that for $p$ large enough ($p \geq 10d = 100$),
we \textit{do} observe convergence to  $\gamma_{Y^{(t)}} = \gamma_Z$ up to
numerical precision. This convergence happens in fewer iterations as $p$
increases, and with a smaller variance with respect to the projection samples.
This suggests a stronger mode of convergence of $\SWpY$ towards $\SWY$, as
hinted at before in \ref{fig:SWpY_sym_p3} and \ref{fig:traj_sgd_Ep_multiple}. 

\paragraph{Quantifying the impact of the dimension.} For different values of the number of points $\npoints$ and the dimension $d$, we run 10 samples of SGD on $\SWY$ for an original measure $\gamma_Z$ drawn with standard Gaussian entries (re-drawn for each sample this time). The SGD schemes are done without additive noise, and with a learning rate of $\alpha=10$. In order to save computation time, the SGD stopping threshold is taken as $\beta = 10^{-5}$ (see \ref{alg:SGD}). For each sample, the initial position $Y^{(0)}$ is drawn with entries that are uniform on $[0, 1]$. Our goal is to estimate the number of iterations required for the convergence of the SGD schemes: to this end, we define convergence as the first step $t$ such that $\frac{1}{d}\W_2^2(\gamma_{Y^{(t)}}, \gamma_Z) < 10^{-5}$.

\begin{figure}
	\centering
	\includegraphics[width=0.7\linewidth]{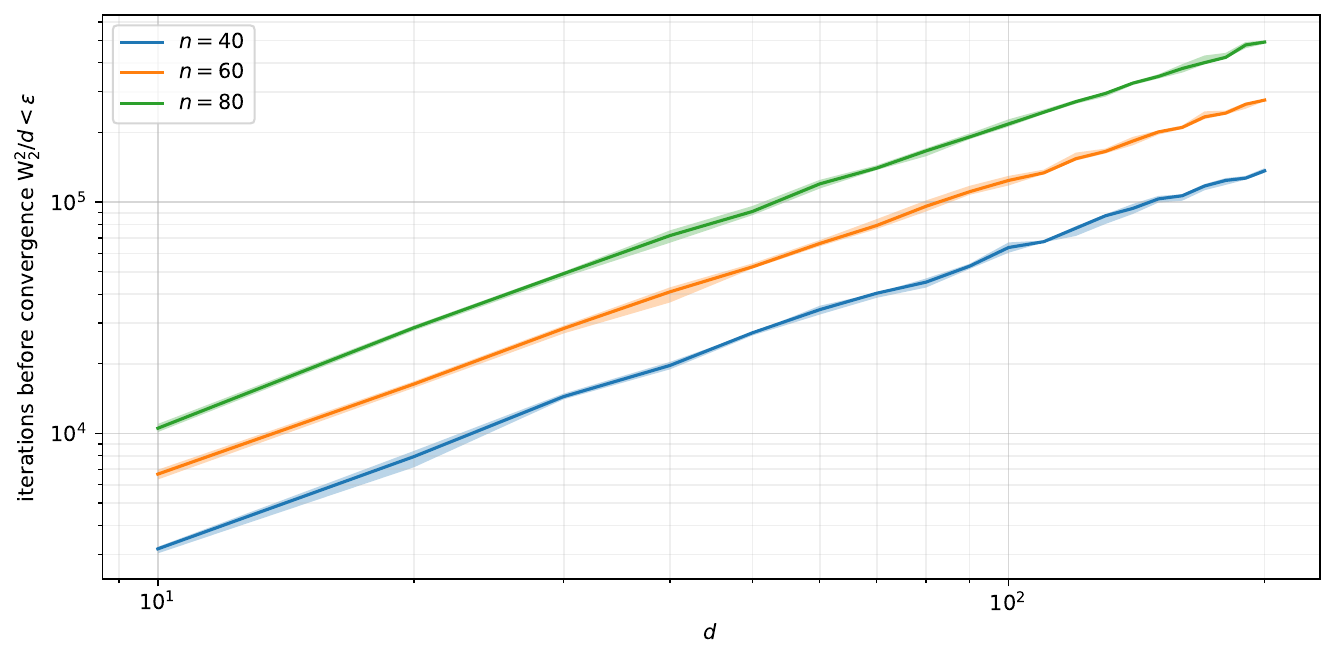}
	\caption{Amount of iterations required for convergence of a SGD scheme
	$Y^{(t)}$ of learning rate $\alpha=10$ on $\SWY$. Here convergence is
	defined as the first step $t$ such that $\frac{1}{d}\W_2^2(\gamma_{Y^{(t)}},
	\gamma_Z) < 10^{-5}$. For each set of parameters (number of points
	$\npoints$ and dimension $d$ ), 10 trials are done with $\gamma_Z,\; Z\in
	\R^{\npoints \times d}$ drawn at random (uniform $[0, 1]^{\npoints \times
	d}$).}
	\label{fig:SGD_its_for_cv}
\end{figure}

\ref{fig:SGD_its_for_cv} (cautiously) suggests that the number of iterations required for convergence its proportional to $d^{1.25}$ (where convergence means that $\frac{1}{d}\W_2^2$ falls below $\varepsilon$). Note that the exponent on $d$ does not seem to depend on $\npoints$. Obviously, the factor in front of $d^{1.25}$ depends on the number of points $\npoints$, the learning rate $\alpha$ and the convergence threshold $\varepsilon$. This superlinear rule remains fairly prohibitive for large Machine Learning models, which can typically have $d$ and $\npoints$ both in excess of $10^6$.

	\section{Conclusion and Outlook}

Throughout this paper, we have investigated the properties of the Sliced Wasserstein (SW) distance between discrete measures, namely the function $\SWY: Y \longmapsto\SW_2^2(\gamma_Y, \gamma_Z)$, where $Y$ and $Z$ are supports with $\npoints$ points in dimension $d$. Due to the intractability of the expectation in $\SWY$, we introduced its Monte-Carlo empirical counterpart $\SWpY$, computed as an average over $p$ directions. In \ref{sec:energies}, we showed and reminded regularity results on $\SWY$ and $\SWpY$: they are locally-Lipschitz and differentiable on certain open sets of full measure. Leveraging the fact that $\SWpY$ is piece-wise quadratic, we showed additional regularity results, and finally showed that the convergence of $\SWpY$ to $\SWY$ (as $p\longrightarrow+\infty$) is almost-surely uniform on any fixed compact. \ref{sec:crit} furthers the study of the optimisation landscapes at hand by presenting properties of the critical points of $\SWY$ and $\SWpY$ (points of differentiability will null gradient), and a convergence of such points of $\SWpY$ to those of $\SWY$ as $p\longrightarrow+\infty$ (in a certain sense). In \ref{sec:SGD}, we put these theoretical results in a more practical context by showing that one can apply the SGD convergence results of \cite{bianchi2022convergence} to our optimisation landscapes. Finally, we illustrate and study these convergence results in \ref{sec:xp} through numerical experiments.

Further work would be welcome on the cells of $\SWpY$ (see \ref{sec:cells}), in particular the law of their size given a fixed configuration $\config$ and their probability of being stable are still open problems, and would have strong consequences in practical applications such as the convergence of BCD (\ref{alg:BCD}). The main difficulty stems from the link between statistical properties of the cells to the so-called Gaussian Orthant Probabilities, which can be broadly defined as the probability of a non-standard Gaussian Vector to be in the positive quadrant $\R_+^d$. This probability is unfortunately not tractable in high dimensions, and its estimation is a field of research in itself \cite{gaussian_orthant}.

Another core limitation of our work concerns the practicality of our results on SGD convergence (\ref{sec:SGD}). Firstly, typical applications use more advanced optimisation methods, such as SGD with momentum or ADAM, which our theory does not encompass yet. Secondly, as mentioned in the introduction, practical applications actually minimise \textit{through} $\SWY$, which is to say a loss function $F: u \longmapsto \SW_2^2(T_u\#\mu, \nu)$ with respect to the parameters $u$ of a model $x \longmapsto T_u(x)$ of the input data $x\sim \mu$. Minimising $F$ through SGD (stochastically on the projections $\theta\sim\bbsigma$, the input data $x\sim\mu$ and the true data $y\sim\nu$) is beyond the scope of this paper, and we leave this generalisation for future work.

\subsection*{Acknowledgements}

\blue{We thank Anna Korba and Quentin Mérigot for their helpful discussions and
comments on the optimisation results. We would also like to thank Antoine
Chambaz for his valuable assistance with empirical processes.} \bluetwo{Finally,
we are grateful for the numerous suggestions of anonymous reviewers that
substantially improved this paper.}

\blue{This research was funded, in part, by the Agence nationale de la recherche (ANR), through the SOCOT project (ANR-23-CE40-0017), and the PEPR PDE-AI project (ANR-23-PEIA-0004).}

	\bibliography{ecl}

\providecommand{\bysame}{\leavevmode\hbox to3em{\hrulefill}\thinspace}
\providecommand{\MR}{\relax\ifhmode\unskip\space\fi MR }
\providecommand{\MRhref}[2]{%
  \href{http://www.ams.org/mathscinet-getitem?mr=#1}{#2}
}
\providecommand{\href}[2]{#2}
\begin{thebibliography}{10}

\bibitem{alghamdi2019patch}
Hana Alghamdi, Mairead Grogan, and Rozenn Dahyot, \emph{Patch-based colour
  transfer with optimal transport}, 2019 27th European Signal Processing
  Conference (EUSIPCO), IEEE, 2019, pp.~1--5.

\bibitem{pmlr-v70-arjovsky17a}
Martin Arjovsky, Soumith Chintala, and L{\'e}on Bottou, \emph{{W}asserstein
  generative adversarial networks}, Proceedings of the 34th International
  Conference on Machine Learning (Doina Precup and Yee~Whye Teh, eds.),
  Proceedings of Machine Learning Research, vol.~70, PMLR, 06--11 Aug 2017,
  pp.~214--223.

\bibitem{gaussian_orthant}
Dario Azzimonti and David Ginsbourger, \emph{Estimating orthant probabilities
  of high-dimensional {G}aussian vectors with an application to set
  estimation}, J. Comput. Graph. Statist. \textbf{27} (2018), no.~2, 255--267.
  \MR{3816262}

\bibitem{Backhoff-Veraguas:2022aa}
Julio Backhoff-Veraguas, Joaquin Fontbona, Gonzalo Rios, and Felipe Tobar,
  \emph{Stochastic gradient descent for barycenters in wasserstein space},
  (2022).

\bibitem{bayraktar_equivalence_W_SW}
Erhan Bayraktar and Gaoyue Guo, \emph{Strong equivalence between metrics of
  {W}asserstein type}, Electronic Communications in Probability \textbf{26}
  (2021).

\bibitem{bianchi2022convergence}
Pascal Bianchi, Walid Hachem, and Sholom Schechtman, \emph{Convergence of
  constant step stochastic gradient descent for non-smooth non-convex
  functions}, Set-Valued and Variational Analysis \textbf{30} (2022), no.~3,
  1117--1147.

\bibitem{bianchi2023stochastic}
\bysame, \emph{Stochastic subgradient descent escapes active strict saddles on
  weakly convex functions}, Mathematics of Operations Research (2023).

\bibitem{sketchedW}
Xin Bing, Florentina Bunea, and Jonathan Niles-Weed, \emph{The sketched
  {W}asserstein distance for mixture distributions}, arXiv preprint
  arXiv:2206.12768 (2022).

\bibitem{bolte2007clarke}
J{\'e}r{\^o}me Bolte, Aris Daniilidis, Adrian Lewis, and Masahiro Shiota,
  \emph{Clarke subgradients of stratifiable functions}, SIAM Journal on
  Optimization \textbf{18} (2007), no.~2, 556--572.

\bibitem{bolte2021conservative}
J{\'e}r{\^o}me Bolte and Edouard Pauwels, \emph{Conservative set valued fields,
  automatic differentiation, stochastic gradient methods and deep learning},
  Mathematical Programming \textbf{188} (2021), 19--51.

\bibitem{bonneel2015sliced}
Nicolas Bonneel, Julien Rabin, Gabriel Peyr{\'e}, and Hanspeter Pfister,
  \emph{Sliced and {R}adon {W}asserstein barycenters of measures}, Journal of
  Mathematical Imaging and Vision \textbf{51} (2015), no.~1, 22--45.

\bibitem{bonneel2015blind}
Nicolas Bonneel, James Tompkin, Kalyan Sunkavalli, Deqing Sun, Sylvain Paris,
  and Hanspeter Pfister, \emph{Blind video temporal consistency}, ACM
  Transactions on Graphics (TOG) \textbf{34} (2015), no.~6, 1--9.

\bibitem{bonnotte}
Nicolas Bonnotte, \emph{Unidimensional and evolution methods for optimal
  transportation.}, PhD Thesis, Paris 11 (2013).

\bibitem{clarke1990optimization}
Frank~H Clarke, \emph{Optimization and nonsmooth analysis}, SIAM, 1990.

\bibitem{cuturi2013sinkhorn}
Marco Cuturi, \emph{Sinkhorn distances: Lightspeed computation of optimal
  transport}, Advances in neural information processing systems \textbf{26}
  (2013).

\bibitem{davis2022proximal}
Damek Davis and Dmitriy Drusvyatskiy, \emph{Proximal methods avoid active
  strict saddles of weakly convex functions}, Foundations of Computational
  Mathematics \textbf{22} (2022), no.~2, 561--606.

\bibitem{davis2023active}
Damek Davis, Dmitriy Drusvyatskiy, and Liwei Jiang, \emph{Active manifolds,
  stratifications, and convergence to local minima in nonsmooth optimization},
  2023.

\bibitem{davis2020stochastic}
Damek Davis, Dmitriy Drusvyatskiy, Sham Kakade, and Jason~D Lee,
  \emph{Stochastic subgradient method converges on tame functions}, Foundations
  of computational mathematics \textbf{20} (2020), no.~1, 119--154.

\bibitem{deshpande_generative_sw}
Ishan Deshpande, Ziyu Zhang, and Alexander~G. Schwing, \emph{Generative
  modeling using the sliced {W}asserstein distance}, 2018 {IEEE} Conference on
  Computer Vision and Pattern Recognition, {CVPR} 2018, Salt Lake City, UT,
  USA, June 18-22, 2018, Computer Vision Foundation / {IEEE} Computer Society,
  2018, pp.~3483--3491.

\bibitem{dudley1969speed}
Richard~Mansfield Dudley, \emph{The speed of mean {G}livenko-{C}antelli
  convergence}, The Annals of Mathematical Statistics \textbf{40} (1969),
  no.~1, 40--50.

\bibitem{fatras2021minibatch}
Kilian Fatras, Younes Zine, Szymon Majewski, R{\'e}mi Flamary, R{\'e}mi
  Gribonval, and Nicolas Courty, \emph{Minibatch optimal transport distances;
  analysis and applications}, arXiv preprint arXiv:2101.01792 (2021).

\bibitem{flamary2021pot}
R{\'e}mi Flamary, Nicolas Courty, Alexandre Gramfort, Mokhtar~Z. Alaya,
  Aur{\'e}lie Boisbunon, Stanislas Chambon, Laetitia Chapel, Adrien Corenflos,
  Kilian Fatras, Nemo Fournier, L{\'e}o Gautheron, Nathalie~T.H. Gayraud,
  Hicham Janati, Alain Rakotomamonjy, Ievgen Redko, Antoine Rolet, Antony
  Schutz, Vivien Seguy, Danica~J. Sutherland, Romain Tavenard, Alexander Tong,
  and Titouan Vayer, \emph{{P}{O}{T}: Python optimal transport}, Journal of
  Machine Learning Research \textbf{22} (2021), no.~78, 1--8.

\bibitem{genevay2018learning}
Aude Genevay, Gabriel Peyr{\'e}, and Marco Cuturi, \emph{Learning generative
  models with {S}inkhorn divergences}, International Conference on Artificial
  Intelligence and Statistics, PMLR, 2018, pp.~1608--1617.

\bibitem{gulrajani2017improved}
Ishaan Gulrajani, Faruk Ahmed, Martin Arjovsky, Vincent Dumoulin, and Aaron~C
  Courville, \emph{Improved training of {W}asserstein {G}{A}{N}s}, Advances in
  neural information processing systems \textbf{30} (2017).

\bibitem{harris2020array}
Charles~R Harris, K~Jarrod Millman, St{\'e}fan~J Van Der~Walt, Ralf Gommers,
  Pauli Virtanen, David Cournapeau, Eric Wieser, Julian Taylor, Sebastian Berg,
  Nathaniel~J Smith, et~al., \emph{Array programming with numpy}, Nature
  \textbf{585} (2020), no.~7825, 357--362.

\bibitem{heitz2021sliced}
Eric Heitz, Kenneth Vanhoey, Thomas Chambon, and Laurent Belcour, \emph{A
  sliced {W}asserstein loss for neural texture synthesis}, Proceedings of the
  IEEE/CVF Conference on Computer Vision and Pattern Recognition, 2021,
  pp.~9412--9420.

\bibitem{jin2017escape}
Chi Jin, Rong Ge, Praneeth Netrapalli, Sham~M Kakade, and Michael~I Jordan,
  \emph{How to escape saddle points efficiently}, International conference on
  machine learning, PMLR, 2017, pp.~1724--1732.

\bibitem{jin2021nonconvex}
Chi Jin, Praneeth Netrapalli, Rong Ge, Sham~M Kakade, and Michael~I Jordan,
  \emph{On nonconvex optimization for machine learning: Gradients,
  stochasticity, and saddle points}, Journal of the ACM (JACM) \textbf{68}
  (2021), no.~2, 1--29.

\bibitem{karras2017progressive}
Tero Karras, Timo Aila, Samuli Laine, and Jaakko Lehtinen, \emph{Progressive
  growing of {G}{A}{N}s for improved quality, stability, and variation}, arXiv
  preprint arXiv:1710.10196 (2017).

\bibitem{kolouri2018slicedAE}
Soheil Kolouri, Phillip~E. Pope, Charles~E. Martin, and Gustavo~K. Rohde,
  \emph{Sliced {W}asserstein auto-encoders}, International Conference on
  Learning Representations, 2019.

\bibitem{levy2012elements}
S.~Levy, F.~Hirsch, and G.~Lacombe, \emph{Elements of functional analysis},
  Graduate Texts in Mathematics, Springer New York, 2012.

\bibitem{Li:2023aa}
Shiying Li and Caroline Moosmueller, \emph{Measure transfer via stochastic
  slicing and matching},  (2023).

\bibitem{liutkus19a_SWflow_generation}
Antoine Liutkus, Umut Simsekli, Szymon Majewski, Alain Durmus, and
  Fabian-Robert St{\"o}ter, \emph{Sliced-{W}asserstein flows: Nonparametric
  generative modeling via optimal transport and diffusions}, Proceedings of the
  36th International Conference on Machine Learning (Kamalika Chaudhuri and
  Ruslan Salakhutdinov, eds.), Proceedings of Machine Learning Research,
  vol.~97, PMLR, 09--15 Jun 2019, pp.~4104--4113.

\bibitem{majewski2018analysis}
Szymon Majewski, B{\l}a{\.z}ej Miasojedow, and Eric Moulines, \emph{Analysis of
  nonsmooth stochastic approximation: the differential inclusion approach},
  arXiv preprint arXiv:1805.01916 (2018).

\bibitem{merigot2021bounds_approx_W2}
Quentin M{\'e}rigot, Filippo Santambrogio, and Cl{\'e}ment Sarrazin,
  \emph{Non-asymptotic convergence bounds for {W}asserstein approximation using
  point clouds}, Advances in Neural Information Processing Systems \textbf{34}
  (2021), 12810--12821.

\bibitem{nadjahi_bayesian}
Kimia Nadjahi, Valentin~De Bortoli, Alain Durmus, Roland Badeau, and Umut
  Şimşekli, \emph{Approximate bayesian computation with the
  sliced-{W}asserstein distance}, ICASSP 2020 - 2020 IEEE International
  Conference on Acoustics, Speech and Signal Processing (ICASSP), 2020,
  pp.~5470--5474.

\bibitem{nadjahi_statistical_properties_sliced}
Kimia Nadjahi, Alain Durmus, L\'{e}na\"{\i}c Chizat, Soheil Kolouri, Shahin
  Shahrampour, and Umut Simsekli, \emph{Statistical and topological properties
  of sliced probability divergences}, Advances in Neural Information Processing
  Systems (H.~Larochelle, M.~Ranzato, R.~Hadsell, M.F. Balcan, and H.~Lin,
  eds.), vol.~33, Curran Associates, Inc., 2020, pp.~20802--20812.

\bibitem{nadjahi2019_guarantees_sw}
Kimia Nadjahi, Alain Durmus, Umut Simsekli, and Roland Badeau, \emph{Asymptotic
  guarantees for learning generative models with the sliced-{W}asserstein
  distance}, Advances in Neural Information Processing Systems (H.~Wallach,
  H.~Larochelle, A.~Beygelzimer, F.~d\textquotesingle Alch\'{e}-Buc, E.~Fox,
  and R.~Garnett, eds.), vol.~32, Curran Associates, Inc., 2019.

\bibitem{pytorch}
Adam Paszke, Sam Gross, Francisco Massa, Adam Lerer, James Bradbury, Gregory
  Chanan, Trevor Killeen, Zeming Lin, Natalia Gimelshein, Luca Antiga, Alban
  Desmaison, Andreas Kopf, Edward Yang, Zachary DeVito, Martin Raison, Alykhan
  Tejani, Sasank Chilamkurthy, Benoit Steiner, Lu~Fang, Junjie Bai, and Soumith
  Chintala, \emph{Pytorch: An imperative style, high-performance deep learning
  library}, Advances in Neural Information Processing Systems (H.~Wallach,
  H.~Larochelle, A.~Beygelzimer, F.~d\textquotesingle Alch\'{e}-Buc, E.~Fox,
  and R.~Garnett, eds.), vol.~32, Curran Associates, Inc., 2019.

\bibitem{computational_ot}
G.~Peyr{\'e} and M.~Cuturi, \emph{Computational optimal transport}, Foundations
  and Trends in Machine Learning \textbf{51} (2019), no.~1, 1--44.

\bibitem{Rabin_texture_mixing_sw}
Julien Rabin, Gabriel Peyr{\'e}, Julie Delon, and Marc Bernot,
  \emph{Wasserstein barycenter and its application to texture mixing}, Scale
  Space and Variational Methods in Computer Vision: Third International
  Conference, SSVM 2011, Ein-Gedi, Israel, May 29--June 2, 2011, Revised
  Selected Papers 3, Springer, 2012, pp.~435--446.

\bibitem{santambrogio2015optimal}
Filippo Santambrogio, \emph{Optimal transport for applied mathematicians},
  Birk{\"a}user, NY \textbf{55} (2015), no.~58-63, 94.

\bibitem{tanguy2023reconstructing}
Eloi Tanguy, R{\'e}mi Flamary, and Julie Delon, \emph{Reconstructing discrete
  measures from projections. consequences on the empirical sliced {W}asserstein
  distance}, arXiv preprint arXiv:2304.12029 (2023).

\bibitem{Tartavel2016}
Guillaume Tartavel, Gabriel Peyr\'{e}, and Yann Gousseau, \emph{Wasserstein
  loss for image synthesis and restoration}, SIAM Journal on Imaging Sciences
  \textbf{9} (2016), no.~4, 1726--1755.

\bibitem{tropp2012concentration}
Joel~A Tropp, \emph{User-friendly tail bounds for sums of random matrices},
  Foundations of computational mathematics \textbf{12} (2012), no.~4, 389--434.

\bibitem{van2000asymptotic}
Aad~W Van~der Vaart, \emph{Asymptotic statistics}, vol.~3, Cambridge university
  press, 2000.

\bibitem{vial1983strong}
Jean-Philippe Vial, \emph{Strong and weak convexity of sets and functions},
  Mathematics of Operations Research \textbf{8} (1983), no.~2, 231--259.

\bibitem{villani}
Cédric Villani, \emph{Optimal transport : old and new / cédric villani},
  Grundlehren der mathematischen Wissenschaften, Springer, Berlin, 2009 (eng).

\bibitem{Wakabayashi_semialgebraic}
Seiichiro Wakabayashi, \emph{Remarks on semi-algebraic functions}, January
  2008, Online Notes.

\bibitem{Wu2019_SWAE}
J.~Wu, Z.~Huang, D.~Acharya, W.~Li, J.~Thoma, D.~Paudel, and L.~Van Gool,
  \emph{Sliced {W}asserstein generative models}, 2019 IEEE/CVF Conference on
  Computer Vision and Pattern Recognition (CVPR) (Los Alamitos, CA, USA), IEEE
  Computer Society, jun 2019, pp.~3708--3717.

\bibitem{xi2022distributional}
Jiaqi Xi and Jonathan Niles-Weed, \emph{Distributional convergence of the
  sliced {W}asserstein process}, Advances in Neural Information Processing
  Systems \textbf{35} (2022), 13961--13973.

\bibitem{xu2022central}
Xianliang Xu and Zhongyi Huang, \emph{Central limit theorem for the sliced
  1-{W}asserstein distance and the max-sliced 1-{W}asserstein distance}, arXiv
  preprint arXiv:2205.14624 (2022).

\end{thebibliography}
	\bibliographystyle{amsplain}
	
	\appendix
	\section{}

\blue{\subsection{Proof of the Central Limit Theorem for Discrete SW}\label{sec:proof_donsker}

Let $c_{\mathcal{K}} := \sup_{Y\in \mathcal{}K}\|Y\|_{\infty, 2}$. Consider the class of functions $\mathcal{F} := \lbrace f_Y |\ Y \in \mathcal{K} \rbrace$, where we define $f_Y:= \theta \longmapsto\W_2^2(P_\theta\#\gamma_Y, P_\theta\#\gamma_Z)$ to fit the notation style of empirical processes. By \ref{prop:w_unif_locLip}, we have for any $\theta \in \SS^{d-1}$ and $Y, Y' \in \mathcal{K}$:
$$|f_Y(\theta) - f_{Y'}(\theta)| \leq  2\npoints (2c_{\mathcal{K}} + c_{\mathcal{K}} + \|Z\|_{\infty, 2})\|Y-Y'\|_{\infty, 2},$$
where we chose the neighbourhood $B_{\|\cdot\|_{\infty, 2}}(0, 2c_{\mathcal{K}}) \supset \mathcal{K}$. In particular, the Lipschitz constant $\kappa := 2\npoints (3c_{\mathcal{K}} + \|Z\|_{\infty, 2})$ is $\bbsigma$-integrable (in this case, it is constant in $\theta$), which allows us to apply Example 19.7 from \cite{van2000asymptotic}, which implies that the family $\mathcal{F}$ is $\bbsigma$-Donsker.

To recall the definition of the property that $\mathcal{F}$ is $\bbsigma$-Donsker, we recall some standard concepts and notations from empirical processes (\cite{van2000asymptotic} Section 19.2), within our specific case of application. For $f$ a function from $\SS^{d-1}$ to $\R$ that is $\bbsigma$-square-integrable, we introduce the real random variable
$$\bbsigma_p f := \cfrac{1}{p}\Sum{i=1}{p}f(\theta_i),\; (\theta_1, \cdots, \theta_p) \sim \bbsigma^{\otimes p},$$
which is meant as a Monte-Carlo approximation of the true expectation $\bbsigma f := \int_{\SS^{d-1}}f(\theta)\dd\bbsigma(\theta)$. We shall study the distribution of the scaled approximation error, defined as the real random variable $\G_pf := \sqrt{p}(\bbsigma_p f - \bbsigma f)$. The central limit theorem shows that $\G_pf$ converges in law to a centred Gaussian (in our case the functions $f\in \mathcal{F}$ are trivially $\bbsigma-L^2$), and our goal is instead to show the uniform convergence of the \textit{random process}
$$\G_p := \lbrace \sqrt{p}(\bbsigma_p f - \bbsigma f),\; f\in \mathcal{F} \rbrace.$$
To study the convergence of the sequence of processes $(\G_p)_{p \in \N^*}$, we introduce the space $\ell^{\infty}(\mathcal{F})$ of the bounded functions $z: \mathcal{F} \longrightarrow \R$, equipped with the norm $\|z\|_{\infty} = \sup_{f \in \mathcal{F}}(|z(f)|)$. The class of functions $\mathcal{F}$ is said to be $\bbsigma$-Donsker if each process $\G_p$ is in $\ell^{\infty}(\mathcal{F})$ (which is to say that it has bounded trajectories), and if the sequence $(\G_p)_{p\in \N^*}$ converges in law (in the sense of the topology of $\ell^{\infty}(\mathcal{F})$) towards a tight\footnote{Since this technical notion is not essential for our result, we refer to \cite{van2000asymptotic} page 260 for a complete presentation.} process $\G$. Due to the usual multivariate Central Limit Theorem, this process $\G$ is necessarily the $\bbsigma$-\textit{Brownian Bridge}, which is to say the centred Gaussian process indexed on $\mathcal{F}$ such that $\mathrm{Cov}[\G f, \G f'] = \bbsigma (ff') - (\bbsigma f)(\bbsigma f')$.

We have now shown that our error process $\G_p$ converges in law (in $\ell^\infty(\mathcal{F})$) towards the Gaussian process $\G$, i.e.
\begin{equation}\label{eqn:donsker_F}
	\lbrace \sqrt{p}(\bbsigma_p f - \bbsigma f),\;  f\in \mathcal{F} \rbrace \xrightarrow[p \longrightarrow +\infty]{\mathcal{L},\; \ell^{\infty}(\mathcal{F})} \G .
\end{equation}
Seeing our energy $\SWpY: Y \longmapsto \bbsigma_p f_Y$ as a process on $\mathcal{K}$, we can identify the indexes $\lbrace f_Y,\; Y \in \mathcal{K}\rbrace$ with $\mathcal{K}$ itself, yielding a convergence of processes on $\mathcal{K}$. To be precise, we define the space $\ell^\infty(\mathcal{K})$ as the space of bounded functions $z: \mathcal{K} \longrightarrow \R$, equipped with the infinite norm. Consider the map
$$\phi := \app{\ell^{\infty}(\mathcal{F})}{\ell^{\infty}(\mathcal{K})}{z}{\app{\mathcal{K}}{\R}{Y}{z(f_Y)}},$$
since $\mathcal{F} := \lbrace f_Y |\ Y \in \mathcal{K} \rbrace$, it is well defined, and it is continuous, since it is linear and verifies for any $z \in \ell^{\infty}(\mathcal{F})$, $\|\phi(z)\|_{\ell^{\infty}(K)} = \sup_{Y\in \mathcal{K}}|z(f_Y)| = \sup_{f\in \mathcal{F}}|z(f)|=\|z\|_{\ell^{\infty}(\mathcal{F})}$. By the continuous mapping theorem (see \cite{van2000asymptotic} Theorem 18.11 for its use on stochastic processes), we can apply $\phi$ to the convergence in law in \ref{eqn:donsker_F}, yielding
\begin{equation}\label{eqn:donsker_RD}
	\sqrt{p}(\SWpY - \SWY) \xrightarrow[p \longrightarrow +\infty]{\mathcal{L},\; \ell^{\infty}(\mathcal{K})} \phi(\G),
\end{equation}
where the process $\phi(\G)$ is the centred Gaussian process on $\mathcal{K}$ with the covariance structure $\mathrm{Cov}\phi(\G)[Y, Y'] = \bbsigma(f_Yf_{Y'}) - (\bbsigma(f_Y))(\bbsigma(f_{Y'}))$. Note that $\SWY(Y) = \bbsigma f_Y$ with our notations.

With the continuous mapping theorem (again \cite{van2000asymptotic} Theorem 18.11), we can apply the continuous map $\|\cdot\|_{\ell^{\infty}(\mathcal{K})}: \ell^{\infty}(\mathcal{K}) \longrightarrow \R$, yielding a uniform convergence result
\begin{equation}\label{eqn:donsker_uniform}
	\sqrt{p}\|\SWpY - \SWY\|_{\ell^{\infty}(\mathcal{K})} \xrightarrow[p\longrightarrow+\infty]{\mathcal{L}} \|\phi(\G)\|_{\ell^{\infty}(\mathcal{K})}.
\end{equation}}

\subsection{Computing \texorpdfstring{$\SWY$}{E}, \texorpdfstring{$\W_2^2$}{W} and \texorpdfstring{$\SWpY$}{Ep} in a simple case}

\paragraph{Computing \texorpdfstring{$\SWY$}{E}}\label{sec:L2_E} We work in polar coordinates, writing 
$$\theta = \left(\begin{array}{c}
	\cos \phi \\
	\sin \phi
\end{array}\right),\; \text{and\ } y = \left(\begin{array}{c}
	u \\
	v
\end{array}\right) = r\left(\begin{array}{c}
	\cos \psi \\
	\sin \psi
\end{array}\right).$$

By symmetry of the problem, we can assume $\psi \in [0, \pi/2]$ (i.e. the top-right quadrant $u \geq 0, v \geq 0$). Now let $\psi \in [0, 2\pi[$, let us compute $\W_2^2(P_\theta\#\gamma_Y, P_\theta\#\gamma_Z)$. Since we project in 1D, computing this slice amounts to sorting $(\theta^T y_1, \theta^T y_2)$ and $(\theta^T z_1, \theta^T z_2)$. Let $\sort{Z}{\theta} \in \mathfrak{S}_2$ such that $\theta^T y_{\sort{Y}{\theta}(1)} \leq \theta^T y_{\sort{Y}{\theta}(2)}$ and similarly $\theta^T z_{\sort{Z}{\theta}(1)} \leq \theta^T z_{\sort{Z}{\theta}(2)}$.
We always have 
$$\W_2^2(P_\theta\#\gamma_Y, P_\theta\#\gamma_Z) = \cfrac{1}{2}\left(\left(\theta^T (y_{\sort{Y}{\theta}(1)} - z_{\sort{Z}{\theta}(1)})\right)^2 + \left(\theta^T (y_{\sort{Y}{\theta}(2)} - z_{\sort{Z}{\theta}(2)})\right)^2\right).$$
We split the integral depending on the values of $\sort{Y}{\theta}$ and $\sort{Z}{\theta}$, which vary depending on the angle of the projection $\phi$. We begin with $\sort{Y}{\theta}$:
\begin{equation}\label{eqn:ex_sort_Y}
	\theta^T y_1 \geq \theta^T y_2 \Longleftrightarrow \cos \phi \cos \psi + \sin\phi \sin \psi \geq 0 \Longleftrightarrow \phi \in [\psi - \pi/2, \psi + \psi/2] + 2\pi \Z.
\end{equation}
The equation for $\sort{Z}{\theta}$ is much simpler:
\begin{equation}\label{eqn:ex_sort_Z}
	\theta^T z_1 \geq \theta^T z_2 \Longleftrightarrow -\sin \phi \geq 0 \Longleftrightarrow \phi \in [\pi, 2\pi] + 2\pi \Z.
\end{equation}
We divide a period of $2\pi$ in four quadrants corresponding to the four possibilities for $(\sort{Y}{\theta}, \sort{Z}{\theta})$. Since we assume $\psi \in [0, \pi/2]$, we can write this simply as:
\begin{align*}
	\SWY(Y) &= \cfrac{1}{4\pi}\Int{-\pi}{\psi-\pi/2}\left(\left(\theta^T (y_1 - z_2)\right)^2 + \left(\theta^T (y_2 - z_1)\right)^2\right)\dd\phi \\
	&+ \cfrac{1}{4\pi}\Int{\psi-\pi/2}{0}\left(\left(\theta^T (y_2 - z_2)\right)^2 + \left(\theta^T (y_1 - z_1)\right)^2\right)\dd\phi \\
	&+ \cfrac{1}{4\pi}\Int{0}{\psi+\pi/2}\left(\left(\theta^T (y_2 - z_1)\right)^2 + \left(\theta^T (y_1 - z_2)\right)^2\right)\dd\phi \\
	&+ \cfrac{1}{4\pi}\Int{\psi+\pi/2}{\pi}\left(\left(\theta^T (y_1 - z_1)\right)^2 + \left(\theta^T (y_2 - z_2)\right)^2\right)\dd\phi.
\end{align*}
Elementary trigonometric integration yields
\begin{equation}
	\SWY(Y) = \frac{r^2}{2} + \frac{1}{2} - \frac{2}{\pi}\left(r\cos\psi + r\psi \sin\psi\right) = \frac{u^2+v^2}{2} + \frac{1}{2} - \frac{2}{\pi}\left(u + v\Arctan\frac{v}{u}\right),
\end{equation}
which holds for $\psi \in [0, \pi/2]$. By symmetry, we obtain the following expression for any $(u, v) \in \R^2$ (recall that we stack the vectors in $Y$ line by line):
\begin{equation}
	\SWY\left(\begin{array}{cc}
		u & v \\
		-u & -v
	\end{array}\right) = \frac{u^2+v^2}{2} + \frac{1}{2} - \frac{2}{\pi}\left(|u| + |v|\Arctan\left|\frac{v}{u}\right|\right).
\end{equation}
In the general case, dimension $d$ would require the use of $d$-dimensional spherical coordinates, making the equations~\ref{eqn:ex_sort_Y} and~\ref{eqn:ex_sort_Z} intractable. Furthermore, generalising to $\npoints$ points would separate the integral into $(\npoints!)^2$ parts, losing all hopes of tractability and legibility.

\paragraph{Computing \texorpdfstring{$\W_2^2$}{W}}\label{sec:L2_W} In the case $\npoints=2$, the Kantorovich LP formulation of the Wasserstein distance can be written as:
$$\underset{a\in [0, 1]}{\min}\ \Sum{k,l \in \llbracket 1, 2 \rrbracket}{}\pi_{k,l}(a)\|y_k-z_l\|_2^2,\quad \text{with}\ \pi(a) := \cfrac{1}{2}\left(\begin{array}{cc}
	1-a & a \\
	a & 1-a
\end{array}\right).$$
Substituting $y_1 = \left(\begin{array}{c}
	u \\
	v
\end{array}\right),\; y_2 = \left(\begin{array}{c}
-u \\
-v
\end{array}\right),\; z_1 = \left(\begin{array}{c}
0 \\
-1
\end{array}\right),\; z_2 = \left(\begin{array}{c}
0 \\
1
\end{array}\right)$ yields:
$$\W_2^2(\gamma_Y, \gamma_Z) = \underset{a \in [0,1]}{\min}\ \left(u^2 + (v+1)^2 - 4av\right) = u^2 + (|v|-1)^2.$$

\paragraph{Computing \texorpdfstring{$\SWpY$}{Ep}}\label{sec:L2_Ep} For simplicity, in the following we will only consider $\theta \in \SS^{d-1}$ such that the $\theta^T y_k$ are distinct, and such that the $\theta^T z_k$ are also distinct. We will express the cases for the values of the sortings $\sort{Y}{\theta}$ and $\sort{Z}{\theta}$ in a different (yet equivalent) manner.

We have $\sort{Y}{\theta} = I$ if $\theta^T y_1 < \theta^T y_2$ and $\sort{Y}{\theta} = (2, 1)$ otherwise. Then $\sort{Z}{\theta} \circ (\sort{Y}{\theta})^{-1} = I$ if $\sort{Y}{\theta} = \sort{Z}{\theta}$, and $\sort{Z}{\theta_i} \circ (\sort{Y}{\theta_i})^{-1} = (2, 1)$ otherwise. The system 
$$\sort{Z}{\theta} \circ (\sort{Y}{\theta})^{-1} = I \Longleftrightarrow \left\lbrace\begin{array}{c}
	\theta^T y_1 < \theta^T y_2\ \mathrm{and}\ \theta^T z_1 < \theta^T z_2 \\
	\mathrm{or} \\
	\theta^T y_2 < \theta^T y_1\ \mathrm{and}\ \theta^T z_2 < \theta^T z_1 \\
\end{array} \right.$$ 
can be simplified, yielding:
\begin{equation}\label{eqn:2sort_linear}
	\sort{Z}{\theta} \circ (\sort{Y}{\theta})^{-1} = I \Longleftrightarrow \left(\theta \theta^T (z_2 - z_1)\right)^T (y_2 - y_1) > 0.
\end{equation}
\ref{eqn:2sort_linear} is a linear equation in $Y$. Additionally, \ref{eqn:2sort_linear} only depends on $y_2 - y_1 = -2y$, which makes our symmetrical simplification inconsequential. Plugging in the specific point values yields a more explicit definition of the cells. We write the condition on $y\in \R^2,$ since $Y = (y, -y)^T$.
\begin{equation}\label{eqn:2cell_polytope}
	\mathcal{C}_\config = \left\lbrace y\in \R^2\ |\ \forall i \in \llbracket 1, p \rrbracket,\; -\sign\left[\theta_i^T \left(\begin{array}{c}
		0 \\
		1
	\end{array}\right)\ \theta_i^T y\right] = +1\  \mathrm{if}\ \config_i = I,\ \mathrm{else} -1 \right\rbrace.
\end{equation}
Equation~\ref{eqn:2cell_polytope} describes $\mathcal{C}_\config$ as an intersection of $p$ half-planes of $\R^2$, thus it is a polytope. Note that we use strict inequalities, which lifts configuration ambiguities, and implies that the $(\mathcal{C}_\config)_{\config \in \mathfrak{S}_2^p}$ are disjoint, and that the union of their closure is $\R^{2}$.

Straightforward computation yields 
$$\underset{X \in \R^{\npoints\times d}}{\argmin}\ q_\config(X) = (A^{-1}(B^{\config}_{1, 1}z_1 + B^{\config}_{1, 2}z_2), A^{-1}(B^{\config}_{2, 1}z_1 + B^{\config}_{2, 2}z_2)),$$

where $A := \cfrac{1}{p}\Sum{i=1}{p}\theta_i\theta_i^T$ and $B^{\config}_{k, l} := \cfrac{1}{p}\Sum{\substack{i=1 \\ \config_i(k)=l}}{p}\theta_i\theta_i^T$.

Note that our $\npoints=2$ setting, we have the simplifications $B^{\config}_{1, 2} = A - B^{\config}_{1, 1},\; B^{\config}_{2, 1} = B^{\config}_{1, 2}$ and $B^{\config}_{1, 1} = B^{\config}_{2, 2}$.
Furthermore, $B^{\config}_{k, l}$ is (up to a factor), a Monte-Carlo estimation of $S_{k,l}^{Y,Z}$ (see \ref{cor:crit_points_S}).

\subsection{Discrete Wasserstein stability}\label{sec:WC_stability}

Consider the following generic discrete Kantorovich problem, given weights \blue{$\alpha \in \Sigma_n$ and $\beta \in \Sigma_m$} and a generic cost matrix $C \in \R_+^{n\times m}$:
\begin{equation}\label{eqn:WC}
	\W(\alpha, \beta; C) := \underset{\pi \in \Pi(\alpha, \beta)}{\inf}\ \pi \cdot C,
\end{equation}
\blue{where $\Pi(\alpha, \beta, C)$ is the set of $n\times m$ matrices $\pi$ with non-negative entries such that $\pi \mathbbold{1} = \alpha$ and $\pi^T \mathbbold{1} = \beta$.}

\begin{lemma}[Stability of the Wasserstein cost]\ Let $\alpha, \oll{\alpha}, \in
	\Sigma_\npoints$, $\beta, \oll{\beta} \in \Sigma_m$  and $C, \oll{C} \in
	\R_+^{\npoints\times m}$, \blue{such that the weights verify $\alpha,
	\oll{\alpha}, \beta, \oll{\beta} >
	0$ entry-wise.} Then:
	\blue{	\begin{equation}\label{eqn:wass_stability}
			\left|\W(\alpha, \beta; C) - \W(\oll{\alpha}, \oll{\beta}; \oll{C})\right| \leq \|C-\oll{C}\|_{\infty} + \|C\|_{\infty}( \|\alpha - \oll{\alpha}\|_1 + \|\beta - \oll{\beta}\|_1).
		\end{equation}}
	\blue{\begin{equation}\label{eqn:wass_stabilityL2}
		\left|\W(\alpha, \beta; C) - \W(\oll{\alpha}, \oll{\beta}; \oll{C})\right| \leq \|C-\oll{C}\|_{F} + \|C\|_F\left(\|\alpha - \oll{\alpha}\|_2+\|\beta - \oll{\beta}\|_2\right),
		\end{equation}}
	\blue{where $\|\cdot\|_F$ denotes the Frobenius norm.}
\end{lemma}

\blue{Note that this result is a generalisation of \cite{sketchedW}, Theorem 2 (they assumes that the cost matrices are pairwise distances, which amount to the $\W_1$ case), but requires the weights to have positive entries (as opposed to non-negative entries).}

\begin{proof}
	
	We split the difference in two terms:
	\begin{align*}\left|\W(\alpha, \beta; C) - \W(\oll{\alpha}, \oll{\beta}; \oll{C})\right|
		&\leq \left|\W(\alpha, \beta; C) - \W(\alpha, \beta; \oll{C})\right| =: \mathrm{I} \\
		&+ \left|\W(\alpha, \beta; C) - \W(\oll{\alpha}, \oll{\beta}; C)\right| =: \mathrm{II}
	\end{align*}
	
	\textrm{---} \textit{Step 1}: Controlling I \blue{Using the primal formulation}
	
	\blue{We use Equation \ref{eqn:WC}: let $\oll{\pi}^*$ optimal for $\W(\alpha, \beta, \oll{C})$. In particular, $\oll{\pi}^*$ is admissible for the problem $\W(\alpha, \beta, C)$. We have}
	\blue{\begin{align*}
		\W(\alpha, \beta; C) - \W(\alpha, \beta; \oll{C}) &= \underset{\pi \in \Pi(\alpha, \beta)}{\min}\ \pi \cdot C - \underset{\pi \in \Pi(\alpha, \beta)}{\min}\ \pi \cdot \oll{C}\\
		&\leq \oll{\pi}^* \cdot C - \oll{\pi}^* \cdot \oll{C}
		= \Sum{i=1}{n}\Sum{j=1}{m}\oll{\pi}^*_{i,j}(C_{i,j} - \oll{C}_{i,j}) \\
		&\leq \|C-\oll{C}\|_{\infty} \Sum{i=1}{n}\Sum{j=1}{m}\oll{\pi}^*_{i,j}
		= \|C-\oll{C}\|_{\infty},
	\end{align*}}
	\blue{where the property $\pi \in \Pi(\alpha, \beta)$ implied $\sum_{i,j}\oll{\pi}^*_{i,j}=1$. By using the same argument symmetrically, we obtain}
	$$\blue{\mathrm{I} \leq  \|C-\oll{C}\|_{\infty}.}$$

	\textrm{---} \textit{Step 2}: \blue{Controlling the dual variables}
	
	\blue{Consider the Legendre dual problem associated to \ref{eqn:WC}:}
	\begin{equation}\label{eqn:WCdual}
		\blue{\W(\alpha, \beta; C) = \underset{\substack{f\in \R^\npoints, g\in \R^m \\ f \oplus g \leq C}}{\sup}\ f^T \alpha + g^T \beta.}
	\end{equation}
	\blue{Let $f^*, g^*$ optimal for the dual formulation, our objective is to bound this dual solution in a set which depends on $C$. First, notice that the value and constraints remain unchanged if we replace $(f^*, g^*)$ with $(f^* - t\mathbbold{1}, g^* + t\mathbbold{1})$ for $t\in \R$, which allows us to assume $f^*_1=0$. We now leverage the complementary slackness property (which characterises the primal-dual optimality conditions for this linear problem, see \cite{computational_ot} Section 3.3): for any $\pi^*$ optimal for the primal problem \ref{eqn:WC}, we have the implication}
	\blue{$$\pi^*_{i,j} \neq 0 \Longrightarrow f^*_i + g^*_j = C_{i,j}. $$}
	\blue{The primal constraints imply that $\sum_j\pi^*_{1,j} = \alpha_1 > 0$ and that $\pi^* \geq 0$ entry-wise, there exists $j_1 \in \llbracket 1, m \rrbracket$ such that $\pi^*_{1, j_1} \neq 0$. Using the complementary slackness implication, we obtain $0 + g^*_{j_1} = C_{1, j_1}$. We now use the dual constraint $f^*\oplus g^* \leq C$ at $i=1$ to show that $\forall j \in \llbracket 1, m \rrbracket,\; g^*_j \leq C_{1,j}$. This allows us to find a lower-bound on $f^*$: since $\forall i \in \llbracket 2, n \rrbracket,\; \sum_j \pi^*_{i,j} = \alpha_i > 0$, thus there exists a $j_i \in \llbracket 1, m \rrbracket$ such that $\pi^*_{i,j_i} \neq 0$, yielding $f^*_i+g^*_{j_i}=C_{i,{j_i}}$, then since $g^*_{j_i} \leq C_{1, j_i}$, this yields the lower-bound $f^*_i \geq C_{i,j_i} - C_{1,j_i}$. For an upper-bound on $f^*$, we use the dual constraint at $(i, j_1)$: we have $f^*_i + g^*_{j_1} \leq C_{i, j_1}$, then we use $g^*_{j_1}=C_{1,j_1}$ proved earlier to show $f^*_i \leq C_{i,j_1} - C_{1,j_1}$. At this point, we have the following control on $f^*_i$:}
	\blue{$$f^*_1 = 0, \quad \forall i \in \llbracket 2, n \rrbracket,\; C_{i,j_i} - C_{1, j_i} \leq f^*_i \leq C_{i,j_1} - C_{1,j_1}.$$}
	\blue{	Regarding $g^*$, we already have $\forall j \in \llbracket 1, m \rrbracket,\; g_j \leq C_{1,j}$. For a lower bound, since $\sum_i \pi_{i,j} = \beta_j > 0$, there exists $i_j \in \llbracket 1, n \rrbracket$ such that $\pi^*_{i_j, j} \neq 0$, so by complementary slackness $f^*_{i_j} + g^*_j = C_{i_j, j}$, thus by the upper-bound on $f^*$ we have if $i_j \neq 1$ that $g^*_j \geq C_{i_j, j} - C_{i_j,j_1} + C_{1,j_1}$. If $i_j=1$ then $f^*_{i_j}=0$ and $g^*_{j} = C_{i_1, j}$. Our control on $g^*$ is the following:}
	\blue{$$\forall j \in \llbracket 1, m \rrbracket, \left\lbrace\begin{array}{cc}
		C_{i_j, j} - C_{i_j,j_1} + C_{1,j_1} \leq g^*_j \leq C_{1,j} & \text{if}\ i_j \neq 1 \\
		g^*_{j} = C_{i_1, j} & \text{if}\ i_j=1
	\end{array} \right. .$$}
	\blue{We summarise our bounds in the following (weaker) statement, which holds thanks to the condition $C \geq 0$ (entry-wise):}
	\blue{$$\|f^*\|_{\infty} \leq \|C\|_\infty,\quad \|g^*\|_{\infty} \leq \|C\|_\infty.$$}
	
	\textrm{---} \textit{Step 3}: \blue{Bounding II using the dual formulation}
	
	\blue{Let $f^*, g^*$ optimal for the dual formulation \ref{eqn:WCdual} of $\W(\alpha, \beta, C)$, which by Step 2 we can choose to verify $\|f^*\|_{\infty} \leq \|C\|_{\infty}$ and $\|g^*\|_{\infty} \leq \|C\|_{\infty}$. In particular, $(f^*, g^*)$ is admissible for the dual formulation of $\W(\oll{\alpha}, \oll{\beta}, C)$.}
	\blue{\begin{align*}
		\W(\alpha, \beta; C) - \W(\oll{\alpha}, \oll{\beta}; C) &= \underset{f\oplus g \leq C}{\max}\ f^T \alpha + g^T \beta - \underset{\oll{f}\oplus \oll{g} \leq C}{\max}\ \oll{f}^T \oll{\alpha} + \oll{g}^T \oll{\beta} \\
		& \leq (f^*)^T \alpha + (g^*)^T \beta - (f^*)^T \oll{\alpha} - (g^*)^T \oll{\beta} \\
		&= (f^*)^T (\alpha - \oll{\alpha}) + (g^*)^T (\beta - \oll{\beta}) \\
		&\leq \|f^*\|_\infty \|\alpha - \oll{\alpha}\|_1 + \|g^*\|_\infty \|\beta - \oll{\beta}\|_1 \\
		&\leq \|C\|_{\infty}( \|\alpha - \oll{\alpha}\|_1 + \|\beta - \oll{\beta}\|_1).
	\end{align*}}
	\blue{By symmetry, we obtain $|\W(\alpha, \beta; C) - \W(\oll{\alpha}, \oll{\beta}; C)| \leq \|C\|_{\infty}( \|\alpha - \oll{\alpha}\|_1 + \|\beta - \oll{\beta}\|_1)$.}
	
	\textrm{---} \textit{Step 4}: Wrapping up

	\blue{By Step 1 and Step 3 combined we conclude: }
	\blue{$$\left|\W(\alpha, \beta; C) - \W(\oll{\alpha}, \oll{\beta}; \oll{C})\right| \leq \mathrm{I} + \mathrm{II} \leq  \|C-\oll{C}\|_{\infty} + \|C\|_{\infty}( \|\alpha - \oll{\alpha}\|_1 + \|\beta - \oll{\beta}\|_1).$$}
	
	\textrm{---} \blue{Details for the proof of \ref{eqn:wass_stabilityL2}}
	
	\blue{For the first term, to get the Frobenius norm $\|C-\oll{C}\|_F$ instead of the infinite norm, it suffices to use that $\|M\|_\infty \leq \|M\|_F$.}
	
	\blue{For the second term, note that the penultimate inequality of Step 3 can also be written with the Cauchy-Schwarz inequality, yielding $\|C\|_F(\|\alpha - \oll{\alpha}\|_2 +  \|\beta - \oll{\beta}\|_2)$, where the upper-bound on $\|f^*\|_2$ and $\|g^*\|_2$ by $\|C\|_F$ are obtained using the element-wise bounds on $f^*$ and $g^*$ from Step 2.}
\end{proof}

\subsection{Proof of \texorpdfstring{\ref{thm:cv_fixed_point_distance}}{ref thm fixed point} and convergence rate}\label{p:cv_fixed_point_distance}

The proof of \ref{thm:cv_fixed_point_distance} requires matrix concentration technicalities. In the following, $\|\cdot\|_{\mathrm{op}}$ denotes the $\|\cdot\|_2$-induced operator norm on $\R^{d \times d}$, and $S_d(\R)$ denotes the space of symmetric $d\times d$ matrices. We write $\preceq$ for the Loewner order of positive semi-definite symmetric matrices ($A \preceq B$ means that $B-A$ is positive semi-definite). We recall the following Hoeffding inequality.

\begin{theorem}[Matrix Hoeffding Inequality,~\cite{tropp2012concentration}, Theorem 1.3]\label{thm:tropp_matrix_concentration}\
  
	Let $q \in \N^*,\; (X_i)_{i \in \llbracket 1, q \rrbracket}$ independent random variables with values in $S_d(\R)$, such that $\E{X_i} = 0$.
	Suppose that $\forall i \in \llbracket 1, q \rrbracket, \; \exists A_i \in S_d(\R) : X_i^2 \preceq A_i^2$. Let $\sigma^2 := \left\|\sum_iA_i^2\right\|_{\mathrm{op}}$,
	then for any $t>0$, 
	$$\P\left(\left\|\Sum{i=1}{q}X_i\right\|_{\mathrm{op}} \geq t\right) \leq d \exp\left(-\cfrac{t^2}{8\sigma^2}\right).$$
\end{theorem}

We deduce from~\ref{thm:tropp_matrix_concentration} the following lemma, where the $X_i$ follow a uniform law on $\Theta \subset \SS^{d-1}$.
\begin{lemma}[Hoeffding applied to $\theta \sim \mathcal{U}(\Theta)$]\label{lemma:hoeffding}\

	Let $(\theta_i)_{i \in \llbracket 1, q \rrbracket}$, independent random vectors following the uniform law on $\Theta \subset\SS^{d-1}$, where $\Theta$ is $\bbsigma$-measurable with $\bbsigma(\Theta) > 0$.
	Let $S_\Theta := \frac{1}{s_\Theta}\int_\Theta\theta \theta^T \dd \bbsigma(\theta)$, where $s_\Theta := \bbsigma(\Theta)$. $S_\Theta$ is the covariance matrix of $\theta \sim \mathcal{U}(\Theta)$.
	Let $\eta \in ]0, 1[$ and $t > 0$. Then with probability exceeding $1 - \eta$ we have 
	$$q \geq \cfrac{32\log\left(d/\eta\right)}{t^2} \Longrightarrow\left\|\cfrac{1}{q}\Sum{i=1}{q}\theta_i\theta_i^T - S_\Theta\right\|_{\mathrm{op}} \leq t.$$
	In the case $\Theta = \SS^{d-1}$, the condition $q \geq \cfrac{8\log\left(d/\eta\right)}{t^2}$ is sufficient.
\end{lemma}

\begin{proof}
	The idea is to apply~\ref{thm:tropp_matrix_concentration} to $X_i := \frac{1}{q}\theta_i\theta_i^T - \frac{1}{q}S_\Theta$. First, by definition, $\E{X_i} = 0$.

	We now find $A \in S_d^+(\R)$ such that $X_i^2 \preceq A$. Let $u \in \SS^{d-1}$, we compute:
	$$u^TX_i^2u = \frac{1}{q^2}\left(u^T\theta_i \theta_i^Tu - u^T\theta_i\theta_i^T S_\Theta u - u^T S_\Theta\theta_ i\theta_i^T u + u^T S_\Theta^2 u\right) \leq \left(\frac{1 + \|S_\Theta\|_{\mathrm{op}}}{q}\right)^2.$$
	Moreover, $\|S_\Theta\|_{\mathrm{op}} \leq 1$, since 
	$$\forall u \in \SS^{d-1}, \; u^TS_\Theta u = \cfrac{1}{s_\Theta}\Int{\Theta}{}u^T\theta \theta^T u \dd \bbsigma(\theta) \leq \cfrac{1}{s_\Theta}\Int{\Theta}{}1\dd \bbsigma(\theta) = 1.$$
	In conclusion $X_i^2 \preceq \frac{4}{q^2}I$. Using the notations of~\ref{thm:tropp_matrix_concentration},  we compute $\sigma^2 = 4/q$, and apply the Matrix Hoeffding inequality with $\Delta := \sum_iX_i = \frac{1}{q}\sum_i\theta_i \theta_i^T - S_\Theta$. It follows that for any $t > 0,\; \P\left(\|\Delta\|_{\mathrm{op}} \geq t\right) \leq d \exp\left(-\frac{qt^2}{32}\right)$.
	In order to have the event $\|\Delta\|_{\mathrm{op}} \leq t$ with probability exceeding $1 - \eta$, it is therefore sufficient that $\eta \geq  d \exp\left(-\frac{qt^2}{32}\right)$, which is equivalent to $q \geq \frac{32 \log(d/\eta)}{t^2}$.

	In the case $\Theta = \SS^{d-1}$, one has $S_\Theta = I/d$, and a finer Loewner upper-bound can be established, since 
	$$u^TX_i^2u =  \cfrac{1}{q^2}\left(u^T\theta_i\theta_i^Tu - \frac{2}{d}u^T\theta_i\theta_i^Tu + \frac{1}{d^2}\right) \leq \left(\cfrac{1 - \frac{1}{d}}{q}\right)^2 \leq \cfrac{1}{q^2},$$
	and thus $\sigma^2 = 1/q.$ This yields the Hoeffding inequality $\P\left(\|\Delta\|_{\mathrm{op}} \geq t\right) \leq d\exp\left(-\frac{qt^2}{8}\right)$, which in turn provides the announced weaker condition on $q$.
\end{proof}

With this tool at hand, we now prove a quantitative concentration result:

\begin{theorem}[Concentration of cell optima]\label{thm:ystar_concentration}\
	
	Let $\mathbf{m} = (\sigma_1, \cdots, \sigma_p)$ be a fixed matching configuration (see~\ref{sec:cells}) and let $(\theta_i)_{i \in \llbracket 1, p \rrbracket} \sim \bbsigma^{\otimes p}$ (uniform on $\SS^{d-1}$). We introduce the following notations and variables:
	\begin{itemize}
		\item For $(k, l) \in \llbracket 1, \npoints \rrbracket^2$, let $q_{k,l} := \#\lbrace i \in \llbracket 1, p \rrbracket\ |\ k = \sigma_i(l) \rbrace$;
		\item Let $\oll{c}_Z := \underset{l \in \llbracket 1, \npoints \rrbracket}{\max}\ \|z_l\|_2$;
		\item Let $\varepsilon \in ]0, \frac{4}{3}\npoints\oll{c}_Z]$;
		\item Let $\eta\in]0,1[$.
	\end{itemize}	
	
	Assume the following:

	\begin{itemize}
		\item[\textbullet] $(H_q): \; \forall (k, l)\in \llbracket 1, \npoints \rrbracket^2, \; q_{k,l} \geq \oll{q}$ or $q_{k,l} < \ull{q}, $ with $1 \leq \ull{q} \leq  \oll{q} \leq p\ ;$

		\item[\textbullet] $(H_1): \; p \geq \cfrac{697d^2\npoints^2\oll{c}_Z^2\log\left(3d/\eta\right)}{\varepsilon^2}\; ;$

		\item[\textbullet] $(H_2) : \; \oll{q} \geq \cfrac{512d^2\oll{c}_Z^2\log(3d\npoints\npoints^+/\eta)}{\varepsilon^2}\; ; \quad  \npoints^+ := \underset{k \in \llbracket 1, \npoints \rrbracket}{\max}\#\lbrace l \in \llbracket 1, \npoints \rrbracket\ |\ q_{k,l} \geq \oll{q} \rbrace; $

		\item[\textbullet] $(H_3): \; \ull{q} \leq \cfrac{\varepsilon}{8d\npoints^-\oll{c}_Z}\ p; \quad  \npoints^- := \underset{k \in \llbracket 1, \npoints \rrbracket}{\max}\#\lbrace l \in \llbracket 1, \npoints \rrbracket\ |\ q_{k,l} \leq  \ull{q} \rbrace;$

		\item[\textbullet] $(H_4): \; p \geq \cfrac{8d^2\npoints^2\oll{c}_Z^2\log(6\npoints^2/\eta)}{\varepsilon^2}.$

	\end{itemize}

	Then with probability exceeding $1-\eta$, writing $Y^* := \underset{Y' \in \R^{\npoints \times d}}{\argmin}\ q_\config(Y')$, we have
	\begin{equation}\label{eqn:ystar_concentration}
          \forall k \in \llbracket 1, \npoints \rrbracket,\; \left\|y_k^* - \Sum{l=1}{\npoints}S_{k,l}z_l\right\|_{2} \leq \varepsilon,
	\end{equation}
        where the normalized conditional covariance matrices $S_{k,l}$ are defined in~\ref{cor:crit_points_S} (we omit the $Y^*,Z$ exponent here for legibility).
\end{theorem}

\begin{proof}
	\textrm{---} \textit{Step 1}: Re-writing~\ref{eqn:next_opt_pos}.

	Remind that the matching configuration $\mathbf{m}$ is fixed here. Let $Y^* := \underset{Y' \in \R^{\npoints\times d}}{\argmin}\ q_\mathbf{m}(Y')$ and $k \in \llbracket 1, \npoints \rrbracket$. By~\ref{eqn:next_opt_pos}, we have 
	$$y_k^* = A^{-1}\left(\cfrac{1}{p}\Sum{i=1}{p}\theta_i \theta_i^T z_{\sigma_i(k)}\right), \text{\ with\ } A = \cfrac{1}{p}\Sum{i=1}{p}\theta_i\theta_i^T.$$
	Let $I_{k,l} := \lbrace i \in \llbracket 1, p \rrbracket\ |\ \sigma_i(k) = l \rbrace$. Since the $\sigma_i$ are permutations, we have $\llbracket 1, p \rrbracket = \Reu{l=1}{\npoints}I_{k,l} = \Reu{k=1}{\npoints}I_{k,l}$ and $k \neq k' \Rightarrow I_{k,l} \cap I_{k',l} = \varnothing; \; l \neq l' \Rightarrow I_{k,l} \cap I_{k,l'} = \varnothing$. We re-order the sum: 
	$$\cfrac{1}{p}\Sum{i=1}{p}\theta_i \theta_i^Tz_{\sigma_i(k)} = \Sum{l=1}{\npoints}\cfrac{1}{p}\Sum{i \in I_{k,l}}{}\theta_i \theta_i^T z_l = \Sum{l=1}{\npoints}\cfrac{q_{k,l}}{p}B_{k,l}z_l,$$
	where $q_{k,l} := \#I_{k,l}$ and $B_{k,l} := \cfrac{1}{q_{k,l}}\Sum{i \in I_{k,l}}{}\theta_i \theta_i^T$.
	This invites the definition of the matrix $R = (r_{k,l}),\; r_{k,l} := \frac{q_{k,l}}{p}$, which is bi-stochastic by construction.

	\textrm{---} \textit{Step 2}: Separating the terms in $y_k^*$.

	We will see later that the empirical covariance matrix $A$ concentrates towards the covariance matrix of $\theta \sim \bbsigma$, which is $I/d$. In order to quantify the impact of this concentration on $y_k^*$, we introduce the error term: $\delta A^- := A^{-1} - dI$.

	A similar concentration will be observed for $B_{k,l}$, but the $\theta_i$ in the sum are \textit{selected} such that $i \in I_{k,l}$. Recall that since we project in 1D, the permutations $\sigma_i$ arise from a sorting problem, namely $\sigma_i = \sort{Z}{\theta_i} \circ (\sort{Y}{\theta_i})^{-1}$, where we recall that $\sort{Y}{\theta}$ is a permutation sorting the numbers $(y_1^T \theta, \cdots, y_\npoints^T \theta)$.

        By definition, we have $\sigma_i(k)=l \Longleftrightarrow \theta_i \in \Theta_{k,l} = \left\lbrace \theta \in \SS^{d-1}\ |\ \sort{Z}{\theta} \circ (\sort{Y}{\theta})^{-1}(k) = l \right\rbrace$, where we omit again the $Y,Z$ exponent on $\Theta_{k,l}$ for legibility.

	Since the $\theta_i$ in $B_{k,l}$ are drawn under the condition $\theta_i \in \Theta_{k,l}$, we study the concentration $B_{k,l} \approx C_{k,l}$, where $C_{k,l} :=\frac{1}{d\bbsigma(\Theta_{k,l})} S_{k,l} $.
	In order to quantify this approximation, we define the error term $\delta B_{k,l} := B_{k,l} - C_{k,l}$.
	Similarly, the $r_{k,l} := \frac{q_{k,l}}{p}$ are Monte-Carlo approximations of $\bbsigma(\Theta_{k,l})$, which leads to the definition $\delta r_{k,l} := r_{k,l} - \bbsigma(\Theta_{k,l})$.

	We may now separate the terms in the result from Step 1:
	\begin{align*}
		y_k^* &= (dI + \delta A^-)\left(\Sum{l=1}{\npoints}r_{k,l}(\underbrace{C_{k,l} + \delta B_{k,l}}_{B_{k,l}})z_l\right)\\
		&=\underbrace{d\Sum{l=1}{\npoints}\bbsigma(\Theta_{k,l})C_{k,l}z_l}_{v}
        +\underbrace{\delta A^- \left(\Sum{l=1}{\npoints}r_{k,l}B_{k,l}z_l\right)}_{\delta v_1}
        +\underbrace{d\Sum{\substack{l=1 \\ q_{k,l} \geq \oll{q}}}{\npoints}r_{k,l}\delta B_{k,l}z_l}_{\delta v_2} \\
        &+\underbrace{d\Sum{\substack{l=1 \\ q_{k,l} < \ull{q}}}{\npoints}r_{k,l}\delta B_{k,l}z_l}_{\delta v_3}
        +\underbrace{d\Sum{l=1}{\npoints}\delta r_{k,l}C_{k,l}z_l}_{\delta v_4}.
	\end{align*}
        
	The separation of the terms in the second equality arises from $(H_q)$, formulated in the theorem. Observe that the first term $v$ is exactly $\Psi(Y^*)$, with $\Psi$ defined in~\ref{sec:E_crit}.
	Our objective is to provide conditions under which $\forall i \in \lbrace 1, 2, 3, 4\rbrace,\;\|\delta v_i\|_2 \leq \varepsilon / 4$ with probability exceeding $1-\eta$. To that end, we let $\varepsilon > 0$ and $\eta \in ]0, 1[$.

	\textrm{---} \textit{Step 3}: Condition for $\|\delta v_2\|_2 \leq \cfrac{\varepsilon}{4}$.

	First of all, note that if the sum defining $\delta v_2$ is empty, the condition holds trivially almost-surely. In the following, we suppose that the sum has at least one non-zero term.
	We have from Step 2, 
	$$\|\delta v_2\|_2 = \left\|d\Sum{\substack{l=1 \\ q_{k,l} \geq \oll{q}}}{\npoints}r_{k,l}\delta B_{k,l}z_l\right\|_2 \leq d \oll{c}_Z\Sum{\substack{l=1 \\ q_{k,l} \geq \oll{q}}}{\npoints}r_{k,l}\|\delta B_{k,l}\|_{\mathrm{op}}.$$
	Let the shorthands $\npoints_k^+ := \# J_k^+$ and $J_k^+ := \lbrace l \in \llbracket 1, \npoints \rrbracket\ |\ q_{k,l} \geq \oll{q} \rbrace$. We upper-bound the right term by
	$\Sum{\substack{l \in J_k^+}}{}r_{k,l}\|\delta B_{k,l}\|_{\mathrm{op}}  \leq \Sum{l\in J_k^+}{}r_{k,l} \underset{l \in J_k^+}{\max }\|\delta B_{k,l}\|_{\mathrm{op}} \leq \underset{l \in J_k^+}{\max }\|\delta B_{k,l}\|_{\mathrm{op}}$.

	For $l \in J_k^+,$ by \ref{lemma:hoeffding}, we have $\|\delta B_{k,l}\|_{\mathrm{op}} \leq t$ with probability exceeding $1 - \eta/(3\npoints\npoints_k^+)$ provided that $q_{k,l} \geq \frac{32\log\left(3d\npoints\npoints_k^+/\eta\right)}{t^2}$.
    Since the probability of $\bigcup_{l \in J_k^+} \{\|\delta B_{k,l}\|_{\mathrm{op}} > t\}$ can be upper bounded by the sum of the probabilities of each of the $n_k^+$ terms, it is upper bounded by $\eta/(3\npoints)$.
    Therefore, writing the event $\{\forall l \in J_k^+,\; \|\delta B_{k,l}\|_{\mathrm{op}} \leq t\}$ as the complementary of this union, we conclude that it holds with probability exceeding $1 - \eta/(3\npoints)$, provided that 
    $$\forall l \in J_k^+, \; q_{k,l} \geq \cfrac{32\log\left(3d\npoints\npoints_k^+/\eta\right)}{t^2}.$$
	A sufficient condition for this last assumption to hold is $(H_2^k): \; \oll{q} \geq \cfrac{32\log\left(3d\npoints\npoints_k^+/\eta\right)}{t^2}$.
    Applying this result to $t := \frac{\varepsilon}{4d\oll{c}_Z}$, and by letting $\npoints^+ := \underset{k \in \llbracket 1, \npoints \rrbracket}{\max}\ \npoints_k^+,$ a sufficient condition to have $\|\delta v_2\|_2 \leq \frac{\varepsilon}{4}$ with probability exceeding $1 - \eta/(3\npoints)$ is
	$$\quad (H_2) : \quad \oll{q} \geq \cfrac{512d^2\oll{c}_Z^2\log(3d\npoints\npoints^+/\eta)}{\varepsilon^2}.$$
    
	\textrm{---} \textit{Step 4}: Condition for $\|\delta v_3\|_2 \leq \cfrac{\varepsilon}{4}$.

	With a computation analogous to Step 3, we write 
	$$\|\delta v_3\|_2 = \left\|d\Sum{\substack{l=1 \\ q_{k,l} < \ull{q}}}{\npoints}r_{k,l}\delta B_{k,l}z_l\right\|_2 \leq d\oll{c}_Z \Sum{l \in J_k^-}{}r_{k,l}\|\delta B_{k,l}\|_{\mathrm{op}},$$	
	where, like in Step 3, we define $\npoints_k^- := \# J_k^-$ and $J_k^- := \lbrace l \in \llbracket 1, \npoints \rrbracket\ |\ q_{k,l} \leq \ull{q} \rbrace$.
	If $\npoints_k^- = 0$ then the objective holds almost-surely, thus we suppose $\npoints_k^- \geq 1$.
	In this setting, the $q_{k,l}$ are small, thus we have little control over $\|\delta B_{k,l}\|_{\mathrm{op}}$, which can be upper bounded by $2$.

	Leveraging the condition $q_{k,l} \leq \ull{q}$, which holds for $l \in J_k^-$, we have $r_{k,l} = q_{k,l}/p \leq \ull{q} / p$.
	In order to have $\|\delta v_3\|_2 \leq \frac{\varepsilon}{4}$ almost-surely, it is sufficient to have $(H^k_3): \quad \ull{q} \leq \frac{\varepsilon}{8d\npoints_k^-\oll{c}_Z}\ p.$
	Again, with $\npoints^- := \underset{k \in \llbracket 1, \npoints \rrbracket}{\max}\ \npoints_k^-,$ we obtain the sufficient condition:
	$$(H_3): \quad \ull{q} \leq \cfrac{\varepsilon}{8d\npoints^-\oll{c}_Z}\ p.$$
        
	\textrm{---} \textit{Step 5}: Condition for $\|\delta v_4\|_2 \leq \cfrac{\varepsilon}{4}$.

	By definition, $\delta v_4 = d\Sum{l=1}{\npoints}\delta r_{k,l} C_{k,l}z_l$, then $\|\delta v_4\|_{2} \leq \oll{c}_Zd\Sum{l=1}{\npoints}|\delta r_{k,l}| \|C_{k,l}\|_{\mathrm{op}}$.
	We use the upper-bound $\|C_{k,l}\|_{\mathrm{op}} \leq 1$ (observe that  $\|C_{k,l}\|_{\mathrm{op}}$ can be made as close to $1$ as desired by choosing $\Theta_{k,l}$ as a very small portion of the sphere).
	In order to have $\|\delta v_4\|_2 \leq \frac{\varepsilon}{4}$, it is sufficient to have $\forall l \in \llbracket 1, \npoints \rrbracket,\; |\delta r_{k,l}| \leq \frac{\varepsilon}{4d\npoints\oll{c}_Z} =: t$. Our objective is to quantify the Monte-Carlo error 
	$$\delta r_{k,l} = \cfrac{\#\lbrace i \in \llbracket 1, p \rrbracket\  |\  \theta_i \in \Theta_{k,l}\rbrace}{p} - \bbsigma(\Theta_{k,l}).$$
	To that end, we fix $l \in \llbracket 1, \npoints \rrbracket$ and apply the standard Bernoulli Chernoff concentration inequality (additive form) to $X_i := \mathbbold{1}(\theta_i \in \Theta_{k,l})$. By definition, $\E{X_i} = \bbsigma(\Theta_{k,l})$, hence by Chernoff
	$$\P\left(\left|\cfrac{1}{p}\Sum{i=1}{p}X_i - \bbsigma(\Theta_{k,l})\right| > t\right) \leq 2e^{-2pt^2}.$$
	It follows that the inequality $p \geq \frac{\log(6\npoints^2/\eta)}{2t^2}$ implies $|\delta r_{k,l}|\leq t$ with probability exceeding $1 - \frac{\eta}{3\npoints^2}$.
	Substituting $t = \frac{\varepsilon}{4d\npoints\oll{c}_Z}$ yields 
	$$(H_4):\; p \geq \cfrac{8d^2\npoints^2\oll{c}_Z^2\log(6\npoints^2/\eta)}{\varepsilon^2}.$$
	
	Using the same reasoning as in previous steps, under $(H_4)$, the event $\{ \forall l \in \llbracket 1, \npoints \rrbracket,\; |\delta r_{k,l}| \leq \frac{\varepsilon}{4d\npoints\oll{c}_Z} \}$
	holds with probability exceeding $1 - \frac{\eta}{3\npoints}$, which implies that our objective $\|\delta v_4\|_2 \leq \frac{\varepsilon}{4}$ also holds with the same probability.

	\textrm{---} \textit{Step 6}: Condition for $\|\delta v_1\|_2 \leq \cfrac{\varepsilon}{4}$.

	We have
	$$\|\delta v_1\|_2 \leq \|\delta A^-\|_{\mathrm{op}} \left\|\Sum{l=1}{\npoints}r_{k,l}B_{k,l}z_l\right\|_2 \leq  \cfrac{\|\delta A^-\|_{\mathrm{op}}}{d} \left(\|v\|_2 + \|\delta v_2\|_2 + \| \delta v_3\|_2 + \|\delta v_4\|_2\right).$$
	In the following, we continue conditionally on the three events $\text{\textquotedblleft}\|\delta v_i\|_2 \leq \frac{\varepsilon}{4}\text{\textquotedblright},\; i \in \lbrace 2, 3, 4 \rbrace$, under which:
	$$\|\delta v_1\|_2 \leq \cfrac{\|\delta A^-\|_{\mathrm{op}}}{d} \left(\|v\|_2 + \cfrac{3\varepsilon}{4}\right).$$
	We now dominate $\|v\|_2 = \left\|\Sum{l=1}{\npoints}S_{k,l}z_l\right\|_2$.
      Recall that the $(\Theta_{k,l})_{l \in \llbracket 1, \npoints \rrbracket}$ are disjoint, with $\Reu{l=1}{\npoints}\Theta_{k,l} = \SS^{d-1}$, which implies $\Sum{l=1}{\npoints}S_{k,l} = d\Int{\SS^{d-1}}{}\theta \theta^T \dd \bbsigma(\theta) = I$.
	Since the $S_{k,l}$ are symmetric semi-definite, the previous equation provides $\|S_{k,l}\|_{\mathrm{op}} \leq 1$, which in turn yields $\|v\|_2 \leq \npoints\oll{c}_Z$.
	Assuming $\varepsilon \leq \frac{4}{3}\npoints\oll{c}_Z$, we get finally $\|\delta v_1\|_2 \leq \|\delta A^-\|_{\mathrm{op}} \frac{2\npoints\oll{c}_Z}{d}$.

	It is sufficient to find a condition under which $\|\delta A^-\|_{\mathrm{op}} \leq \frac{d\varepsilon}{8\npoints\oll{c}_Z} =: t$.
	We cannot apply \ref{lemma:hoeffding} directly since $\delta A^-$ has an inverse operation.
	First, $\|\delta A^-\|_{\mathrm{op}} = \|A^{-1} -dI\|_{\mathrm{op}} = \left\|d(I - d\delta A)^{-1} - dI\right\|_{\mathrm{op}}$, with $\delta A := I/d - A$.
	Then, assuming $(H_{\delta A}) : d\|\delta A\|_{\mathrm{op}} < 1$, we use a Neumann series for the inverse:
	$$\|\delta A^{-}\|_{\mathrm{op}} = \left\|\Sum{k=1}{+\infty}(d\delta A)^k\right\|_{\mathrm{op}} \leq \Sum{k=1}{+\infty}(d\|\delta A\|_{\mathrm{op}})^k,$$
	and finally $\|\delta A^-\|_{\mathrm{op}} \leq \cfrac{d^2\|\delta A\|_{\mathrm{op}}}{1 - d\|\delta A\|_{\mathrm{op}}}$. Consider $f := \app{[0, \frac{1}{d}[}{[0, +\infty [}{u}{\frac{d^2u}{1-du}}$.
	
	The function $f$ is bijective and increasing, with $f^{-1} = \app{[0, +\infty [}{[0, \frac{1}{d}[}{v}{\frac{v}{d(d+v)}}$.
	This analysis yields under $(H_{\delta A}),\; \|\delta A^-\|_{\mathrm{op}} \leq t \Longleftarrow \|\delta A\|_{\mathrm{op}} \leq \cfrac{t}{d(d+t)}$.
	
	Conveniently, by \ref{lemma:hoeffding}, $\|\delta A\|_{\mathrm{op}} \leq s$ with probability $1 - \eta/3$ if $p \geq \frac{8\log\left(3d/\eta\right)}{s^2}.$
	We can apply this to 
	$$\frac{t}{d(d+t)} = \frac{\varepsilon}{8d\npoints\oll{c}_Z(1 + \frac{\varepsilon}{8\npoints\oll{c}_Z})},$$ 
	but in order to simplify the expression, we apply it to 
	$$s := \frac{3\varepsilon}{28d\npoints\oll{c}_Z} \leq \frac{t}{d(d+t)},$$ 
	where the inequality holds thanks to $\varepsilon \leq \frac{4}{3}\npoints\oll{c}_Z$.

	Now we must quantify the assumption $(H_{\delta A}) : \|\delta A\|_{\mathrm{op}} < 1/d$. Notice that $s \leq 1/d$ and thus the event $\|\delta A\|_{\mathrm{op}} < s$ is contained in the event $\|\delta A\|_{\mathrm{op}} < 1/d$, hence it is sufficient to satisfy $(H_1)$, which we write (after upper-bounding $8\times 28^2 / 9 \leq 697$):
	$$(H_1): \quad p \geq \cfrac{697d^2\npoints^2\oll{c}_Z^2\log\left(3d/\eta\right)}{\varepsilon^2}.$$
	To summarise, under $(H_1)$, we have $\|\delta A\|_{\mathrm{op}} \leq s$ with probability exceeding $1 - \eta/3$.
	Conditionally to the events $\text{\textquotedblleft}\|\delta A\|_{\mathrm{op}} \leq s\text{\textquotedblright},\;\text{\textquotedblleft}\|\delta v_i\|_2 \leq \frac{\varepsilon}{4}\text{\textquotedblright},\; i \in \lbrace 2, 3, 4 \rbrace$, this step shows $\|\delta v_1\|_2 \leq \frac{\varepsilon}{4}$.

	\textrm{---} \textit{Step 7}: Wrapping up.

	We now work under the conditions $(H_i),\; i \in \lbrace 1, 2, 3, 4 \rbrace$.
	By Step 1, 
	$$\|y_k^* - v_k\|_2 \leq\|\delta v^k_1\|_2 + \|\delta v^k_2\|_2 + \|\delta v^k_3\|_2 + \|\delta v^k_4\|_2,$$
	where we restore the omitted $k$ indices.
	By Step 3, with probability exceeding $1 - \eta/(3\npoints)$, we have $\|\delta v^k_2\|_2 \leq \frac{\varepsilon}{4}$,
	thus with probability $1 - \eta/3$ we have $\forall k \in \llbracket 1, \npoints \rrbracket, \; \|\delta v^k_2\|_2 \leq \frac{\varepsilon}{4}$.
	By Step 4, we have almost-surely $\forall k \in \llbracket 1, \npoints \rrbracket, \;\|\delta v_3^k\|_2 \leq \frac{\varepsilon}{4}$.
	By Step 5, with probability $1 - \eta/3,\; \|\delta A\|_{\mathrm{op}} \leq s$. Putting this together yields that with probability $1 - \eta$, we have:
	$$\forall k \in \llbracket 1, \npoints \rrbracket,\; \|\delta v^k_2\|_2 \leq \cfrac{\varepsilon}{4},\; \|\delta v_3^k\|_2 \leq \cfrac{\varepsilon}{4},\; \|\delta v_4^k\|_2 \leq \cfrac{\varepsilon}{4}\ \mathrm{and}\ \|\delta A\|_{\mathrm{op}} \leq s.$$
	Finally, Step 5 shows that conditionally to the events above, $\|\delta v_1^k\|_2 \leq \cfrac{\varepsilon}{4}$ almost-surely.
	Thus with probability exceeding $1 - \eta,\; \forall k \in \llbracket 1, \npoints \rrbracket, \|y_k^* - v_k\|_2 \leq \varepsilon.$
	Since $v_k = \Sum{l=1}{\npoints}S_{k,l}z_l$, with probability over $1 - \eta:\; \forall k \in \llbracket 1, \npoints \rrbracket,\; \left\|y_k^* - \Sum{l=1}{\npoints}S_{k,l}z_l\right\|_{2} \leq \varepsilon. $
\end{proof}

In order to get the summarised result from~\ref{sec:closeness}, we simplify the conditions as follows.
\begin{corollary}[Simplified conditions for~\ref{thm:ystar_concentration}]\label{cor:simplified_condition_concentration}

	With the notations of~\ref{thm:ystar_concentration}, the condition:
	\begin{equation}\label{eqn:simplified_condition_concentration}
		(H_p): \quad p \geq \left(\cfrac{4096 d^3\npoints\oll{c}_Z^3 \log(3d\npoints^2/\eta)}{\varepsilon^3}\right) \vee \left(\cfrac{697d^2\npoints^2\oll{c}_Z^2\log\left(3d/\eta\right)}{\varepsilon^2}\right) \vee \left(\cfrac{8d^2\npoints^2\oll{c}_Z^2\log(6\npoints^2/\eta)}{\varepsilon^2}\right)
	\end{equation}
	implies $(H_q)$ and $(H_i)_{i \in \lbrace 1, 2, 3, 4 \rbrace}$, and thus is sufficient in order to have~\ref{eqn:cv_fixed_point_distance}.
\end{corollary}

\begin{proof}
	The second and third terms of~\ref{eqn:simplified_condition_concentration} correspond to $(H_1)$ and $(H_4)$ respectively.
	Then, using $\npoints^+, \npoints^- \leq \npoints$, we have 
	\begin{align*}
			(H_2) &\Longleftarrow \oll{q} \geq \cfrac{512d^2\oll{c}_Z^2\log(3d\npoints^2\eta)}{\varepsilon^2}, \\
			(H_3) &\Longleftarrow \ull{q} \leq \cfrac{\varepsilon}{8d\npoints\oll{c}_Z}\ p.
	\end{align*}
	Let $q := \cfrac{512d^2\oll{c}_Z^2\log(3d\npoints^2/\eta)}{\varepsilon^2}$; $\oll{q} = \ull{q} = q$. $(H_q)$ and $(H_2)$ are automatically satisfied by this choice.
	For $q$ to satisfy $(H_3)$, it is sufficient to have 
	$$\cfrac{512d^2\oll{c}_Z^2\log(3d\npoints^2/\eta)}{\varepsilon^2} \leq \cfrac{\varepsilon}{8d\npoints\oll{c}_Z}\ p,\text{ i.e. }p \geq \cfrac{4096 d^3\npoints\oll{c}_Z^3\log(3d\npoints^2/\eta)}{\varepsilon^3}$$
\end{proof}

\blue{\subsection{Closed-form expression for Block-Coordinate Descent}\label{sec:closed_form_BCD}

In \ref{alg:BCD}, we mention in line 4 the minimisation $\min_YJ(\pi, Y)$, where
\begin{equation*}
	J := \app{\U^p \times \R^{\npoints \times d}}{\R_+}{(\pi^{(1)}, \cdots, \pi^{(p)}), Y}{\cfrac{1}{ p}\Sum{i=1}{p}\Sum{k=1}{\npoints}\Sum{l=1}{\npoints}(\theta_i^T y_k - \theta_i^T z_l)^2\pi_{k,l}^{(i)}},
\end{equation*}
and claim that it can in fact be done explicitly. We provide the formula below, which stems from a straightforward quadratic minimisation: let $Y^* = ((y_1^*)^T, \cdots, (y_\npoints^*)^T)^T = \argmin_YJ(\pi, Y)$, we obtain
\begin{equation*}
	\forall k \in \llbracket 1, \npoints \rrbracket,\; y_k^* = \left(\cfrac{1}{\npoints}\Sum{i=1}{p}\theta_i\theta_i^T\right)^{-1}\left(\Sum{i=1}{p}\Sum{l=1}{\npoints}\pi_{k,l}^{(i)}\theta_i\theta_i^Tz_l\right),
\end{equation*}
where we used the constraint $\pi \in \mathbb{U}^p$ which implies $\sum_l\pi_{k,l}^{(i)}= \frac{1}{\npoints}$.}

\end{document}